\documentclass[twoside,letterpaper]{article}
\makeatletter
\providecommand{\Notice@String}{}
\makeatother

\usepackage[preprint]{aistats2027}
\usepackage{etoc}

\AtBeginDocument{\setlength{\footskip}{12pt}}
\usepackage[T1]{fontenc}
\usepackage[utf8]{inputenc}
\usepackage{lmodern,microtype}
\usepackage{amsmath,amssymb,amsthm,mathtools}
\usepackage{pgfplots}
\usepackage{graphicx,booktabs,tabularx,longtable,array}
\usepackage[round,authoryear]{natbib}

\usepackage{xurl}
\usepackage{xcolor,fancyvrb,fvextra}
\usepackage{adjustbox}
\usepackage[colorlinks=true,allcolors=blue]{hyperref}

\theoremstyle{plain}
\newtheorem{theorem}{Theorem}[section]
\newtheorem{lemma}[theorem]{Lemma}
\newtheorem{proposition}[theorem]{Proposition}
\newtheorem{corollary}[theorem]{Corollary}

\theoremstyle{definition}
\newtheorem{assumption}{Assumption}

\theoremstyle{remark}

\allowdisplaybreaks[1]
\newcommand{\bodytitle}{Strong Gaussian Approximation of Constant Learning Rate SGD: Bias and Trajectory Inference}

\hypersetup{
  pdftitle={Strong Gaussian Approximation of Constant Learning Rate SGD: Bias and Trajectory Inference},
  pdfauthor={Joseph Keslin and Wei Biao Wu}
}

\runningtitle{Strong Gaussian Approximation of Constant Learning Rate SGD}
\runningauthor{Joseph Keslin and Wei Biao Wu}

\begin{document}

\twocolumn[
\aistatstitle{\bodytitle}

\aistatsauthor{Joseph Keslin \qquad Wei Biao Wu}

\aistatsaddress{
Department of Statistics\\
University of Chicago
}
]

\etocdepthtag.toc{main}
\raggedbottom

\begin{abstract}
This paper develops strong Gaussian approximations (KMT type) for constant learning rate stochastic gradient descent (SGD). We work with heavy tailed noise and assume that the number of iterates grows and the learning rate shrinks simultaneously. Our work constructs simple Gaussian processes on the same space and bounds, in probability, the maximum difference between the SGD errors and their respective limits. While there exist central limit theorems for constant rate SGD centered at the optimum, these only provide individual iterate inference. By contrast, our quenched KMT approximations allow for finite sample, simultaneous, whole-trajectory-level inference and imply stochastic process convergence (Functional Central Limit Theorems (FCLTs)). For accumulated errors centered at the true optimum, we first obtain a KMT-type approximation bound when the learning rate scales as $\gamma_n=\frac{c}{\sqrt n}+o(n^{-1/2})$ for a constant $c \in \mathbb{R}^+$ and $n$ iterations. The result reveals a trade off between controlling the bias of SGD and the initialization effect, with no learning rate scaling generally making both contributions vanish in the accumulated error limit. To eliminate these contributions, we consider a Burn in scheme and Richardson-Romberg extrapolation and prove respective KMT and FCLT theorems. We also prove a finite-horizon KMT-type approximation and FCLT for the rescaled iterate trajectory to a stationary Ornstein--Uhlenbeck (OU) process. We consider applications in risk control and simultaneous inference and perform simulations to illustrate our results.
\end{abstract}

\section{Introduction}
Constant learning rate stochastic gradient descent (SGD) is widely used in  online learning, online optimization and for learning with large datasets \citep{BottouCurtisNocedal2018}.  In particular, consider the optimization problem with minimizer $ \theta^* = \text{argmin}_{\theta\in\mathbb R^d}G(\theta)$, where $G(\theta)=\mathbb E[g(\theta,X)]$ is smooth and strongly convex. For  i.i.d. data $X_1,X_2,\ldots$ and learning rate $\gamma$, constant learning rate stochastic gradient descent (SGD) iterates are defined as
\begin{equation}\label{eq:sgd-main}
\theta_k(\gamma)=\theta_{k-1}(\gamma)-\gamma\nabla g(\theta_{k-1}(\gamma),X_k)
\end{equation}
with initialization $\theta_0$. A smaller learning rate reduces stationary bias but takes longer to converge and slows decay of the initialization effect \citep{DieuleveutDurmusBach}. The learning rate remains fixed throughout each run, but the learning rate decreases with longer runs, so that $n\to\infty$ and $\gamma_n\to0$ jointly. We allow for arbitrary initializations. While existing results like \cite{NEURIPS2025_7aae9e3e} establish Gaussian approximation in this setting, these only hold for individual iterate level control. To obtain SGD boundary crossing probabilities, or accumulated loss hitting loss probabilities or simultaneous confidence intervals along the trajectory we require a stronger Gaussian KMT type approximation \citep{KMT1975,GotzeZaitsev2010}. These applications are especially important for instance in finance and online learning where hitting probabilities can represent bankruptcy.

We provide inference for these settings by constructing Gaussian
processes on the same probability space and bounding their maximum
difference from the corresponding SGD processes. These KMT-type
strong approximations allow us to study accumulated and trajectory
SGD errors through well studied properties of Gaussian processes.
We also construct consistent estimators of the local curvature
matrix $A$ and gradient-noise covariance $S$, allowing the distribution of the Gaussian process to be estimated from data.

For example, we study an SGD based portfolio allocation problem
and, under suitable assumptions, obtain bounds on the probability
of bankruptcy over a specified horizon. We also study the
probability that an online prediction error ever exceeds a
specified tolerance. A third application we explore is developing simultaneous confidence intervals to enable valid inference for the true optimizer using an averaged iterate picked after observing the SGD trajectory.

We study KMT and FCLT results \citep{KMT1975,billingsley1999convergence,whitt2002stochastic} for two different types of SGD processes. The first consists of accumulated errors, $\sum_{k=1}^m(\theta_k-\theta^*)$, as $m$ increases. Dividing these sums by $m$ yields the Averaged Stochastic Gradient (ASGD) algorithm \citep{PolyakJuditsky1992}. The second consists of the individual iterates themselves after a suitable Burn in period. The individual errors, divided by $\sqrt\gamma$, are approximated by an Ornstein--Uhlenbeck process \citep{whitt2002stochastic}.

Nonlinear SGD generally incurs a bias (\cite{highdimsgd,DieuleveutDurmusBach}). This causes the accumulated errors centered at the true optimum in the root-$n$ limit, to have contributions of both sizes $\sqrt n\gamma$ and $1/(\sqrt n\gamma)$. Hence, generally no learning-rate choice makes both disappear. At the scaling $\sqrt n\gamma_n\to c\in(0,\infty)$, both remain present in the limit.

To circumvent this issue we investigate two ways to obtain limits without these contributions. Burn in discards the early iterates, allowing a smaller learning rate to control stochastic bias while the Burn in leads to exponential decay of the initialization term. Richardson--Romberg extrapolation combines two runs at learning rates $\gamma$ and $2\gamma$ to cancel the leading stochastic bias \citep{DieuleveutDurmusBach,Sheshukova2025}. We prove maximal Gaussian KMT-type approximations and FCLTs for these schemes. We finally conduct simulations to demonstrate the effect of bias, initialization and convergence in the setting of online logistic regression.

\subsection{Related work}\label{sec:related}
Our main contribution is a KMT type approximation of the entire SGD path centered at the true optimizer, with simultaneous scaling of learning rate and iterations. The closest result is \citet{NEURIPS2025_7aae9e3e}, who establish quenched CLTs for individual constant step size SGD iterates. We extend this analysis to the trajectory level maximum approximation error for both accumulated errors and rescaled individual iterates.

For accumulated errors, \citet{LiLouRichterWu2026} establish finite moment strong approximations and functional central limit theorems at fixed learning rates, including for Richardson Romberg combinations, centered at the iterates' stationary means. Our results address centering at the true optimizer. This presents additional difficulty due to the bias of SGD and the scaling relationship between the learning rate and number of iterations. Many interesting inference problems concern the optimizer, and concentration around a stationary mean does not by itself guarantee accuracy relative to the true optimizer. Our results provide guarantees directly for the true optimizer instead.

In the SGD literature, \citet{LiTaiE2017} develop stochastic differential equation approximations, and \citet{Flamand2026} establish a trajectory FCLT in a convex setting that also permits nonsmooth objectives. However, neither \citet{LiTaiE2017} nor \citet{Flamand2026} obtain a strong approximation type result. In addition, \citet{Flamand2026} do not work with constant step size while \citet{LiTaiE2017} do not focus on the heavy tailed case. 
\citet{Fontaine2021} obtain $L^2$ error bounds between SGD and a diffusion process. Converting these bounds into a probability bound for the maximal error by a union bound introduces a factor of the number of iterates $N$. By contrast, our result directly controls the maximal error in probability, providing guarantees over the entire trajectory. Moreover, their root mean square bound need not vanish after division by $\sqrt{\gamma}$, whereas our trajectory maximal error tends to zero in probability after division by $\sqrt{\gamma}$ and burn in. Their Remark (c) discusses functional extensions using Doob's inequality, but does not state a corresponding maximal error bound. Their analysis assumes uniformly Lipschitz sample gradients, a stronger regularity condition than our moment bound in Assumption~\ref{ass:moments}. Finally, their comparison diffusion need not be Gaussian and they do not study strong approximations for accumulated error like we do.

\cite{seo2026sharpstationarygaussianapproximation} establishes a Gaussian approximation for the distribution of a stationary constant step size SGD iterate. However, this describes individual iterates and does not control the maximum difference between the SGD trajectory and a Gaussian process like our KMT result does. Our KMT type bound enables inference about the entire trajectory that their result doesn't, like boundary crossing probabilities and simultaneous confidence intervals. 

Bias reduction through Richardson--Romberg extrapolation has been developed for constant step size SGD by \citet{DieuleveutDurmusBach}. \citet{Sheshukova2025} give nonasymptotic moment bounds for extrapolated averaged SGD. However, these works do not consider trajectory wise KMT type control. 

For statistical inference, \citet{ChenLeeTongZhang2020} develop plug-in and covariance estimators for averaged SGD, and \citet{ZhuChenWu2023} provide an online covariance estimator. \citet{XieZhang2022} establish FCLTs for constant step stochastic approximation. Again however, none of these represent KMT type results. Moreover, our estimators of local curvature and gradient noise covariance allow us to estimate the Gaussian limit processes from data. 

To our best knowledge this is the first paper to obtain maximal Gaussian approximations (KMT type results) for constant rate SGD centered at the true optimum.

\section{Assumptions and notation}\label{sec:assumptions}
Write $\Delta_k=\theta_k-\theta^*$, $\|V\|_p=(\mathbb E\|V\|_2^p)^{1/p}$ for a random vector, and use the Euclidean and operator norms for vectors and matrices. All expectations, probabilities, and projections for the recursion are conditional on initialization. Define $\rho_p=\|\Delta_0\|_2$ for a fixed starting point. Bounds for random initialization follow by tower law and taking expectations. Constants may depend on this point but are uniform over the stated learning-rate range. Weak convergence is in the Skorokhod $J_1$ topology.

\begin{assumption}\label{ass:convex}
$G$ is twice continuously differentiable and $\mu$-strongly convex. The observations $X_k$ are i.i.d., independent of initialization, and $\mathbb E[\nabla g(\theta,X)]=\nabla G(\theta)$.
\end{assumption}
\begin{assumption}\label{ass:moments}
$g(\theta,x)$ is continuously differentiable in $\theta$. When we make noise moment order assumptions for order $r\ge2$ we mean
\[
\begin{gathered}
M_r=\|\nabla g(\theta^*,X)\|_r<\infty,\\
\|\nabla g(\theta_1,X)-\nabla g(\theta_2,X)\|_r
\le L_r\|\theta_1-\theta_2\|_2.
\end{gathered}
\]
\end{assumption}
\begin{assumption}\label{ass:smooth}
For some $L,\kappa>0$,
$\|\nabla^2G(\theta)-\nabla^2G(\theta^*)\|_2\le L\|\theta-\theta^*\|_2$
when $\|\theta-\theta^*\|_2\le\kappa$.
\end{assumption}
Let $C_L=\max\{L,2L_2/\kappa\}$. The global gradient bound and the local Hessian bound then give $\|\nabla^2G(\theta)-A\|_2\le C_L\|\theta-\theta^*\|_2$ for all $\theta$.
For the bias and Richardson-Romberg extrapolation results, we also assume $G$ has a globally Lipschitz third derivative and our moment conditions hold at order 4. Enlarge $L_4$ to bound this derivative Lipschitz constant as well. For the third derivative tensors, we use the norm $\|T\|_2:=\sup_{H=H^\top,\,\|H\|_2\le1}\|T[H]\|_2$.
The third derivative assumption is not used for the trajectory KMT/FCLT.

\begin{assumption}[Only for Hessian-based covariance estimation]\label{ass:cov-new}
The sample Hessian $J(\theta,X)=\nabla^2g(\theta,X)$ exists,
$\mathbb E J(\theta^*,X)=A$, $\mathbb E\|J(\theta^*,X)\|_2<\infty$, and
\[
\mathbb E\|J(\theta,X)-J(\theta^*,X)\|_2
\le L_J\|\theta-\theta^*\|_2
\]
for every $\theta$.
\end{assumption}
Some applications satisfying these assumptions include online ridge regression, online \(\ell_2\) regularized logistic regression, and the portfolio optimization problem described in Section~\ref{portfolio example}, with bounded covariates for the regression models and the above moment conditions on the noise.

Define
\[
\begin{gathered}
A=\nabla^2G(\theta^*),\quad Z_k=\nabla g(\theta^*,X_k),\\
S=\mathbb E Z_kZ_k^\top,\quad A\Gamma+\Gamma A=S,
\end{gathered}
\]
and let $\lambda_*,\lambda^*$ be the smallest and largest eigenvalues of $A$ respectively. For every moment we assume of order r and step size $\eta$, let $\alpha_0$ be a fixed number satisfying
\begin{equation}\label{eq:steps-main}
\begin{gathered}
0<\alpha_0\eta<1,\quad
\alpha_0\le2\mu-(6r-5)L_r^2\eta,\\
\eta L_2\le\tfrac12,\quad \eta\le1.
\end{gathered}
\end{equation}
Set $\lambda=\min(\mu,\lambda_*,\alpha_0)$. These conditions on the step size apply for all SGD learning rates.
We use \emph{quenched} to refer to validity of our theorems for any fixed initialization. $B(t)$ refers to standard Brownian motion.

\section{Accumulated errors: bias and strong approximation}\label{sec:accumulated}
The moment and functional dependence bounds of \citet{NEURIPS2025_7aae9e3e} yield
\begin{align}
\|\Delta_j\|_r^2&\le\rho_r^2(1-\alpha_0\gamma)^j+
3(r-1)M_r^2\gamma/\alpha_0,\label{eq:mom-main}\\
\|\Delta_j-\Delta_j^{(t)}\|_r&\le C_{4,r}\gamma(1-\alpha_0\gamma)^{(j-t)/2},    t\le j,\label{eq:dep-main}
\end{align}
where $\Delta_j^{(t)}$ replaces $X_t$ by an independent copy, and
$C_{4,r}=2\sqrt2\{M_r^2+L_r^2[\rho_r^2+3(r-1)M_r^2\bar\gamma/\alpha_0]\}^{1/2}$
for a fixed upper bound $\bar\gamma$ on all step sizes.
Let $Q=I_d-\gamma A$. We consider a linearization of the SGD error as
\begin{equation}\label{eq:decomp-main}
\Delta_k=Q^k\Delta_0-L_k+I_{1,k}+I_{2,k},
\end{equation}
where $L_k=\gamma\sum_{j\leq k}Q^{k-j}Z_j$(the linearized iterates),
$I_{1,k}=\gamma\sum_{j\leq k}Q^{k-j}R_j$,
and
$I_{2,k}=\gamma\sum_{j\leq k}Q^{k-j}D_j$ for $R_j=A\Delta_{j-1}-\nabla G(\theta_{j-1})$, the nonlinear error, and $D_j=\nabla G(\theta_{j-1})-\nabla g(\theta_{j-1},X_j)+Z_j$, the stochastic error. 
Furthermore, let $b=-\tfrac12\nabla^3G(\theta^\star)[\Gamma]$ and $R_{\mathrm{det}}(x)=\nabla G(\theta^*+x)-Ax$, and let $x(s)$ be the solution to the ODE $\dot x(s)=-\nabla G(\theta^*+x(s))$, with $x(0)=\Delta_0$.
We further define a coupled Gaussian process with i.i.d. $Z_j^\diamond\sim N(0,S)$, $L_k^\diamond:=\gamma\sum_{j\leq k}Q^{k-j}Z^\diamond_j$ and
\begin{equation}\label{eq:bias-proxy-main}
\Delta_k^\diamond=Q^k\Delta_0-L_k^\diamond+
\gamma\sum_{j=1}^kQ^{k-j}\{\gamma b-R_{det}(x((j-1)\gamma))\}.
\end{equation}
This is a Gaussian process plus a deterministic initialization correction.

\begin{theorem}[Bias KMT]\label{thm:sum-main}
Suppose our moment conditions hold at some orders $p_1>2$, $p_2>2$, and $2p_3>4$. Assume $0<\gamma\le
\min\left\{
1,\frac{1}{2L_2},
\min_{r\in\{2,4,p_2,2p_3\}}
\frac{\mu}{(6r-5)L_r^2}
\right\}.$ For constants $C,C_i$ independent of $n$ and $\gamma$, there exists (see the appendix for explicit $y_0$) $y_0\le C(1+n\gamma^{3/2}+n\gamma^2)$ such that for every $x>0$,
\begin{align}
&\mathbb P\left\{\max_{m\le n}\left\|\sum_{k=1}^m(\Delta_k-\Delta_k^\diamond)\right\|_2>y_0+x\right\}\notag\\
&\quad\le C_1\frac n{x^{p_1}}+
C_2\frac{(\gamma^{-1}+n\gamma)^{p_2/2}}{x^{p_2}}\notag\\
& +C_3\frac{(\gamma^{-1/2}+\sqrt{n\gamma})^{p_3}}{x^{p_3}}.\label{eq:sum-main}
\end{align}
\end{theorem}
\begin{proof}[Proof outline]
The full constants, decomposition and proof are in Theorem \ref{thm:bias_kmt}. Using our linearization, we first perform a Fubini sum swap on all terms. Then we bound the accumulated $I_1$ and $I_2$ terms by a maximum of the partial sums. For the martingale stochastic error, Doob's inequality followed by Rio's inequality(\cite{Rio2009}) gives the middle term in \eqref{eq:sum-main}. Define $\mathcal P_tY=\mathbb E[Y\mid X_1,\ldots,X_t]
-\mathbb E[Y\mid X_1,\ldots,X_{t-1}]$. For the centered Taylor remainder, at each fixed lag $\ell$, the projections $\mathcal P_tR_{t+\ell}$ form a martingale difference sequence in $t$. Minkowski and Doob's inequalities then give
\[
\left\|\max_{m\le n}\left\|\sum_{j=1}^m\widetilde R_j\right\|_2\right\|_p
\le\frac p{p-1}\sum_{\ell=1}^{n-1}\left\|\sum_{t=1}^{n-\ell}\mathcal P_tR_{t+\ell}\right\|_p.
\]
Rio's inequality bounds the right side by
$p(p-1)^{-1/2}\sum_\ell(\sum_t\|\mathcal P_tR_{t+\ell}\|_p^2)^{1/2}$.
The aforementioned functional dependence bounds control the $L^p$ projection norm yielding the third term. The first term follows from the G\"otze--Zaitsev coupling of the i.i.d. noise partial sums. Markov's inequality converts these moment bounds into the three probability terms. The deterministic offset $y_0$ is obtained by bounding the bias and initialization remainders using a Taylor expansion about $\mathbb E\Delta_j$.
\end{proof}

\begin{theorem}\label{thm:fclts-main}
Under the same assumptions as Theorem \ref{thm:sum-main}, if $\sqrt n\gamma\to c\in(0,\infty)$, then, for every
$\epsilon\in(0,1)$, in $D([\epsilon,1],\mathbb R^d)$
\begin{equation}\label{eq:bias-limit-main}
\begin{aligned}
\frac1{\sqrt n}\sum_{k\le nt}\Delta_k
&\Rightarrow -A^{-1}S^{1/2}B(t)+ctA^{-1}b\\
& +\frac1c\int_0^\infty x(s)\,ds.
\end{aligned}
\end{equation}
where $B(t)$ is a standard Brownian motion.
\end{theorem}
The proof is given in theorem \ref{thm:fclt_init_bias}. Integrating the gradient-flow equation gives
$\int_0^\infty x(s)\,ds=A^{-1}\Delta_0-A^{-1}\int_0^\infty R_{det}(x(s))\,ds$.
\section{Burn in and Richardson--Romberg extrapolation}
\label{sec:burnin-rr-main}

We next consider two ways to obtain accumulated error limits centered at
the true optimizer without the bias and initialization contributions in equation
\eqref{eq:bias-limit-main}. Burn in discards early iterates to handle the initialization error and uses
a smaller learning rate to eliminate the accumulated stochastic bias. Richardson--Romberg extrapolation uses two runs of SGD to cancel the leading
stochastic bias term which allows for a larger
learning rate to eliminate the initialization error.

\subsection{Burn in}

We discard the first $n_0(\gamma)$ iterates and define
\[
S_m^{\mathrm{burn}}
=\sum_{k=1}^m\Delta_{n_0(\gamma)+k}
\]
and corresponding Gaussian process
\[
\Delta_k^{\diamond,\mathrm{burn}}
=Q^k\Delta_0-L_k^\diamond,
k\ge0.
\]

\begin{theorem}[Burn in KMT-type approximation]
\label{thm:burnin-kmt-main}
Suppose our moment conditions hold at orders $p>2$ and $2p'$, where
$p'\ge1$. Let $q=1-\lambda\gamma$ and define
\[
H_{n,\gamma}
=1+n\gamma+n_0(\gamma)q^{n_0(\gamma)-1}
+\frac{q^{n_0(\gamma)}}{\gamma}.
\]
Then for every
$x>0$,
\begin{align*}
&\mathbb P\left\{
\max_{m\le n}\left\|
S_m^{\mathrm{burn}}
-\sum_{k=1}^m\Delta_{n_0(\gamma)+k}^{\diamond,\mathrm{burn}}
\right\|_2>x\right\}\\
&\quad\le
C_1\frac{n+n_0(\gamma)}{x^p}
+C_2\frac{H_{n,\gamma}^{p/2}}{x^p}
+C_3\frac{H_{n,\gamma}^{p'}}{x^{p'}}.
\end{align*}
The constants are independent of $n$, $n_0(\gamma)$, and $\gamma$.
\end{theorem}
See Theorem~\ref{KMT Burn in} for a complete proof.
The terms containing $q^{n_0(\gamma)}$ decay exponentially with the
Burn in length. The $n\gamma$ term corresponds to the stochastic bias that
Burn in alone does not remove.

\begin{theorem}[Burn in FCLT]\label{thm:burnin-fclt-main}
Let $n\to\infty$ and $\gamma\to0$ satisfy
\[
\sqrt n\gamma\to0,
 ,  n\gamma\ge\log n\to\infty,
\]
and choose an integer Burn in length so that
$$
\begin{gathered}
n_0(\gamma)\ge
\frac{2}{\lambda\gamma}|\log(\sqrt n\gamma)|,\\
n_0(\gamma)=O\!\left(
\frac{|\log(\sqrt n\gamma)|}{\lambda\gamma}\right).
\end{gathered}
$$
Then
\[
\frac{1}{\sqrt n}S_{\lfloor nt\rfloor}^{\mathrm{burn}}
\Rightarrow -A^{-1}S^{1/2}B(t).
\]
\end{theorem}

The proof is given in Theorem~\ref{fclt-burn in}. Theorem \ref{thm:burnin-fclt-main}'s conditions are satisfied by $\gamma=n^{-a}$ with $1/2<a<1$ and
$n_0(\gamma)$ of order $\gamma^{-1}\log n$.

\subsection{Richardson--Romberg extrapolation}

Run SGD at learning rates $\gamma$ and $2\gamma$ with the same
initialization and observations, and define
\[
\hat\theta_k=2\theta_k(\gamma)-\theta_k(2\gamma),
 ,  \hat\Delta_k=\hat\theta_k-\theta^*.
\]
For $\eta\in\{\gamma,2\gamma\}$, let
\[
Q_\eta=I_d-\eta A,
 , 
L_k^\diamond(\eta)=\eta\sum_{j=1}^kQ_\eta^{k-j}Z_j^\diamond,
\]
using the same Gaussian vectors for both learning rates. Define
\[
\hat\Delta_k^\diamond
=\bigl(2Q_\gamma^k-Q_{2\gamma}^k\bigr)\Delta_0
-\bigl(2L_k^\diamond(\gamma)-L_k^\diamond(2\gamma)\bigr).
\]
Our step size conditions apply to both learning rates.

\begin{theorem}[Richardson--Romberg KMT-type approximation]
\label{thm:rr-kmt-main}
Suppose our moment conditions hold at orders $p_1>2$, $p_2>2$, and
$2p_3>4$. Then for constants $C_i,C$ and a deterministic offset(see the appendix for explicit $y_0$)
\[
0\le y_{0,\mathrm{RR}}
\le C\bigl(\gamma^{-1}+1+n\gamma^{3/2}+n\gamma^2\bigr)
\]
we have that for every $x>0$,
\begin{align*}
&\mathbb P\left\{
\max_{m\le n}\left\|
\sum_{k=1}^m(\hat\Delta_k-\hat\Delta_k^\diamond)
\right\|_2>y_{0,\mathrm{RR}}+x\right\}\\
&\quad\le C_1\frac{n}{x^{p_1}}
+C_2\frac{(\gamma^{-1}+n\gamma)^{p_2/2}}{x^{p_2}}\\
& +C_3
\frac{(\gamma^{-1/2}+\sqrt{n\gamma})^{p_3}}{x^{p_3}}.
\end{align*}
The constants are independent of $n$ and $\gamma$.
\end{theorem}

The full bound and proof are given in Theorem~\ref{KMT_Romberg}. 

\begin{theorem}[Richardson--Romberg FCLT]
\label{thm:rr-fclt-main}
Let $n\to\infty$ and $\gamma\to0$ satisfy
\[
\sqrt n\gamma\to\infty,
  n\gamma^3\to0.
\]
Then
\begin{equation}\label{eq:rr-limit-main}
\frac1{\sqrt n}\sum_{k=1}^{\lfloor nt\rfloor}\hat\Delta_k
\Rightarrow -A^{-1}S^{1/2}B(t)
\quad\text{in }D([0,1],\mathbb R^d).
\end{equation}
\end{theorem}

The proof is given in Theorem~\ref{romberg fclt}. The conditions
permit $\gamma=n^{-a}$ with $1/3<a<1/2$. Richardson-Romberg extrapolation eliminates the
stochastic bias term of order $\sqrt n\gamma$ and our learning conditions force
the remaining order $\sqrt n\gamma^{3/2}$ term to vanish. The
$\sqrt n\gamma\to\infty$ condition causes the initialization contribution to vanish without discarding iterates. Sharing noise observations between runs results in the limiting $A^{-1}SA^{-1}$ covariance of one run whereas independent runs result in the same result but with an inflated limiting covariance of $5A^{-1}SA^{-1}$.

\section{Strong approximation of the iterate trajectory}\label{sec:trajectory}

Consider a fixed horizon $T$, and let
\[
\begin{gathered}
n_0(\gamma)=2\left\lceil\frac\nu{2\gamma}\log(1/\gamma)\right\rceil,\quad
\nu>3/\min(\lambda_*,\alpha_0),\\
Y_\gamma(t)=\frac{\Delta_{n_0(\gamma)+\lfloor t/\gamma\rfloor}}{\sqrt\gamma}.
\end{gathered}
\]
\begin{theorem}[Trajectory KMT-type approximation]\label{thm:trajectory-main}
Assume assumptions \ref{ass:convex}, \ref{ass:moments}
at order $r=2$, and \ref{ass:smooth} hold. Also assume the
step-size restrictions of \eqref{eq:steps-main} at order $r=2$
and $M_p<\infty$ for some $p>2$. Then we can construct a stationary OU process on the same space,
\[
dX_t=-AX_t\,dt+S^{1/2}\,dB_t,    X_0\sim N(0,\Gamma),
\]
with $X_0$ independent of $B$ and
$A\Gamma+\Gamma A=S$,
such that
\begin{equation}\label{eq:trajectory-main}
\begin{aligned}
&\sup_{0\le t\le T}\|Y_\gamma(t)-X_t\|_2\\
&\quad=O_{\mathbb P}\left(\gamma^{1/2-1/p}\log^{1/p}(1/\gamma)
+\sqrt\gamma\log(1/\gamma)\right).
\end{aligned}
\end{equation}
 A finite sample probability bound and the proof of the result are provided in Theorem \ref{thm:trajectory-kmt-new}.
\end{theorem}

\begin{theorem}[Trajectory FCLT]
\label{thm:trajectory-fclt-main}
Assume Assumptions~\ref{ass:convex}--\ref{ass:smooth} 
and that the step size conditions in ~\eqref{eq:steps-main}
hold at order $r=2$. With the same Burn in defined above and $\nu>3/\lambda$, as $\gamma\to0$,
\[
Y_\gamma(\cdot)\Rightarrow X(\cdot)
\quad\text{in }D([0,T],\mathbb R^d),
\]
where $X$ is the same OU process as in \ref{thm:trajectory-main}
\end{theorem}

Unlike its quantitative KMT approximation counterpart, this result
requires only second moments of the noise. The proof is given in
Theorem~\ref{thm:trajectory_fclt}.

\section{Some example applications}\label{sec:applications-main}
Suppose we  wish to investigate the probability that the SGD trajectory crosses some boundary. This is useful to assess risk probabilities such as bankruptcy in finance. Our previous results allow us to transfer this question to well studied properties of Gaussian processes. In particular, the idea is that for stochastic processes $U$ and $V$, if we have a maximal bound like $\mathbb P(\|U-V\|_\infty>r)\le\delta$ then for any unit vector $a$ and deterministic path $h$, letting the pathwise error on this projection be $M(U)=\sup_{t\le T}\{a^\top U_t-h(t)\}$, then
\begin{equation}\label{eq:cross-main}
\begin{aligned}
\mathbb P\{M(V)>r\}-\delta&\le\mathbb P\{M(U)>0\}\\
&\le\mathbb P\{M(V)>-r\}+\delta.
\end{aligned}
\end{equation}
In our case $V$ is a nice Gaussian process which allows the upper and lower bounds to be computed analytically or by numerical simulation (i.e. well known Monte Carlo methods). The maximum bound also enables us to obtain simultaneous confidence intervals valid even if we choose an iterate after observing the run.

The trajectory result also allows us to control regret post Burn in. For $T=J\gamma$, we define regret in this paper as
\[
\mathcal L_\gamma(T)=\sum_{k=n_0(\gamma)}^{n_0(\gamma)+J-1}
\{G(\theta_k)-G(\theta^*)\}.
\]
\begin{corollary}\label{cor:loss-main}
For fixed $T$, letting $X_t$ be the same OU process as before then $\mathcal L_\gamma(\lfloor T/\gamma\rfloor\gamma)
\Rightarrow\frac12\int_0^T X_t^\top AX_t\,dt$.
Quantitatively on $\{\|Y_\gamma-X\|_\infty\le r,\ \|X\|_\infty\le M\}$ if \(\sqrt{\gamma}(M+r)\le\kappa,\)
\[
\begin{aligned}
&\left|\mathcal L_\gamma(T)-\frac12\int_0^T X_t^\top AX_t\,dt\right|\\
&\quad\le T\left\{\frac{\lambda^*}{2}r(2M+r)
+\frac L6\sqrt\gamma(M+r)^3\right\}.
\end{aligned}
\]
\end{corollary}
The proof is provided in the appendix, theorem \ref{thm:loss}.

\section{Estimation}
To actually conduct inference and compute the aforementioned probabilities, we can learn the parameters (i.e. covariance) of the Gaussian process from the run or past data. To this end, with any deterministic Burn in $b_n\ge0$, let $H_k=\nabla g(\theta_{b_n+k-1},X_{b_n+k})$ and define our estimators
\begin{equation}\label{eq:cov-main-new}
\begin{aligned}
\widehat A_m&=\frac1m\sum_{k=1}^m\nabla^2g(\theta_{b_n+k-1},X_{b_n+k}),\\
\widehat S_m&=\frac1m\sum_{k=1}^m H_kH_k^\top,   
\widehat\Sigma_m=\widehat A_m^{-1}\widehat S_m\widehat A_m^{-1}.
\end{aligned}
\end{equation}

\begin{theorem}\label{thm:cov-main-new}
Under our second moment conditions and Assumption~\ref{ass:cov-new}, if $\gamma\to0$ and $n\gamma\to\infty$, then, for each fixed $\tau\in(0,1)$, $\widehat A_m,\widehat S_m,\widehat\Sigma_m$ converge in probability to $A,S,\Sigma$, uniformly for $\lceil\tau n\rceil\le m\le n$. The solution of $\widehat A_m\widehat\Gamma_m+\widehat\Gamma_m\widehat A_m=\widehat S_m$ similarly estimates $\Gamma$ uniformly.
\end{theorem}
See the proof in Appendix~\ref{app:feasible-new}. Consistency requires no fourth moment. Under fourth moments on the noise and $\mathbb E\|J(\theta^*,X)\|_2^2<\infty$ we also provide a rate of convergence. 

\paragraph{Simultaneous confidence intervals.}
Let $\bar\theta_m$ average the first $m$ post-Burn in
iterates or Richardson--Romberg iterates. Under the
corresponding FCLT
(Theorems~\ref{thm:burnin-fclt-main}
and~\ref{thm:rr-fclt-main}), for a fixed $v$ such that \(v^\top\Sigma v>0\),
\[
v^\top\bar\theta_m
\ \pm\
q_{\tau,1-\delta}
\sqrt{\frac{v^\top\widehat\Sigma_m v}{m}},
 ,  \lceil\tau n\rceil\le m\le n,
\]
cover $v^\top\theta^*$ simultaneously with probability
tending to $1-\delta$, under the same assumptions as \ref{thm:cov-main-new}.
Here $q_{\tau,1-\delta}$ is the $(1-\delta)$ quantile of
$\sup_{\tau\le t\le1}|B(t)|/\sqrt t$ and $B(t)$ is Brownian motion.

\paragraph{Inference for a risk probability.}
When we substitute $\widehat A,\widehat S$ into the prior section's probability bounds (for the OU model) we get consistent estimates of the previously described crossing and regret probabilities when $\Gamma\succ0$ (see Proposition~\ref{prop:risk-plugin-new} for more details and a proof). The KMT gives finite time probability bounds on our approximation. For example, with these two results, one can obtain finite sample confidence intervals for crossing and regret tail probabilities as defined in the previous section (see Proposition~\ref{prop:risk-CI-new} for more details). 

\section{A Portfolio Optimization Example}\label{portfolio example}
We use the portfolio allocation framework of
\citet{Boyd2017Trading}, which allows for short positions and borrowing.Stochastic gradient methods have been used for portfolio optimization without borrowing or short selling constraints in the literature(see for example \citep{KopeliovichPokojovy2024}). In this framework we have i.i.d. noise $R_j$ (returns) with mean $\mu_R$, covariance $S$,
and a finite $p$-th moment for some $p>2$, and
\begin{align*}
g(\theta,R_j)&=\tfrac12\theta^\top A\theta-\theta^\top R_j
\end{align*}
We penalize risky positions with our quadratic penalty using risk matrix $A\succ0$. We assume this matrix is known or estimated a priori.
$\theta_j$ represents our dollar holdings in various assets and
negative components are short positions. After a Burn in of  $b$ iterates,
let $W_m$ be total wealth and $C_m$ the cash balance for
\begin{align*}
C_m&=W_m-\mathbf1^\top\theta_{b+m},\\
W_m&=W_{m-1}+\theta_{b+m-1}^\top R_{b+m}\\
&\quad+r_{\rm lend}(C_{m-1})_+
-r_{\rm borrow}(-C_{m-1})_+.
\end{align*}
Here $0\le r_{\rm lend}\le r_{\rm borrow}$ are  interest
rates.

Using our existing KMT theory, Proposition~\ref{prop:wealth-new} constructs a simulatable wealth process
$W_m^\diamond$ from the Gaussian process and proves the below rate. Let our error tolerance be $e_\delta$,
$p^\diamond(z)=\mathbb P\{\min_{0\le m\le n}W_m^\diamond\le z\}$, then
\begin{align*}
p^\diamond(h-e_\delta)-\delta
&\le\mathbb P\{\min_{0\le m\le n}W_m\le h\}\\
&\le p^\diamond(h+e_\delta)+\delta.
\end{align*}
 $h=0$ gives bankruptcy probability bounds for our SGD trading policy.
The reference probabilities can be simulated using estimated $\mu_R,S$. See Appendix~\ref{app:portfolio-wealth} for more details and a proof. To the best of our knowledge, we present the first such bankruptcy probability bounds for constant rate SGD in this setting.

\section{Experiments}

\begin{figure*}[t]
    \centering
    \includegraphics[width=\textwidth]{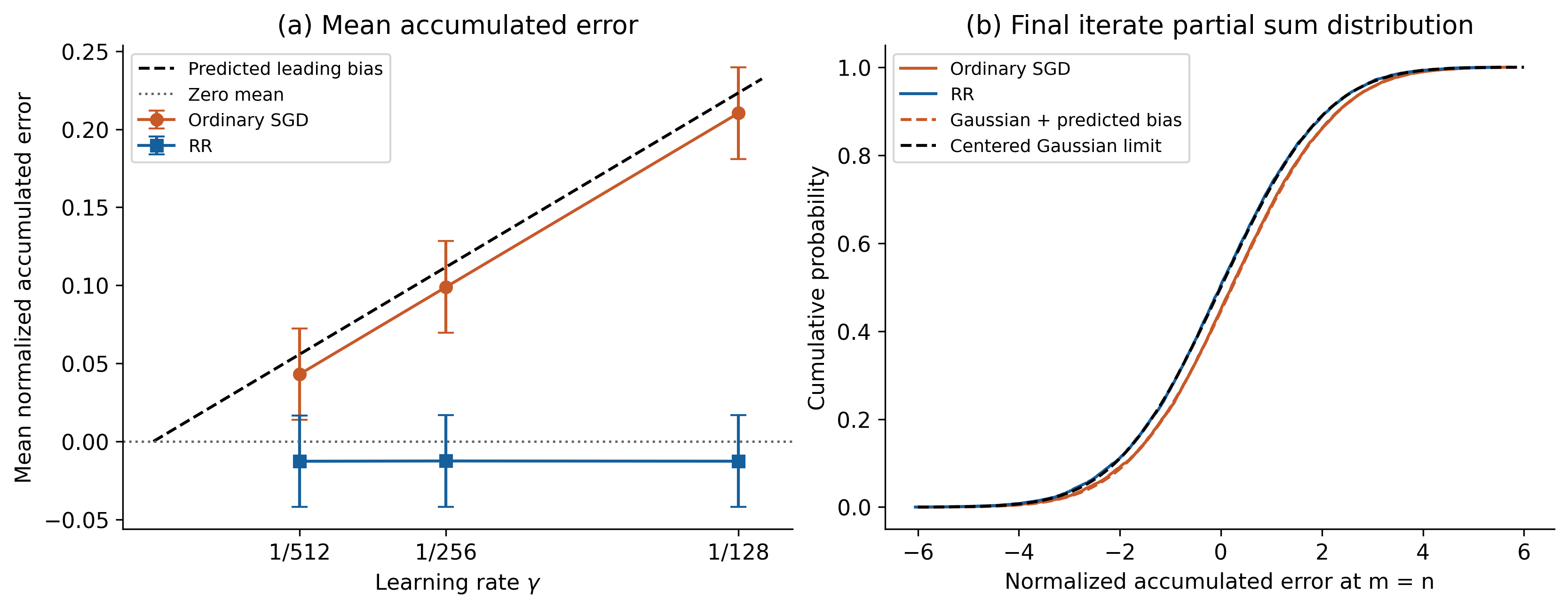}
    \caption{Left: mean normalized accumulated errors against learning rate, bands are 95\% Monte Carlo intervals.
    Right: final iterate partial sum distribution at $m=n$ and $\gamma=1/128$,
    compared with the Gaussian limits. Ordinary SGD and RR start at the optimizer.}
    \label{fig:logistic-bias-updated}
\end{figure*}

\begin{figure}[t]
\centering
\includegraphics[width=\linewidth]{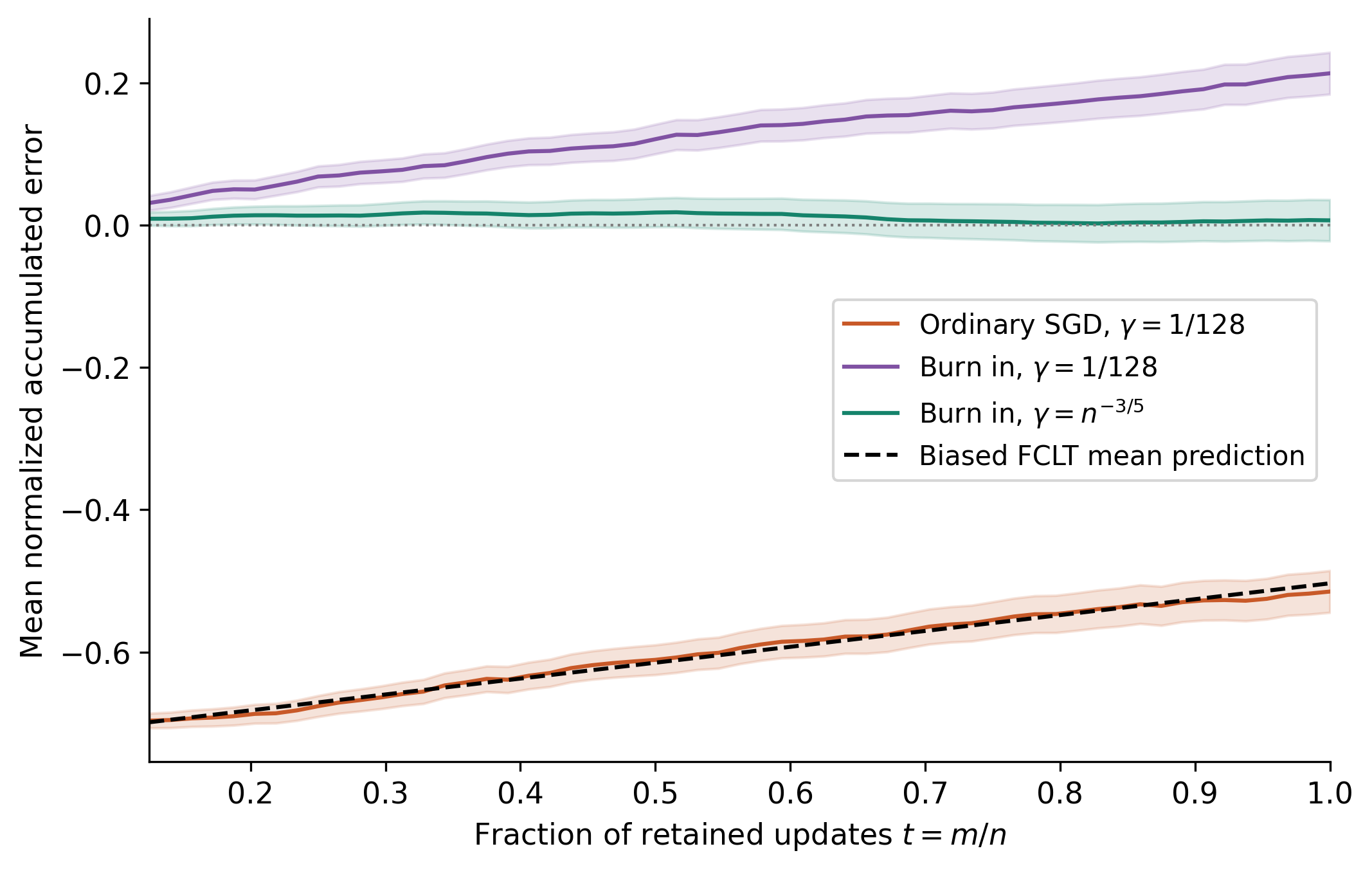}
\caption{Mean normalized partial sums against $t=m/n$ for ordinary SGD
and burn in at two learning rates. The dashed curve is the biased FCLT predicted
mean bias. Shaded regions are pointwise 95\% Monte Carlo intervals.}
\label{fig:logistic-initialization-updated}
\end{figure}

\begin{figure}[t]
\centering
\includegraphics[width=\linewidth]{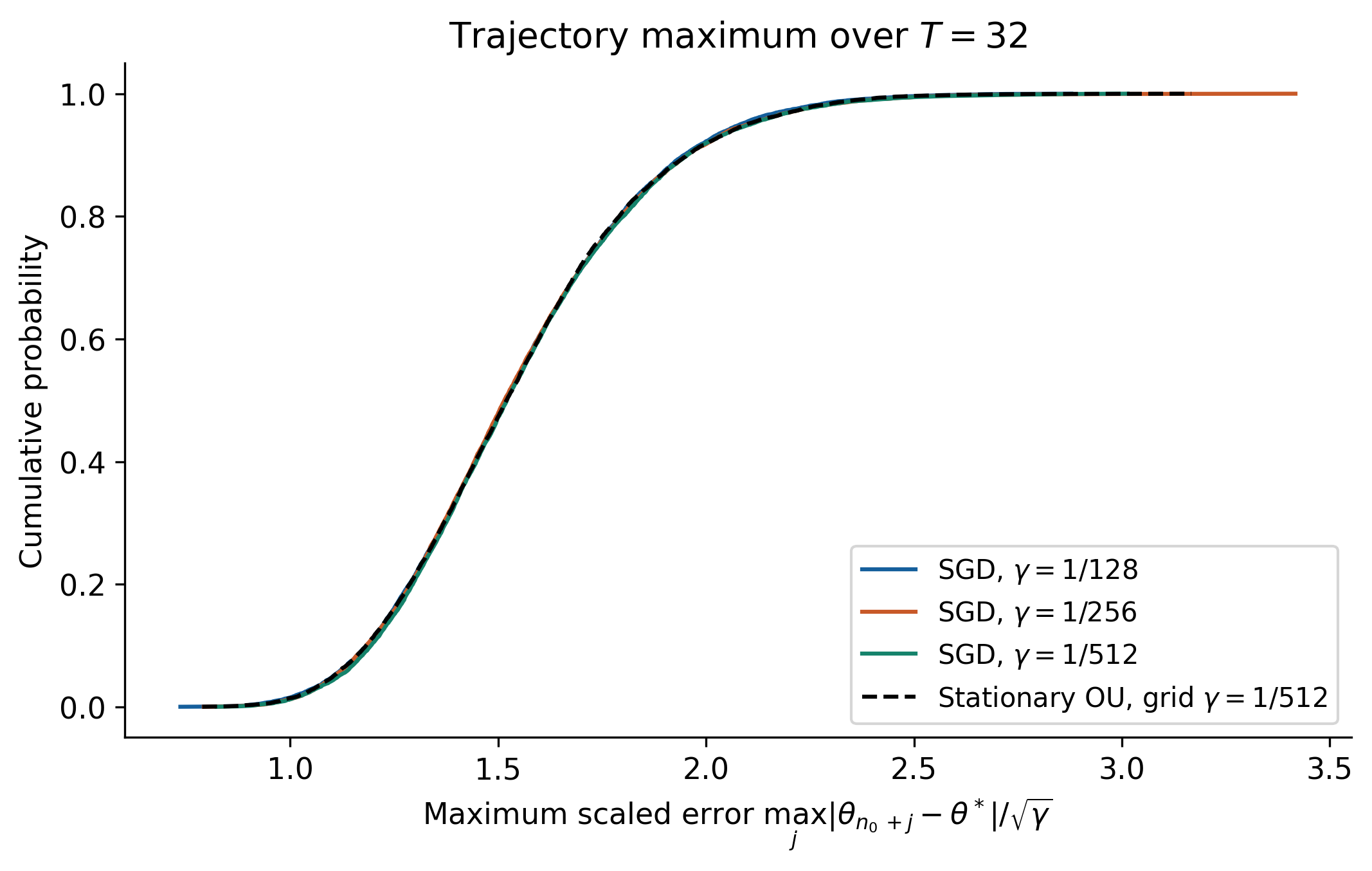}
\caption{Empirical distributions of the maximum scaled SGD error over
$T=32$, compared with the maximum of the stationary OU iterates on the grid, $\gamma=1/512$.}
\label{fig:logistic-trajectory-updated}
\end{figure}

We focus on synthetic simulations because our results give theoretical probability estimates. Therefore, to assess the accuracy of our results we must be able to compare to the ground truth probabilities which are inherently unavailable in real world data.

For our simulation we choose to study regularized logistic regression with heavy tailed
observation weights, to demonstrate our theory in the heavy tailed setting. See
Appendix~\ref{app:logistic-experiment} for more details on the simulation experiment we ran.
In our specific model, the averaged SGD iterates (Polyak-Ruppert) have an asymptotic SGD bias of
$ \gamma A^{-1}b=0.056 \gamma$ in the notation of Theorem \ref{thm:fclts-main}. We first focus on the effect of Richardson--Romberg extrapolation to remove the asymptotic bias term. We initialize the SGD runs at
the optimizer and compare their normalized accumulated errors with the
Gaussian process limits described in Section~\ref{sec:accumulated}.
We use 12,000 Monte Carlo replications with $n=512^2$ retained iterates
per run, after burn in where applicable. 
Figure~\ref{fig:logistic-bias-updated} shows the mean accumulated errors
and the final iterate partial sum distribution, evaluated at $m=n$. Both are normalized, that is divided, by $\sqrt{n}$. Note that the mean normalized accumulated error should be zero without bias since Brownian motion in the limit has mean zero. We see that clearly Richardson Romberg removes the leading order bias and that asymptotic bias also diminishes with smaller learning rates. Our theory predicts the average bias reduction in the mean
of the normalized accumulated error to be $0.22335$ and indeed our Monte Carlo estimate is $0.22298$, 95\% Monte Carlo confidence interval half-width of $0.00040$. The small downward shift is within the Monte Carlo error and since we reuse noise observations for the different runs there is a common negative shift.

We also investigate the effect of initialization at $\gamma=1/128$ as before(constant) and a learning rate scaling in accordance with our FCLT theorem($\gamma=n^{-3/5}$). Figure~\ref{fig:logistic-initialization-updated} plots the Monte Carlo
means of ordinary SGD and the two burn in sequences,
\[
\frac1{\sqrt n}\sum_{k=1}^{\lfloor nt\rfloor}(\theta_k-\theta^*),
 , 
\frac1{\sqrt n}\sum_{k=1}^{\lfloor nt\rfloor}
(\theta_{n_0+k}-\theta^*),
\]
together with the biased FCLT prediction for the ordinary SGD mean. The shaded regions are pointwise 95\% Monte Carlo confidence intervals
for the means. We see that burn in removes the negative  initialization term but without a small enough rate, the positive bias begins to dominate. Only with both the small enough rate and burn in do we obtain a near zero mean.

Finally, for the trajectory FCLT we compare the distribution of the maximum over the horizon \(T=32\), and
\[
n_0=2\left\lceil
\frac{64}{2\gamma}\log(1/\gamma)
\right\rceil.
\]
Here $\mu=\alpha_0=1/16$, so
$\nu=64>3/\min(A,\alpha_0)=48$.
We compare the maximum of
$|\theta_{n_0+j}-\theta^*|/\sqrt\gamma$ over
$j=0,\ldots,T/\gamma$ with the maximum of the corresponding stationary
OU process.
After burn in, we observe $T/\gamma$ SGD updates with
$T=32$, which is 4,096, 8,192, and 16,384 iterates for 
\(\gamma=1/128\),
\(1/256\), and
\(1/512\) respectively.
The respective empirical KS distances between
maximum distributions are $0.01342$, $0.00650$, and
$0.01125$. We see that the OU limiting distribution for the maximum and true SGD maximum distributions are extremely close as predicted. For more details on the experiment see
Appendix~\ref{app:logistic-experiment}.

We also conduct a few additional experiments to demonstrate the quality of our Gaussian approximations for the same weighted logistic regression model and bankruptcy bounds. In Appendix~\ref{app:empirical-coupling}, we display graphs showing the Monte Carlo estimated maximum (over time steps) approximation error between the SGD sum/trajectory and respective Gaussian process for the bias KMT, Richardson--Romberg KMT, burn in KMT and trajectory KMT against the number of iterates. These are all normalized according to the respective KMT. See Appendix~\ref{app:empirical-coupling} for more details. 

The normalized maximum errors show good approximation and decrease with the number of retained iterates. Our KMT theorems upper bounds provide reference rates. Our graphs also show these normalized maximum errors display the dependence on run length from the KMT upper bound. For the OU comparison, the 95th percentile of the maximum error decreases from $0.781$ at $n=512$ to $0.333$ at $n=16{,}384$, while its root mean square decreases from $0.627$ to $0.273$. In the portfolio allocation example we simplify the problem and set  borrowing and lending rates to zero. Then the  Gaussian process derived bankruptcy probabilities differ from the Monte Carlo simulation estimates by at most $0.0034$ within the run. For example, at $n=4{,}096$, the simulated and Gaussian process bankruptcy probabilities are $23.972\%$ and $23.978\%$, respectively. 

\section{Discussion}
This paper establishes strong Gaussian approximations for the accumulated iterates of SGD as well as the trajectory of SGD itself by bounding the maximum difference between the Gaussian process limit and SGD sequence in probability. Bias is shown to be a fundamental problem in the accumulated error setting which we resolve by Richardson--Romberg extrapolation and Burn in. We also provide consistent estimators for the parameters (covariance) of the Gaussian process allowing for inference using an existing run or past data. As opposed to simple final iterate inference, our results allow for trajectory level inference. We have provided an example of the usefulness of these results by obtaining bankruptcy probabilities for SGD in a portfolio allocation problem. Future directions to consider include extending our results to incorporate dependent noise, infinite horizons in the trajectory case, and nonconvex problems.

\clearpage

\bibliographystyle{plainnat}
\bibliography{references}

\clearpage

\onecolumn
\etocdepthtag.toc{appendix}
\etocsettagdepth{main}{none}
\etocsettagdepth{appendix}{subsection}
\tableofcontents
\appendix
\clearpage

\section{FCLT sum variants}
For all subsequent proofs we impose the assumptions  of Section~\ref{sec:assumptions}.
 Let $\gamma \to 0$ and $n \to \infty$ such that $n \gamma \geq \log(n)$, $n\gamma \to \infty$. We recall the linearization decomposition introduced in \cite{NEURIPS2025_7aae9e3e}, which is that the estimation error $\Delta_k = \theta_k - \theta^*$ satisfies the recursion
\begin{align*}
\Delta_k = (I_d - \gamma A)^k \Delta_0 - L_k + I_{1,k} + I_{2,k},
\end{align*}
where $L_k = \gamma \sum_{j=1}^k (I_d - \gamma A)^{k-j} Z_j$ is the linearized sequence, $I_{1,k} = \gamma \sum_{j=1}^k (I_d - \gamma A)^{k-j} R_j$ with $R_j = A\Delta_{j-1} - \nabla G(\theta_{j-1})$, and $I_{2,k} = \gamma \sum_{j=1}^k (I_d - \gamma A)^{k-j} D_j$ with $D_j = \nabla G(\theta_{j-1}) - \nabla g(\theta_{j-1}, X_j) + \nabla g(\theta^*, X_j)$.

\begin{lemma}\label{linear covariance expansion}
Let $\Gamma$ be the solution to  $A\Gamma + \Gamma A = S$. Then

\begin{align*}
\mathbb{E}[L_m L_m^\top] &= \gamma \left( \Gamma - (I_d - \gamma A)^m \Gamma (I_d - \gamma A)^m \right) + \gamma^3 \sum_{i=0}^{m-1} (I_d - \gamma A)^i A \Gamma A (I_d - \gamma A)^i.
\end{align*}
\end{lemma}

\begin{proof}

\begin{align*}
&(I_d - \gamma A) \Gamma (I_d - \gamma A) = \Gamma - \gamma A \Gamma - \gamma \Gamma A + \gamma^2 A \Gamma A \\
&= \Gamma - \gamma (A \Gamma + \Gamma A) + \gamma^2 A \Gamma A \\
&= \Gamma - \gamma S + \gamma^2 A \Gamma A \\
&\implies \gamma S = \Gamma - (I_d - \gamma A) \Gamma (I_d - \gamma A) + \gamma^2 A \Gamma A.
\end{align*}

Now because the $Z_i$ are independent and mean-zero

\begin{align*}
&\mathbb{E}[L_m L_m^\top] = \gamma^2 \sum_{i=1}^m (I_d - \gamma A)^{m-i} S (I_d - \gamma A)^{m-i} \\
&= \gamma \sum_{i=0}^{m-1} (I_d - \gamma A)^i (\gamma S) (I_d - \gamma A)^i \\
&= \gamma \sum_{i=0}^{m-1} (I_d - \gamma A)^i \Big[ \Gamma - (I_d - \gamma A) \Gamma (I_d - \gamma A) + \gamma^2 A \Gamma A \Big] (I_d - \gamma A)^i \\
&= \gamma \sum_{i=0}^{m-1} \Big( (I_d - \gamma A)^i \Gamma (I_d - \gamma A)^i - (I_d - \gamma A)^{i+1} \Gamma (I_d - \gamma A)^{i+1} \Big) \\
&\quad + \gamma^3 \sum_{i=0}^{m-1} (I_d - \gamma A)^i A \Gamma A (I_d - \gamma A)^i.
\end{align*}

The first sum is a telescoping series so

\begin{align*}
\mathbb{E}[L_m L_m^\top] &= \gamma \left( \Gamma - (I_d - \gamma A)^m \Gamma (I_d - \gamma A)^m \right) + \gamma^3 \sum_{i=0}^{m-1} (I_d - \gamma A)^i A \Gamma A (I_d - \gamma A)^i.
\end{align*}

\end{proof}

\begin{lemma} \label{Sum stochastic error to 0}
Let $I_{2,k} = \gamma \sum_{j=1}^k (I_d - \gamma A)^{k-j} D_j$. If $n\gamma \to \infty$ and $\gamma \to 0$, then $$\max_{1 \le m \le n} \frac{1}{\sqrt{n}} \left\| \sum_{k=1}^m I_{2,k} \right\|_2 \xrightarrow{\mathbb{P}} 0.$$

\end{lemma}
\begin{proof}
By Fubini's theorem

\begin{align*}
& \sum_{k=1}^m I_{2,k}
= \gamma \sum_{k=1}^m \sum_{j=1}^k (I_d - \gamma A)^{k-j} D_j \\
&= \gamma \sum_{j=1}^m \left(\sum_{k=j}^m (I_d - \gamma A)^{k-j}\right)D_j \\
& = \sum_{j=1}^m A^{-1} \bigl(I_d - (I_d - \gamma A)^{m-j+1}\bigr)D_j \\
& = A^{-1} \sum_{j=1}^m D_j
- A^{-1} \sum_{j=1}^m (I_d - \gamma A)^{m-j+1} D_j .
\end{align*}
where we can invert $A := \nabla^2 G(\theta^*)$ since we assume strong convexity.

Let $\tilde{M}_m = \sum_{j=1}^m D_j$. By Doob's maximal $\mathcal{L}^2$ inequality since the $D_j$ are martingale differences and hence uncorrelated

\begin{align*}
\mathbb{E}\left[ \max_{1 \le m \le n} \|\tilde{M}_m\|_2^2 \right] &\le 4 \sum_{j=1}^n \mathbb{E}\|D_j\|_2^2 \\
&\le 16 L_2^2 \sum_{j=1}^n \left( \rho_2^2(1-\alpha_0\gamma)^{j-1} + \frac{3M_2^2}{\alpha_0}\gamma \right) \\
&\le 16 L_2^2 \left( \rho_2^2 \sum_{j=0}^\infty (1-\alpha_0\gamma)^j + \frac{3M_2^2}{\alpha_0} n \gamma \right) \\
&= 16 L_2^2 \left( \frac{\rho_2^2}{\alpha_0 \gamma} + \frac{3M_2^2}{\alpha_0} n \gamma \right).
\end{align*}
since by the stochastic Lipschitzness assumption on the gradients(Assumption 2.2) and Theorem 2.3 of \cite{NEURIPS2025_7aae9e3e} $\mathbb{E}\|D_j\|_2^2 \le 4L_2^2 \mathbb{E}\|\Delta_{j-1}\|_2^2\le 4L_2^2(\rho_2^2(1-\alpha_0\gamma)^{j-1} + \frac{3M_2^2}{\alpha_0}\gamma)$.

so

\begin{align*}
\mathbb{E}\left[ \max_{1 \le m \le n} \frac{1}{n} \|\tilde{M}_m\|_2^2 \right] &\le \frac{16L_2^2 \rho_2^2}{\alpha_0 n \gamma} + \frac{48L_2^2 M_2^2}{\alpha_0} \gamma \rightarrow 0
\end{align*}
as $n \to \infty$ and $\gamma \to 0$ simultaneously with $n\gamma \to \infty$. So, by Markov's inequality, \[\mathbb{P}\left(\max_{1 \le m \le n} \frac{1}{\sqrt{n}} \|\tilde{M}_m\|_2 > \epsilon\right) \to 0\].

We now deal with the second term which we rewrite as
\begin{align*}
\sum_{j=1}^m (I_d - \gamma A)^{m-j+1} D_j &= \sum_{j=1}^m (I_d - \gamma A)^{m-j+1} (\tilde{M}_j - \tilde{M}_{j-1}) \\
&= \sum_{j=1}^m (I_d - \gamma A)^{m-j+1} \tilde{M}_j - \sum_{j=1}^{m-1} (I_d - \gamma A)^{m-j} \tilde{M}_j \\
&= (I_d - \gamma A) \tilde{M}_m - \gamma A \sum_{j=1}^{m-1} (I_d - \gamma A)^{m-j} \tilde{M}_j.
\end{align*}
Noting that $(I_d - \gamma A) \tilde{M}_m = \gamma A \sum_{j=1}^{m-1} (I_d - \gamma A)^{m-j} \tilde{M}_m + (I_d - \gamma A)^m \tilde{M}_m$, we substitute this back in to get
\begin{align*}
\sum_{j=1}^m (I_d - \gamma A)^{m-j+1} D_j &= (I_d - \gamma A)^m \tilde{M}_m + \gamma A \sum_{j=1}^{m-1} (I_d - \gamma A)^{m-j} (\tilde{M}_m - \tilde{M}_j) \\
&= (I_d - \gamma A)^m \tilde{M}_m + \gamma A \sum_{k=0}^{m-2} (I_d - \gamma A)^{k+1} (\tilde{M}_m - \tilde{M}_{m-1-k}).
\end{align*}
so
\begin{align*}
\frac{1}{\sqrt{n}} \sum_{k=1}^m I_{2,k} &= A^{-1} \frac{\tilde{M}_m}{\sqrt{n}} - A^{-1} \frac{(I_d - \gamma A)^m \tilde{M}_m}{\sqrt{n}} - A^{-1} \sum_{k=0}^{m-2} \gamma A (I_d - \gamma A)^{k+1} \frac{\tilde{M}_m - \tilde{M}_{m-1-k}}{\sqrt{n}}.
\end{align*} We have already shown $\frac{\tilde{M}_m}{\sqrt{n}}$ converges to 0 in probability so by a simple union bound the first two terms converge to zero in probability.

For the last term, the geometric sum and the preceding maximal estimate give
\[
\max_{m\le n}\left\|\frac{\gamma}{\sqrt n}
\sum_{k=0}^{m-2}(I_d-\gamma A)^{k+1}
(\widetilde M_m-\widetilde M_{m-1-k})\right\|_2
\le\frac{2}{\lambda_*}\max_{j\le n}\frac{\|\widetilde M_j\|_2}{\sqrt n}
\xrightarrow{\mathbb P}0.
\]
Here $A^{-1}A=I_d$ has been canceled. This proves the lemma.
\end{proof}

\begin{lemma} \label{Sum centered taylor error to 0}
Let $\tilde{R}_j = R_j - \mathbb{E}[R_j]$ and $\tilde{I}_{1,k} = \gamma \sum_{j=1}^k (I_d - \gamma A)^{k-j} \tilde{R}_j$. If $n\gamma \to \infty$ and $\gamma \to 0$, and we assume a fourth moment for the gradient noise then
$$\max_{1 \le m \le n} \frac{1}{\sqrt{n}} \left\| \sum_{k=1}^m \tilde{I}_{1,k} \right\|_2 \xrightarrow{\mathbb{P}} 0.$$
\end{lemma}
\begin{proof}
We first do the same Fubini sum trick as in the previous lemma
\begin{align*}
\sum_{k=1}^m \tilde{I}_{1,k} &= \gamma \sum_{k=1}^m\sum_{j=1}^k (I_d - \gamma A)^{k-j} \tilde{R}_j \\
&= \sum_{j=1}^m A^{-1} \left( I_d - (I_d - \gamma A)^{m-j+1} \right)\tilde{R}_j \\
&= A^{-1} \sum_{j=1}^m \tilde{R}_j - A^{-1} \sum_{j=1}^m (I_d - \gamma A)^{m-j+1} \tilde{R}_j.
\end{align*}
By the exact same proof as in lemma \ref{Sum stochastic error to 0}, since the $\tilde{R}_j$ are mean zero, it suffices to show \[\max_{1 \le m \le n} \frac{1}{\sqrt{n}} \left\| \sum_{j=1}^m \tilde{R}_j \right\|_2 \xrightarrow{\mathbb{P}} 0\] to prove the above quantity converges to zero in probability.

To analyze this we look at its martingale projections $\mathcal{P}_t(\cdot) = \mathbb{E}[\cdot | \mathcal{F}_t] - \mathbb{E}[\cdot | \mathcal{F}_{t-1}]$ so we analyze
\begin{align*}
\sum_{j=1}^m \tilde{R}_j &= \sum_{j=1}^m \sum_{t=1}^j \mathcal{P}_t(R_j) = \sum_{t=1}^m \sum_{j=t}^m \mathcal{P}_t(R_j)
\end{align*}
We bound the projection $\|\mathcal{P}_t(R_j)\|_2$ by the functional dependence measure $\|R_j - R_j^{(t)}\|_2$ where $R_j^{(t)}$ is the same sequence but with i.i.d. copy of $X_t$. 
\begin{align*}
&|R_j - R_j^{(t)} | =  |A(\Delta_{j-1} - \Delta_{j-1}^{(t)}) -(\nabla G(\theta_{j-1}) - \nabla G(\theta^{(t)}_{j-1}) |  \\
&\le \sup_{s \in [0,1]} |A - \nabla^2 G(s\theta_{j-1} + (1-s)\theta_{j-1}^{(t)})| |\theta_{j-1} - \theta_{j-1}^{(t)}| \text{, by mean value theorem} \\
&\le C_L \sup_{s \in [0,1]} |s(\theta_{j-1}-\theta^*) + (1-s)(\theta_{j-1}^{(t)}-\theta^*)| \|\Delta_{j-1} - \Delta_{j-1}^{(t)} | \\
&\le C_L \left( |\Delta_{j-1} | + |\Delta_{j-1}^{(t)}| \right) |\Delta_{j-1} - \Delta_{j-1}^{(t)}| 
\end{align*}
By Cauchy-Schwarz $\mathbb{E}[X^2 Y^2] \le \sqrt{\mathbb{E}[X^4]\mathbb{E}[Y^4]} \implies \|XY\|_2 \le \|X\|_4 \|Y\|_4$ and hence
\begin{align*}
\|\mathcal{P}_t(R_j)\|_2 &\le \|R_j - R_j^{(t)}\|_2 \\
&\le C_L \left\| \left( |\Delta_{j-1}| + |\Delta_{j-1}^{(t)}| \right) |\Delta_{j-1} - \Delta_{j-1}^{(t)}| \right\|_2 \\
&\le C_L \left\| |\Delta_{j-1}| + |\Delta_{j-1}^{(t)}| \right\|_4 \left\| |\Delta_{j-1} - \Delta_{j-1}^{(t)}| \right\|_4 \\
&\le 2C_L \|\Delta_{j-1}\|_4 \|\Delta_{j-1} - \Delta_{j-1}^{(t)}\|_4.
\end{align*}

By Theorem 2.3 of \cite{NEURIPS2025_7aae9e3e}, we have $\|\Delta_{j-1}\|_4 \le \rho_4 (1-\alpha_0\gamma)^{(j-1)/2} + 3 \sqrt{\frac{M_4^2}{\alpha_0}} \sqrt{\gamma}$ and by Theorem 3.1 of \cite{NEURIPS2025_7aae9e3e} for $p=4$ 
$$\|\Delta_{j-1} - \Delta_{j-1}^{(t)}\|_4 \le 2\sqrt{2}\gamma(1-\alpha_0\gamma)^{(j-1-t)/2}\sqrt{M_4^2 + L_4^2[(1-\alpha_0\gamma)^{t-1}\rho_4^2 + 9M_4^2\alpha_0^{-1}\gamma]} \le C_4 \gamma (1-\alpha_0\gamma)^{(j-1-t)/2}$$
where $C_4 = 2\sqrt{2}\sqrt{M_4^2 + L_4^2(\rho_4^2 + 9M_4^2\alpha_0^{-1}\gamma)}$ since $(1-\alpha_0\gamma)^{t-1} \le 1$. Hence,
\begin{align*}
\|\mathcal{P}_t(R_j)\|_2 &\le 2C_L \|\Delta_{j-1}\|_4 \|\Delta_{j-1} - \Delta_{j-1}^{(t)}\|_4 \\
&\le 2C_L \left( \rho_4 (1-\alpha_0\gamma)^{(j-1)/2} + 3 M_4 \alpha_0^{-1/2} \sqrt{\gamma} \right) C_4 \gamma (1-\alpha_0\gamma)^{(j-1-t)/2} \\
&= 2C_L C_4 \gamma \left( \rho_4 (1-\alpha_0\gamma)^{j - 1 - t/2} + 3 M_4 \alpha_0^{-1/2} \sqrt{\gamma} (1-\alpha_0\gamma)^{(j-1-t)/2} \right).
\end{align*}

For each fixed $\ell\ge1$,  $\mathcal{P}_t(R_{t+\ell})$ is a martingale difference sequence in $t$. Thus, by Minkowski's inequality and Doob's maximal $\mathcal{L}^2$ inequality applied separately at each fixed $\ell$,
\begin{align*}
&\mathbb E\max_{m\le n}\frac{1}{n}
\left|\sum_{j=1}^m\tilde R_j\right|^2 \\
&=  \mathbb E\max_{m\le n}\frac{1}{n}
\left|\sum_{\ell=1}^{n-1}\sum_{t=1}^{(m-\ell)_+}
\mathcal P_t(R_{t+\ell})  \right|^2  \text{, since $\mathcal{P}_t(R_t) = 0$} \\
&\le \frac{4}{n}
\left(
\sum_{\ell=1}^{n-1}
\left(\sum_{t=1}^{n-\ell}
\|\mathcal P_t(R_{t+\ell})\|_2^2\right)^{1/2}
\right)^2 \\
&\le \frac{16C_L^2C_4^2\gamma^2}{n}
\left(
\sum_{\ell=1}^{n-1}
\left\{
\frac{\rho_4(1-\alpha_0\gamma)^{\ell-1}}
{\sqrt{\alpha_0\gamma}}
+3M_4\alpha_0^{-1/2}\sqrt{n\gamma}
(1-\alpha_0\gamma)^{(\ell-1)/2}
\right\}
\right)^2 \\
&\le \frac{16C_L^2C_4^2\gamma^2}{n}
\left(
\frac{\rho_4}{\sqrt{\alpha_0\gamma}}
\frac{1}{\alpha_0\gamma}
+3M_4\alpha_0^{-1/2}\sqrt{n\gamma}
\frac{2}{\alpha_0\gamma}
\right)^2 \text{, since $1-(1-\alpha_0\gamma)^{1/2}
\ge \frac{1}{2}\alpha_0\gamma$} \\
&= \frac{16C_L^2C_4^2}{n}
\left(
\frac{\rho_4}{\alpha_0^{3/2}\sqrt{\gamma}}
+\frac{6M_4}{\alpha_0^{3/2}}\sqrt{n\gamma}
\right)^2
 \\
&\le \frac{32C_L^2C_4^2\rho_4^2}{\alpha_0^3n\gamma}
+\frac{1152C_L^2C_4^2M_4^2}{\alpha_0^3}\gamma
\longrightarrow0,
\end{align*}
since $n\gamma\to\infty$ and $\gamma\to0$.

\end{proof}

\begin{lemma}\label{sum FCLT lin iterates}
Let $$L_k = \gamma \sum_{j=1}^k (I_d - \gamma A)^{k-j} Z_j.$$ Assume the above assumptions of strong convexity, stochastic Lipschitzness and $\gamma < \min \left\{ \frac{1}{\lambda^{*}}, \, \frac{\mu}{(6p - 5)L_p^2} \right\}$,  then
$\frac{1}{\sqrt{n}} \sum_{k=1}^{\lfloor nt \rfloor} L_k \Rightarrow A^{-1} S^{1/2} B(t)$ in $D([0,1], \mathbb{R}^d)$
\end{lemma}

\begin{proof}

Let $M_n(t) = \frac{1}{\sqrt{n}} M_{\lfloor nt \rfloor}$ then
\begin{align*}
& \sum_{k=1}^m L_k = \gamma \sum_{k=1}^m \sum_{j=1}^k (I_d - \gamma A)^{k-j} Z_j \\
&= \gamma \sum_{j=1}^m \left(\sum_{k=j}^m (I_d - \gamma A)^{k-j}\right)Z_j \\
& = \sum_{j=1}^m A^{-1} (I_d - (I_d - \gamma A)^{m-j+1})Z_j  \\
& = A^{-1} \sum_{j=1}^m Z_j - A^{-1} \sum_{j=1}^m (I_d - \gamma A)^{m-j+1} Z_j 
\end{align*}
where we can invert $A := \nabla^2 G(\theta^*)$ since we assume strong convexity. Now by the multivariate Donsker's theorem(see Theorem 4.3.5 in \cite{whitt2002stochastic})
$$A^{-1} \frac{1}{\sqrt{n}} \sum_{j=1}^{\lfloor nt \rfloor} Z_j \Rightarrow A^{-1} S^{1/2} B(t)$$ in $D([0,1], \mathbb{R}^d)$.

So, it only remains to show that the second term: $ \frac{1}{\sqrt{n}} A^{-1} \sum_{j=1}^m (I_d - \gamma A)^{m-j+1} Z_j $, converges uniformly in probability to 0, and then we may apply Slutsky's theorem for metric spaces. 

To handle the second term recognize the underlying martingale sum structure so let $M_m = \sum_{j=1}^m Z_j$ with $M_0 = 0$, so that $Z_j = M_j - M_{j-1}$. Then
\begin{align*}
&\sum_{j=1}^m (I_d - \gamma A)^{m-j+1} Z_j\\
&= \sum_{j=1}^m (I_d - \gamma A)^{m-j+1} M_j - \sum_{j=1}^m (I_d - \gamma A)^{m-j+1} M_{j-1} \\
&= \sum_{j=1}^m (I_d - \gamma A)^{m-j+1} M_j - \sum_{j=1}^{m-1} (I_d - \gamma A)^{m-j} M_j \\
&= (I_d - \gamma A) M_m - \gamma A \sum_{k=0}^{m-2} (I_d - \gamma A)^{k+1} M_{m-1-k}
\end{align*}

Now since $\gamma A \sum_{j=1}^{m-1} (I_d - \gamma A)^{m-1-j} = I_d - (I_d - \gamma A)^{m-1}$ we have that  $(I_d - \gamma A) M_m = \gamma A \sum_{j=1}^{m-1} (I_d - \gamma A)^{m-j} M_m +(I_d - \gamma A)^{m} M_m$, so by plugging this identity into the above equation and using the change of notation $M_n(t) = \frac{1}{\sqrt{n}} M_{\lfloor nt \rfloor}$, we have that
\begin{align*}
& A^{-1} \frac{1}{\sqrt{n}} \sum_{j=1}^m (I_d - \gamma A)^{m-j+1} Z_j \\
&= A^{-1} \Big[ (I_d - \gamma A)^m M_n(m/n) + \gamma A \sum_{k=0}^{m-2} (I_d - \gamma A)^{k+1} \Big( M_n(m/n) - M_n((m-1-k)/n) \Big) \Big]    
\end{align*}

We now show that the maximum of both of these two terms over $1 \le m \le n$ converges to zero in probability. Because $\|A^{-1}\| = \frac{1}{\lambda_*}$ is constant, we focus on the inner terms.

\bigskip

We start with the first term,
For now fix some $\delta \in (0,1)$ which we will later send to zero. 
\begin{align*}
&\max_{1 \le m \le n} | (I_d - \gamma A)^m M_n(m/n) | \leq \max_{m \le \delta n} |(I_d - \gamma A)^m M_n(m/n)| +  \max_{\delta n \le m \le n} |(I_d - \gamma A)^m M_n(m/n)|  \\
& \leq \max_{m \le \delta n} |(I_d - \gamma A)|^m |M_n(m/n)| + \max_{\delta n \le m \le n} |(I_d - \gamma A)|^m |M_n(m/n)| \\
&\le \max_{m \le \delta n} |M_n(m/n)| + (1-\gamma\lambda_*)^{\lfloor \delta n \rfloor} \max_{1 \le m \le n} |M_n(m/n)|
\end{align*}
since $|I_d - \gamma A| \le 1 - \gamma\lambda_* < 1$ by assumption. But $M_m$ is a discrete time martingale so by Doob's maximal $L^p$ inequality 
\begin{align*} 
&\lim_{\delta \rightarrow 0} \lim_{n \rightarrow \infty}\mathbb{P}(\max_{m \le \delta n} |M_n(m/n)|  > \epsilon) = \lim_{\delta \rightarrow 0} \lim_{n \rightarrow \infty} \mathbb{P}(\max_{m \le \delta n} \frac{\|M_m\|}{\sqrt{n}}  > \epsilon) \\
&\le \lim_{\delta \rightarrow 0} \lim_{n \rightarrow \infty} \frac{4 \delta n \text{tr}(S)}{\epsilon^2 n} = \lim_{\delta \rightarrow 0} \frac{4 \delta \text{tr}(S)}{\epsilon^2} = 0.
\end{align*} 
and again using Doob's maximal $L^p$ inequality 
\begin{align*}
& \mathbb{P}((1-\gamma\lambda_*)^{\lfloor \delta n \rfloor} \max_{1 \le m \le n} |M_n(m/n)|  > \epsilon) \\
&\leq (1-\gamma\lambda_*)^{2 \lfloor \delta n \rfloor} \frac{4 \text{tr}(S)}{\epsilon^2} \\
&\leq \exp (- 2 \gamma\lambda_*  (\delta n - 1)) \frac{4 \text{tr}(S)}{\epsilon^2} \\
&\leq \exp (- 2 \lambda_*( \delta \log(n) - \gamma)) \frac{4 \text{tr}(S)}{\epsilon^2} \text{ since $\gamma n \geq \log(n)$}\\
&\rightarrow_{n \rightarrow \infty} 0
\end{align*}
Hence, both terms converge to zero in probability and so by a union bound so does their sum.

\bigskip

We now shift focus to the maximal sum of differences term. We again follow the same logic and split the sum into terms close to m ($k \le \delta n$) and terms distant from m ($k > \delta n$). So define
$$
w_n(\delta)=\sup_{|s-t|\le\delta}\|M_n(s)-M_n(t)\|
$$
to be the modulus of continuity. Then, splitting the maximum as before,
\begin{align*}
&\max_{1\le m\le n}\left|\gamma A\sum_{k=0}^{m-2}(I_d-\gamma A)^{k+1}\bigl(M_n(m/n)-M_n((m-1-k)/n)\bigr)\right|\\
&\le \max_{1\le m\le n}\left|\gamma A\sum_{k=0}^{\min\{\lfloor\delta n\rfloor,m-2\}}(I_d-\gamma A)^{k+1}\bigl(M_n(m/n)-M_n((m-1-k)/n)\bigr)\right|\\
&\quad+\max_{1\le m\le n}\left|\gamma A\sum_{k=\lfloor\delta n\rfloor+1}^{m-2}(I_d-\gamma A)^{k+1}\bigl(M_n(m/n)-M_n((m-1-k)/n)\bigr)\right| \\
&\le \left(\gamma\lambda^*\sum_{k=0}^{\min\{\lfloor\delta n\rfloor,m-2\}}(1-\gamma\lambda_*)^k\right)w_n(\delta+1/n)
+\left(\gamma\lambda^*\sum_{k=\lfloor\delta n\rfloor+1}^{n-2}(1-\gamma\lambda_*)^k\right)\left(2\sup_{t\in[0,1]}|M_n(t)|\right)\\
&\le \frac{\lambda^*}{\lambda_*}w_n(\delta+1/n)
+2\frac{\lambda^*}{\lambda_*}(1-\gamma\lambda_*)^{\lfloor\delta n\rfloor}\sup_{t\in[0,1]}|M_n(t)|.
\end{align*}

Again, $(1-\gamma\lambda_*)^{\lfloor\delta n\rfloor}\sup_{t\in[0,1]}|M_n(t)| = (1-\gamma\lambda_*)^{\lfloor \delta n \rfloor} \max_{1 \le m \le n} |M_n(m/n)|$, converges to 0 in probability as proved above, and thus so does $2\frac{\lambda^*}{\lambda_*}(1-\gamma\lambda_*)^{\lfloor\delta n\rfloor}\sup_{t\in[0,1]}|M_n(t)|$. Moreover, by Donsker's theorem $M_n(t)$ is tight and so by theorem 8.2 of \cite{billingsley1968convergence}  $\lim_{\delta \to 0}\,\limsup_{n \to \infty} \mathbb{P}\bigl(w_n(\delta+1/n) > \varepsilon\bigr) = 0$. Hence, by another two union bounds we have that for the second term $$ \max_{1 \leq m \leq n} \frac{1}{\sqrt{n}} A^{-1} \sum_{j=1}^m (I_d - \gamma A)^{m-j+1} Z_j $$ converges to zero in probability. Hence, we have convergence in probability for the supremum norm (see section 12 of \cite{billingsley1999convergence}) which implies convergence in probability in the Skorokhod space, which then by Slutsky's theorem for metric spaces (see theorem 3.1 of \cite{billingsley1999convergence}) implies the lemma.
\end{proof}

\begin{theorem}\label{kmt for the linear component}
Assume $\mathbb{E}\|Z_1\|_2^p<\infty$ for some fixed $p\ge 2$, and construct, on an extension of the probability space, independent Gaussian vectors $Z_1^\diamond,\ldots,Z_n^\diamond$ with $\operatorname{cov}Z_j^\diamond=\operatorname{cov}Z_j$ as specified below. For $p>2$, use $K_{2,d,S}$ from Lemma~\ref{lem:innovation-new} and for $p=2$, let $K_{2,d,S}=8$, then for
\begin{align*}
L_k := \gamma\sum_{j=1}^k (I_d-\gamma A)^{k-j}Z_j,    L_k^\diamond := \gamma\sum_{j=1}^k (I_d-\gamma A)^{k-j}Z_j^\diamond,
\end{align*}
Then for every $y>0$,
\begin{align*}
P\left[\max_{i\le n}\Big\|\sum_{k=1}^i(L_k-L_k^\diamond)\Big\|_2 \ge y\right] \;\le\; K_{p,d,S} \, n\,\mathbb{E}\|Z_1\|_2^p\left(\frac{2+\lambda^*/\lambda_*}{\lambda_* y}\right)^{p},
\end{align*}
For $p=2$, take Gaussian independent of the original sequence and set $K_{2,d,S}=8$, then Doob's inequality gives $\mathbb E\max_{k\le n}\|\sum_{j\le k}(Z_j-Z_j^\diamond)\|_2^2\le8nM_2^2$, so the displayed bound  also holds in this case. For $p>2$ use the given construction. 
\end{theorem}
\begin{proof}

Let $D_i := M_i - M_i^\diamond = \sum_{j=1}^i(Z_j - Z_j^\diamond)$, with $D_0 := 0$, and $\Delta_n := \max_{1\le m\le n}|D_m|_2$. By the same fubini sum swap trick as before

\begin{align*} 
&\sum_{k=1}^i (L_k - L_k^\diamond) = A^{-1} D_i - A^{-1}\sum_{j=1}^i (I_d-\gamma A)^{i-j+1}(Z_j-Z_j^\diamond)
\end{align*}

Now,
\begin{align*} 
&\sum_{j=1}^i (I_d-\gamma A)^{i-j+1}(Z_j-Z_j^\diamond) = \sum_{j=1}^i (I_d-\gamma A)^{i-j+1}D_j - \sum_{j=1}^i (I_d-\gamma A)^{i-j+1}D_{j-1} \\ 
&= \sum_{j=1}^i (I_d-\gamma A)^{i-j+1}D_j - \sum_{j=0}^{i-1}(I_d-\gamma A)^{i-j}D_j  \\ 
&= (I_d-\gamma A)D_i - (I_d-\gamma A)^i D_0 + \sum_{j=1}^{i-1}\big[(I_d-\gamma A)^{i-j+1}-(I_d-\gamma A)^{i-j}\big]D_j\\ 
&= (I_d-\gamma A)D_i + \sum_{j=1}^{i-1}(I_d-\gamma A)^{i-j}\big[(I_d-\gamma A)-I_d\big]D_j \\ &= (I_d-\gamma A)D_i - \gamma A\sum_{j=1}^{i-1}(I_d-\gamma A)^{i-j}D_j
\end{align*} 
so
\begin{align*} 
&\Big\|\sum_{j=1}^i (I_d-\gamma A)^{i-j+1}(Z_j-Z_j^\diamond)\Big\|_2 \le \|(I_d-\gamma A)D_i\|_2 + \gamma\Big\|A\sum_{j=1}^{i-1}(I_d-\gamma A)^{i-j}D_j\Big\|_2 \\ &\le \|I_d-\gamma A\|_2\|D_i\|_2 + \gamma\|A\|_2\sum_{j=1}^{i-1}\|I_d-\gamma A\|_2^{i-j}\|D_j\|_2 \\ &\le (1-\gamma\lambda_*)\|D_i\|_2 + \gamma\lambda^*\sum_{j=1}^{i-1}(1-\gamma\lambda_*)^{i-j}\|D_j\|_2 
\end{align*}

\begin{align*} 
& \le \Delta_n(1-\gamma\lambda_*) + \gamma\lambda^*\Delta_n\sum_{j=1}^{i-1}(1-\gamma\lambda_*)^{i-j} \\ &\le \Delta_n(1-\gamma\lambda_*) + \Delta_n \frac{\lambda^*}{\lambda_*} \\ 
&\leq \Delta_n(1+\frac{\lambda^*}{\lambda_*})   
\end{align*}

Therefore,
\begin{align*} &\Big\|\sum_{k=1}^i(L_k-L_k^\diamond)\Big\|_2 \le \|A^{-1}\|_2\|D_i\|_2 + \|A^{-1}\|_2\Big\|\sum_{j=1}^i(I_d-\gamma A)^{i-j+1}(Z_j-Z_j^\diamond)\Big\|_2 \\ 
&\le \frac{1}{\lambda_*}\|D_i\|_2 + \frac{1}{\lambda_*}\Delta_n(1+\frac{\lambda^*}{\lambda_*})  \\
&\le \frac{1}{\lambda_*}\Delta_n + \frac{1}{\lambda_*}\Delta_n(1+\frac{\lambda^*}{\lambda_*})  \\
 &= \frac{2+\frac{\lambda^*}{\lambda_*}}{\lambda_*}\Delta_n \\  &\implies \max_{1\le i\le n}\Big\|\sum_{k=1}^i(L_k-L_k^\diamond)\Big\|_2 \le \frac{2+\frac{\lambda^*}{\lambda_*}}{\lambda_*}\Delta_n
  \end{align*} 

Hence, $$ P\Big[\max_{1\le i\le n}\Big\|\sum_{k=1}^i(L_k-L_k^\diamond)\Big\|_2 \ge y\Big] \le P\Big[\Delta_n \ge \frac{\lambda_* y}{2+\frac{\lambda^*}{\lambda_*}}\Big]  $$

For $p>2$, Lemma~\ref{lem:innovation-new}, based on Theorem 4 of \citet{GotzeZaitsev2010}, gives $\mathbb E\Delta_n^p\le K_{p,d,S}n\mathbb E\|Z_1\|_2^p$ and the same bound holds for $p=2$ as above. Markov's inequality and the preceding deterministic bound complete the proof.

\end{proof}

\begin{theorem}[FCLT with Bias] \label{thm:fclt_init_bias}
Assume $G$ has a Lipschitz continuous third derivative with constant $L_4$. Let $n \to \infty$ and $\gamma \to 0$ simultaneously such that $\sqrt{n}\gamma \to c > 0$. Then, on every $D([\epsilon,1],\mathbb R^d)$ with $\epsilon>0$,
\begin{align*}
\frac{1}{\sqrt{n}} \sum_{k=1}^{\lfloor nt \rfloor} (\theta_k - \theta^\star) \Rightarrow -A^{-1} S^{1/2} B(t) + c t A^{-1} b + \frac{1}{c} A^{-1} \Delta_0 -  \frac{1}{c} A^{-1}\int_0^\infty R_{det}(x(s)) ds,
\end{align*}
where $B(t)$ is standard Brownian motion, $b = -\frac{1}{2} \nabla^3 G(\theta^\star)[\Gamma]$, $R_{det}(x):=\nabla G(\theta^\star + x) - \nabla^2 G(\theta^\star)x$ and $x(s)$ is the solution to the gradient-flow ODE $x(0) = \Delta_0$, $\dot{x}(s) = -\nabla G(\theta^\star + x(s))$
\end{theorem}

\begin{proof}
Using the exact linearization decomposition, we write $\Delta_k = (I_d - \gamma A)^k \Delta_0 - L_k + I_{1,k} + I_{2,k}$. By Lemma \ref{sum FCLT lin iterates}, the linear sequence $-\frac{1}{\sqrt{n}} \sum_{k=1}^{\lfloor nt \rfloor} L_k \Rightarrow -A^{-1}S^{1/2}B(t)$. By Lemma \ref{Sum stochastic error to 0}, $\frac{1}{\sqrt{n}} \sum_{k=1}^{\lfloor nt \rfloor} I_{2,k} \xrightarrow{\mathbb{P}} 0$.

For the initialization error,
\begin{align*}
\frac{1}{\sqrt{n}} \sum_{k=1}^{\lfloor nt \rfloor} (I_d - \gamma A)^k \Delta_0 &= \frac{1}{\sqrt{n}\gamma} A^{-1} (I_d - \gamma A) \Delta_0 - \frac{1}{\sqrt{n}\gamma} A^{-1} (I_d - \gamma A)^{\lfloor nt \rfloor+1} \Delta_0 \to \frac{1}{c} A^{-1} \Delta_0,
\end{align*}
since $\sqrt{n}\gamma \to c$ and $\| (I_d-\gamma A)^{\lfloor nt \rfloor} \|_2 \le e^{-\lambda_*\gamma(nt-1)} \to 0$.

For the rest of the proof we focus on $I_{1,k}$. Its centered part is uniformly negligible after division by $\sqrt n$ by Lemma \ref{Sum centered taylor error to 0}.
Now define
$$\begin{aligned}
& \mathcal{D}_m := \sum_{k=1}^m \mathbb{E}[I_{1,k}] - \sum_{k=1}^m \gamma \sum_{j=1}^k (I_d - \gamma A)^{k-j} \Big( \gamma b - R_{det}(x((j-1)\gamma)) \Big) \\
& = \gamma \sum_{k=1}^m \sum_{j=1}^k (I_d - \gamma A)^{k-j} \Big( \mathbb{E}[R_j] - \gamma b + R_{det}(x((j-1)\gamma)) \Big) 
\end{aligned}$$
and thus
$$\begin{aligned}
& \max_{1 \le m \le n} \|\mathcal{D}_m\|_2 \le \frac{2+\lambda^*/\lambda_*}{\lambda_*} \max_{1 \le m \le n} \left\| \sum_{j=1}^m \Big( \mathbb{E}[R_j] - \gamma b + R_{det}(x((j-1)\gamma)) \Big) \right\|_2 \text{, by Lemma \ref{lem:max_convolution_bound}}
\end{aligned}$$

Focusing on $\frac{1}{\sqrt{n}} \sum_{k=1}^{\lfloor nt \rfloor} \gamma \sum_{j=1}^k (I_d - \gamma A)^{k-j} \Big( \gamma b - R_{det}(x((j-1)\gamma)) \Big)$. By the Fubini sum swap,
$$\begin{aligned}
& \frac{1}{\sqrt{n}} \sum_{k=1}^{\lfloor nt \rfloor} \gamma \sum_{j=1}^k (I_d - \gamma A)^{k-j} \Big( \gamma b - R_{det}(x((j-1)\gamma)) \Big) \\
& = A^{-1} \frac{1}{\sqrt{n}} \sum_{j=1}^{\lfloor nt \rfloor} \Big( \gamma b - R_{det}(x((j-1)\gamma)) \Big) - A^{-1} \frac{1}{\sqrt{n}} \sum_{j=1}^{\lfloor nt \rfloor} (I_d - \gamma A)^{\lfloor nt \rfloor-j+1} \Big( \gamma b - R_{det}(x((j-1)\gamma)) \Big)
\end{aligned}$$
We first show the second term converges to zero uniformly on each $[\epsilon,1]$, with $\epsilon\in(0,1)$. Because $\gamma b$ is constant, its Fubini remainder sum is strictly bounded as
$$\begin{aligned}
& \left\| A^{-1} \frac{1}{\sqrt{n}} \sum_{j=1}^{\lfloor nt \rfloor} (I_d - \gamma A)^{\lfloor nt \rfloor - j + 1} \gamma b \right\|_2 \le \frac{\|A^{-1}\|_2}{\sqrt{n}} \sum_{j=1}^{\lfloor nt \rfloor} (1-\lambda_*\gamma)^{\lfloor nt \rfloor - j + 1} \gamma \|b\|_2 \\
& \le \frac{\|A^{-1}\|_2 \gamma \|b\|_2}{\sqrt{n}} \frac{1}{\lambda_*\gamma} = \frac{\|A^{-1}\|_2 \|b\|_2}{\lambda_* \sqrt{n}} \to 0
\end{aligned}$$

For $R_{det}$, we will show below towards the end of this proof that  $\|R_{det}(x((j-1)\gamma))\|_2 \le \frac{C_L}{2} \|\Delta_0\|_2^2 e^{-2\mu(j-1)\gamma}$. Hence,
$$\begin{aligned}
& \left\| A^{-1} \frac{1}{\sqrt{n}} \sum_{j=1}^{\lfloor nt \rfloor} (I_d - \gamma A)^{\lfloor nt \rfloor - j + 1} R_{det}(x((j-1)\gamma)) \right\|_2 \le \frac{\|A^{-1}\|_2 C_L \|\Delta_0\|_2^2}{2\sqrt{n}} \sum_{j=1}^{\lfloor nt \rfloor} (1-\lambda_*\gamma)^{\lfloor nt \rfloor - j + 1} e^{-2\mu(j-1)\gamma} \\
& \le \frac{\|A^{-1}\|_2 C_L \|\Delta_0\|_2^2}{2\sqrt{n}} \lfloor nt \rfloor e^{-\min(\lambda_*, 2\mu)(nt-1)\gamma} \quad \text{, bounding the convolution of geometric decays} \\
& \le \frac{\|A^{-1}\|_2 C_L \|\Delta_0\|_2^2}{2} t \sqrt{n} e^{-\min(\lambda_*,2\mu)(nt-1)\gamma} \to 0
\end{aligned}$$
For weak convergence in $D((0,1], \mathbb{R}^d)$, it suffices to show uniform convergence on compact subsets $[\epsilon, 1]$ (see \cite{billingsley1968convergence}). For any fixed $\epsilon > 0$ and $n$ large enough, the maximum on $[\epsilon, 1]$ occurs at $t = \epsilon$ and then it is clear that the above converges to zero since $-n \gamma$ scales as \(n\gamma/\log n\to\infty\). So moving on we show $\frac1{\sqrt n}\max_{1 \le m \le n} \left\| \sum_{j=1}^m \Big( \mathbb{E}[R_j] - \gamma b + R_{det}(x((j-1)\gamma)) \Big) \right\|_2 $ converges to zero.

Following \cite{NEURIPS2025_7aae9e3e}, (by local smoothness Assumption 3.2 in the paper), when $|\theta_{n-1}-\theta^\star| \le \kappa$, we apply Taylor's theorem and we have $|R_n| \le L|\theta_{n-1}-\theta^\star|^2$. For $|\theta_{n-1}-\theta^\star| > \kappa$, the Lipschitz continuity of the gradient implies $|R_n| \le 2L_2|\theta_{n-1}-\theta^\star| \le 2L_2\kappa^{-1}|\theta_{n-1}-\theta^\star|^2$. So we finally have $|R_n| \le C_0|\Delta_{n-1}|^2$ where $C_0=\max\{L,\,2L_2\kappa^{-1}\}$.  So $\mathbb{E}\|R_j\|_2 \le C_0 \mathbb{E}\|\Delta_{j-1}\|_2^2 \le C_0 \left( \rho_2^2(1-\alpha_0\gamma)^{j-1} + \frac{3M_2^2}{\alpha_0}\gamma \right)$ by  theorem 2.3 of \cite{NEURIPS2025_7aae9e3e}. 
Let $H_{j-1} = \mathbb{E}[\Delta_{j-1}]$ and $S_{j-1} = \Delta_{j-1} - H_{j-1}$ use the integral Taylor formula. To retain the notation in the calculations below, $\nabla^3G(\tilde\theta_{j-1})$ denotes the averaged tensor
\[
2\int_0^1(1-s)\nabla^3G(\theta^*+H_{j-1}+sS_{j-1})\,ds,
\]
not evaluation at a common intermediate point. Then
$$
\nabla G(\theta_{j-1}) = \nabla G(\theta^* + H_{j-1}) + \nabla^2 G(\theta^* + H_{j-1}) S_{j-1} + \frac{1}{2} \nabla^3 G(\tilde{\theta}_{j-1})[S_{j-1}, S_{j-1}]
$$
We first start with a bound for 
\begin{align*}
& \|S_{j-1}\|_4^2 = \left\| \sum_{t=1}^{j-1} \mathcal{P}_t(\theta_{j-1}) \right\|_4^2 \text{,with } \mathcal{P}_t(\cdot) = \mathbb{E}[\cdot|\mathcal{F}_t] - \mathbb{E}[\cdot|\mathcal{F}_{t-1}], \\
& \le \sum_{t=1}^{j-1} 3 \|\mathcal{P}_t(\theta_{j-1})\|_4^2  \text{, by Rio's Inequality for $p=4$, \cite{Rio2009}} \le 3 \sum_{t=1}^{j-1} \|\theta_{j-1} - \theta_{j-1}^{(t)}\|_4^2 \\
&\text{, from bounding the projection with the functional dependence where $\theta_{j-1}^{(t)}$ has an i.i.d. copy of $X_t$,} \leq \\
&  3 \sum_{t=1}^{j-1} \left( 2\sqrt{2}\gamma(1-\alpha_0\gamma)^{(j-1-t)/2}\sqrt{M_4^2 + L_4^2[(1-\alpha_0\gamma)^{t-1}\rho_4^2 + 9M_4^2\alpha_0^{-1}\gamma]} \right)^2  \text{, by Theorem 3.1 of \cite{NEURIPS2025_7aae9e3e}} \\
& \le 3 \sum_{t=1}^{j-1} \left( 2\sqrt{2}\gamma(1-\alpha_0\gamma)^{(j-1-t)/2}\sqrt{M_4^2 + L_4^2(\rho_4^2 + 9M_4^2\alpha_0^{-1}\gamma)} \right)^2  \\
& = 3 \sum_{t=1}^{j-1} 8\gamma^2 (1-\alpha_0\gamma)^{j-1-t} \left( M_4^2 + L_4^2(\rho_4^2 + 9M_4^2\alpha_0^{-1}\gamma) \right)\\
& \le 24 \gamma^2 \left( M_4^2 + L_4^2(\rho_4^2 + 9M_4^2\alpha_0^{-1}\gamma) \right) \sum_{i=0}^{\infty} (1-\alpha_0\gamma)^i   = \frac{24}{\alpha_0} \left( M_4^2 + L_4^2(\rho_4^2 + 9M_4^2\alpha_0^{-1}\gamma) \right) \gamma \\
& \implies \mathbb{E}\|S_{j-1}\|_2^4 = \left( \|S_{j-1}\|_4^2 \right)^2  \le \left( \frac{24}{\alpha_0} \left( M_4^2 + L_4^2(\rho_4^2 + 9M_4^2\alpha_0^{-1}\gamma) \right) \gamma \right)^2  \\
& = \frac{576}{\alpha_0^2} \left( M_4^2 + L_4^2(\rho_4^2 + 9M_4^2\alpha_0^{-1}\gamma) \right)^2 \gamma^2  \le C_S \gamma^2  \text{, for  $C_S := \frac{576}{\alpha_0^2} \left( M_4^2 + L_4^2(\rho_4^2 + 9M_4^2\alpha_0^{-1}\bar\gamma) \right)^2$}.
\end{align*}
Now,
\begin{align*}
& R_j = A \Delta_{j-1} - \nabla G(\theta_{j-1})  \\
& = A (H_{j-1} + S_{j-1}) - \left( \nabla G(\theta^* + H_{j-1}) + \nabla^2 G(\theta^* + H_{j-1}) S_{j-1} + \frac{1}{2} \nabla^3 G(\tilde{\theta}_{j-1})[S_{j-1}, S_{j-1}] \right)  \\
& \implies \mathbb{E}[R_j] = A H_{j-1} - \nabla G(\theta^* + H_{j-1}) - \frac{1}{2} \mathbb{E}\Big[ \nabla^3 G(\tilde{\theta}_{j-1})[S_{j-1}, S_{j-1}] \Big] \\
& = -(\nabla G(\theta^* + H_{j-1}) - A H_{j-1}) - \frac{1}{2} \mathbb{E}\Big[ \nabla^3 G(\tilde{\theta}_{j-1})[S_{j-1}, S_{j-1}] \Big] 
\end{align*}

For the third derivative remainder we have that $$- \frac{1}{2} \mathbb{E}\Big[ \nabla^3 G(\tilde{\theta}_{j-1})[S_{j-1}, S_{j-1}] \Big] = - \frac{1}{2} \nabla^3 G(\theta^*) \big[ \mathbb{E}[S_{j-1} S_{j-1}^\top] \big] - \frac{1}{2} \mathbb{E}\Big[ \big( \nabla^3 G(\tilde{\theta}_{j-1}) - \nabla^3 G(\theta^*) \big) [S_{j-1}, S_{j-1}] \Big] $$ and
\begin{align*}
& \left\| \mathbb{E}\Big[ \big( \nabla^3 G(\tilde{\theta}_{j-1}) - \nabla^3 G(\theta^*) \big) [S_{j-1}, S_{j-1}] \Big] \right\|_2 \le 2L_4\int_0^1(1-s)\|H_{j-1}+sS_{j-1}\|_2\|S_{j-1}\|_4^2\,ds \\
& \le L_4 \|\Delta_{j-1}\|_2 \|S_{j-1}\|_4^2
\text{, since }\|H_{j-1}+sS_{j-1}\|_2\le\|\Delta_{j-1}\|_2.\\
& \le L_4 \sqrt{ \rho_2^2(1-\alpha_0\gamma)^{j-1} + 3M_2^2\alpha_0^{-1}\gamma } \sqrt{ C_S \gamma^2 } \text{, by Theorem 2.3 of \cite{NEURIPS2025_7aae9e3e}} \\
& \le L_4 \sqrt{C_S} \gamma \left( \rho_2(1-\alpha_0\gamma)^{(j-1)/2} + \sqrt{3} M_2 \alpha_0^{-1/2} \sqrt{\gamma} \right) \\
& \implies  \frac{1}{2\sqrt{n}} \sum_{j=1}^{\lfloor nt \rfloor} \left\| \mathbb{E}\Big[ \big( \nabla^3 G(\tilde{\theta}_{j-1}) - \nabla^3 G(\theta^*) \big) [S_{j-1}, S_{j-1}] \Big] \right\|_2 \\
&\le \frac{L_4 \rho_2 \sqrt{C_S}}{2\sqrt{n}} \gamma \sum_{j=1}^{\lfloor nt \rfloor} (1-\alpha_0\gamma)^{(j-1)/2} + \frac{\lfloor nt \rfloor}{2\sqrt{n}} L_4 \sqrt{3C_S} M_2 \alpha_0^{-1/2} \gamma^{1.5}  \\
& \le \frac{L_4 \rho_2 \sqrt{C_S}}{2\sqrt{n}} \frac{\gamma}{1-(1-\alpha_0\gamma)^{1/2}} + t \frac{L_4 \sqrt{3C_S} M_2}{2\alpha_0^{1/2}} \sqrt{n} \gamma^{1.5} \\
& \le \frac{L_4 \rho_2 \sqrt{C_S}}{2\sqrt{n}} \frac{2\gamma}{\alpha_0\gamma} + t \frac{L_4 \sqrt{3C_S} M_2}{2\alpha_0^{1/2}} (\sqrt{n}\gamma)\sqrt{\gamma} \text{, as } 1-(1-x)^{1/2} \ge x/2 \\
& = \frac{L_4 \rho_2 \sqrt{C_S}}{\alpha_0 \sqrt{n}} + t \frac{L_4 \sqrt{3C_S} M_2}{2\alpha_0^{1/2}} (\sqrt{n}\gamma)\sqrt{\gamma} \\
& \to 0 + c \cdot 0 = 0
\end{align*}

We now analyze 
$A^{-1} \frac{1}{\sqrt{n}} \sum_{j=1}^{\lfloor nt \rfloor} - \frac{1}{2} \nabla^3 G(\theta^*) \mathbb{E}[S_{j-1} S_{j-1}^\top]$.
We first start by showing that the nonlinear remainder converges to zero.
Let $ \phi_{j-1} = \tilde{I}_{1, j-1} + I_{2, j-1}$ with $\tilde{I}_{1, j-1} = I_{1, j-1} - \mathbb{E}[ I_{1, j-1}]$ then
\begin{align*}
& \mathbb{E}[S_{j-1} S_{j-1}^\top] - \mathbb{E}[L_{j-1} L_{j-1}^\top] = \mathbb{E}[(\phi_{j-1} - L_{j-1})(\phi_{j-1} - L_{j-1})^\top] - \mathbb{E}[L_{j-1} L_{j-1}^\top]  \\
& = \mathbb{E}[\phi_{j-1} \phi_{j-1}^\top] - \mathbb{E}[\phi_{j-1} L_{j-1}^\top] - \mathbb{E}[L_{j-1} \phi_{j-1}^\top]  \implies \|\mathcal{E}_{j-1}\|_2 \le \|\mathbb{E}[\phi_{j-1} \phi_{j-1}^\top]\|_2 + 2\|\mathbb{E}[L_{j-1} \phi_{j-1}^\top]\|_2 \\
&\le \mathbb{E}\|\phi_{j-1}\|_2^2 + 2 \mathbb{E}\big[\|L_{j-1}\|_2 \|\phi_{j-1}\|_2\big]  \le \mathbb{E}\|\phi_{j-1}\|_2^2 + 2 \sqrt{\mathbb{E}\|L_{j-1}\|_2^2 \mathbb{E}\|\phi_{j-1}\|_2^2}  \text{, by Cauchy-Schwarz}
\end{align*}

We first start with a bound on the norm of the linear terms:
\begin{align*}
& \mathbb{E}\|L_{j-1}\|_2^2 = \mathbb{E} \left\| \gamma \sum_{k=1}^{j-1} (I_d - \gamma A)^{j-1-k} Z_k \right\|_2^2 = \gamma^2 \sum_{k=1}^{j-1} \mathbb{E} \| (I_d - \gamma A)^{j-1-k} Z_k \|_2^2  \\
& \le \gamma^2 \sum_{k=1}^{j-1} (1-\lambda\gamma)^{2(j-1-k)} \mathbb{E} \| Z_k \|_2^2 \le \gamma^2 M_2^2 \sum_{i=0}^\infty (1-\lambda\gamma)^{2i}  = \gamma^2 M_2^2 \frac{1}{1 - (1-\lambda\gamma)^2} \le \gamma^2 M_2^2 \frac{1}{\lambda\gamma} = \frac{M_2^2}{\lambda} \gamma
\end{align*}

Now for the nonlinear remainder we have that $\mathbb{E}\|\phi_{j-1}\|_2^2 \le 2\mathbb{E}\|\tilde{I}_{1, j-1}\|_2^2 + 2\mathbb{E}\|I_{2, j-1}\|_2^2 $ so starting with the second term
\begin{align*}
& \mathbb{E}\|I_{2, j-1}\|_2^2 = \mathbb{E} \left\| \gamma \sum_{k=1}^{j-1} (I_d - \gamma A)^{j-1-k} D_k \right\|_2^2 = \gamma^2 \sum_{k=1}^{j-1} \mathbb{E} \| (I_d - \gamma A)^{j-1-k} D_k \|_2^2  \text{, since } D_k \text{ is a martingale} \\
& \le \gamma^2 \sum_{k=1}^{j-1} (1-\lambda\gamma)^{2(j-1-k)} 4 L_2^2 \mathbb{E}\|\Delta_{k-1}\|_2^2  \text{, by the stochastic lipschitzness assumption} \\
& \le 4 L_2^2 \gamma^2 \sum_{k=1}^{j-1} (1-\lambda\gamma)^{2(j-1-k)} \left( \rho_2^2 (1-\alpha_0\gamma)^{k-1} + 3 M_2^2 \alpha_0^{-1} \gamma \right)  \text{, by Theorem 2.3  of \cite{NEURIPS2025_7aae9e3e}} \\
& \le 4 L_2^2 \rho_2^2 \gamma^2 (j-1) (1-\lambda\gamma)^{j-2} + 12 L_2^2 M_2^2 \alpha_0^{-1} \gamma^3 \frac{1}{\lambda\gamma}  \\
& = C_{\phi 1} \gamma^2 (j-1) (1-\lambda\gamma)^{j-2} + C_{\phi 2} \gamma^2  \text{, where } C_{\phi 1} = 4 L_2^2 \rho_2^2 \text{ and } C_{\phi 2} = \frac{12 L_2^2 M_2^2}{\lambda \alpha_0}
\end{align*}

Now for the first term, use the integral mean-value formula, with $\nabla^2G(\tilde\theta_{k-1})$ denoting $\int_0^1\nabla^2G(\theta^*+H_{k-1}+sS_{k-1})\,ds$. Then
\begin{align*}
& \nabla G(\theta_{k-1}) = \nabla G(\theta^* + H_{k-1}) + \nabla^2 G(\tilde{\theta}_{k-1}) S_{k-1} \\
& \implies R_k = A \Delta_{k-1} - \nabla G(\theta_{k-1}) = \big(A H_{k-1} - \nabla G(\theta^* + H_{k-1})\big) + \big(A - \nabla^2 G(\tilde{\theta}_{k-1})\big) S_{k-1}  \\
& \implies \tilde{R}_k = R_k - \mathbb{E}[R_k] = \big(A - \nabla^2 G(\tilde{\theta}_{k-1})\big) S_{k-1} - \mathbb{E}[\big(A - \nabla^2 G(\tilde{\theta}_{k-1})\big) S_{k-1}] \\
& \implies \mathbb{E}\|\tilde{R}_k\|_2^2 = \mathbb{E}\|\big(A - \nabla^2 G(\tilde{\theta}_{k-1})\big) S_{k-1} - \mathbb{E}[\big(A - \nabla^2 G(\tilde{\theta}_{k-1})\big) S_{k-1}]\|_2^2 \le \mathbb{E}\|\big(A - \nabla^2 G(\tilde{\theta}_{k-1})\big) S_{k-1}\|_2^2  \\
\end{align*}

The central Hessian bound and the integral representation give
\[
\|A-\nabla^2G(\tilde\theta_{k-1})\|_2
\le C_L(2\|H_{k-1}\|_2+\|S_{k-1}\|_2),
 ,  C_L=\max\{L,2L_2/\kappa\}.
\]
Hence,
\begin{align*}
\mathbb E\|\tilde R_k\|_2^2
&\le 8C_L^2\|H_{k-1}\|_2^2\mathbb E\|S_{k-1}\|_2^2
+2C_L^2\mathbb E\|S_{k-1}\|_2^4\\
&\le C_{R1}\gamma(1-\lambda\gamma)^{k-1}+C_{R2}\gamma^2,
\end{align*}
where $C_{R1}=8C_L^2\rho_2^2C_S^{1/2}$ and
$C_{R2}=2C_L^2C_S+24C_L^2M_2^2\alpha_0^{-1}C_S^{1/2}$.
with this centered bound on the iterates we bound
\begin{align*}
& \mathbb{E}\|\tilde{I}_{1, j-1}\|_2^2 = \mathbb{E} \left\| \gamma \sum_{k=1}^{j-1} (I_d - \gamma A)^{j-1-k} \tilde{R}_k \right\|_2^2 \le \left( \gamma \sum_{k=1}^{j-1} (1-\lambda\gamma)^{j-1-k} \sqrt{\mathbb{E}\|\tilde{R}_k\|_2^2} \right)^2 \\
& \le \left( \gamma \sum_{k=1}^{j-1} (1-\lambda\gamma)^{j-1-k} \left( \sqrt{C_{R1}} \sqrt{\gamma} (1-\lambda\gamma)^{(k-1)/2} + \sqrt{C_{R2}} \gamma \right) \right)^2  \\
& \le \left( \sqrt{C_{R1}} \gamma^{1.5} (j-1) (1-\lambda\gamma/2)^{j-2} + \sqrt{C_{R2}} \gamma^2 \frac{1}{\lambda\gamma} \right)^2  \\
& \le 2 C_{R1} \gamma^3 (j-1)^2 (1-\lambda\gamma/2)^{2j-4} + \frac{2 C_{R2}}{\lambda^2} \gamma^2 \\ 
&= C_{R3} \gamma^3 (j-1)^2 (1-\lambda\gamma/2)^{2j-4} + C_{R4} \gamma^2  \text{, defining} C_{R3}=2C_{R1}=16 C_L^2 \rho_2^2 C_S^{1/2}, C_{R4}=2C_{R2}/\lambda^2
\end{align*}
so we can bound the main terms by
\begin{align*}
& \mathbb{E}\|\phi_{j-1}\|_2^2 \le 2 C_{\phi 1} \gamma^2 (j-1) (1-\lambda\gamma)^{j-2} + 2 C_{R3} \gamma^3 (j-1)^2 (1-\lambda\gamma/2)^{2j-4} + 2(C_{\phi 2} + C_{R4}) \gamma^2 \implies\\
& \frac{1}{\sqrt{n}} \sum_{j=1}^{\lfloor nt \rfloor} \mathbb{E}\|\phi_{j-1}\|_2^2 \\
&\le \frac{1}{\sqrt{n}} \sum_{j=2}^{\lfloor nt \rfloor} \left( 2 C_{\phi 1} \gamma^2 (j-1) (1-\lambda\gamma)^{j-2} + 2 C_{R3} \gamma^3 (j-1)^2 (1-\lambda\gamma/2)^{2j-4}\right) + \frac{1}{\sqrt{n}} \sum_{j=1}^{\lfloor nt \rfloor}2(C_{\phi 2} + C_{R4}) \gamma^2    \\
& \le \frac{2 C_{\phi 1} \gamma^2}{\sqrt{n}} \frac{1}{(\lambda\gamma)^2} + \frac{2 C_{R3} \gamma^3}{\sqrt{n}} \frac{2}{(\lambda\gamma/2)^3} + \frac{\lfloor nt \rfloor}{\sqrt{n}} 2(C_{\phi 2} + C_{R4}) \gamma^2   \\
& = \frac{2 C_{\phi 1}}{\lambda^2 \sqrt{n}} + \frac{32 C_{R3}}{\lambda^3 \sqrt{n}} + \frac{\lfloor nt \rfloor}{\sqrt{n}} 2(C_{\phi 2} + C_{R4}) \gamma^2   \\
& \to 0 + 0 + t c \cdot 0 = 0  \text{, since } \sqrt{n}\gamma \to c \text{, and } \gamma \to 0
\end{align*}

and thus it only remains to analyze the cross terms
\begin{align*}
& \mathbb{E}\big[\|L_{j-1}\|_2 \|\phi_{j-1}\|_2\big] \\
&\le \sqrt{ \left( \frac{M_2^2}{\lambda} \gamma \right) \left( 2 C_{\phi 1} \gamma^2 (j-1) (1-\lambda\gamma)^{j-2} + 2 C_{R3} \gamma^3 (j-1)^2 (1-\lambda\gamma/2)^{2j-4} + 2(C_{\phi 2} + C_{R4}) \gamma^2 \right) }  \\
& \le C_{X1} \gamma^{1.5} \sqrt{j-1} (1-\lambda\gamma)^{(j-2)/2} + C_{X2} \gamma^2 (j-1) (1-\lambda\gamma/2)^{j-2} + C_{X3} \gamma^{1.5}  \\
&\text{, for } C^2_{X1} = \frac{M_2^2 2 C_{\phi 1}}{\lambda},  C^2_{X2} = \frac{M_2^2 2 C_{R 3}}{\lambda},  C^2_{X3} = \frac{M_2^2 2 (C_{\phi 2}+C_{R 4})}{\lambda},\\
&\text{Using Cauchy--Schwarz for the first geometric sum,}\\
&\sum_{r=1}^\infty\sqrt r\,(1-\lambda\gamma)^{(r-1)/2}
\le\left(\sum_{r\ge1}r(1-\lambda\gamma)^{(r-1)/2}\right)^{1/2}
\left(\sum_{r\ge1}(1-\lambda\gamma)^{(r-1)/2}\right)^{1/2}
\le\left(\frac2{\lambda\gamma}\right)^{3/2},\\
&\frac{2}{\sqrt n}\sum_{j=1}^{\lfloor nt\rfloor}
\mathbb E[\|L_{j-1}\|_2\|\phi_{j-1}\|_2]
\le\frac{2^{5/2}C_{X1}}{\lambda^{3/2}\sqrt n}
+\frac{8C_{X2}}{\lambda^2\sqrt n}
+2t\sqrt n C_{X3}\gamma^{3/2}\longrightarrow0,\\
& \implies \frac{1}{\sqrt{n}} \sum_{j=1}^{\lfloor nt \rfloor} \|\mathcal{E}_{j-1}\|_2 \le \frac{1}{\sqrt{n}} \sum_{j=1}^{\lfloor nt \rfloor} \mathbb{E}\|\phi_{j-1}\|_2^2 + \frac{1}{\sqrt{n}} \sum_{j=1}^{\lfloor nt \rfloor} 2 \mathbb{E}\big[\|L_{j-1}\|_2 \|\phi_{j-1}\|_2\big] \to 0 
\end{align*}

With this it suffices to analyze the linear term, $\mathbb{E}[L_{j-1} L_{j-1}^\top] = \gamma \Gamma - \gamma (I_d - \gamma A)^{j-1} \Gamma (I_d - \gamma A)^{j-1} + \gamma^3 \sum_{i=0}^{j-2} (I_d - \gamma A)^i A \Gamma A (I_d - \gamma A)^i$  by lemma \ref{linear covariance expansion} so

\begin{align*}
& A^{-1} \frac{1}{\sqrt{n}} \sum_{j=1}^{\lfloor nt \rfloor} - \frac{1}{2} \nabla^3 G(\theta^*) \big[ \mathbb{E}[L_{j-1} L_{j-1}^\top] \big] = \\
&A^{-1} \frac{1}{\sqrt{n}} \sum_{j=1}^{\lfloor nt \rfloor} - \frac{1}{2} \nabla^3 G(\theta^*) \Big[ \gamma \Gamma - \gamma (I_d - \gamma A)^{j-1} \Gamma (I_d - \gamma A)^{j-1} + \gamma^3 \sum_{i=0}^{j-2} (I_d - \gamma A)^i A \Gamma A (I_d - \gamma A)^i \Big]  \\
& = \frac{\lfloor nt \rfloor}{n} (\sqrt{n}\gamma) A^{-1} b + \frac{1}{2\sqrt{n}} A^{-1} \nabla^3 G(\theta^*) \left[ \sum_{j=1}^{\lfloor nt \rfloor} \gamma (I_d - \gamma A)^{j-1} \Gamma (I_d - \gamma A)^{j-1} \right] \\
&+ A^{-1} \frac{1}{\sqrt{n}} \sum_{j=1}^{\lfloor nt \rfloor} - \frac{1}{2} \nabla^3 G(\theta^*) \Big[ \gamma^3 \sum_{i=0}^{j-2} (I_d - \gamma A)^i A \Gamma A (I_d - \gamma A)^i \Big]  \text{, defining } b = - \frac{1}{2} \nabla^3 G(\theta^*)[\Gamma] \\
& \to c t A^{-1} b 
\end{align*}
since for the first term
\begin{align*}
& \left\| \frac{1}{2\sqrt{n}} A^{-1} \nabla^3 G(\theta^*) \left[ \sum_{j=1}^{\lfloor nt \rfloor} \gamma (I_d - \gamma A)^{j-1} \Gamma (I_d - \gamma A)^{j-1} \right] \right\|_2 \\
&\le \frac{\gamma \|A^{-1}\|_2 \|\nabla^3 G(\theta^*)\|_2 \|\Gamma\|_2}{2\sqrt{n}} \sum_{j=1}^{\lfloor nt \rfloor} \|I_d - \gamma A\|_2^{2j-2}  \\
& \le \frac{\gamma \|A^{-1}\|_2 \|\nabla^3 G(\theta^*)\|_2 \|\Gamma\|_2}{2\sqrt{n}} \sum_{j=1}^\infty (1-\lambda_*\gamma)^{2j-2} = \frac{\gamma \|A^{-1}\|_2 \|\nabla^3 G(\theta^*)\|_2 \|\Gamma\|_2}{2\sqrt{n}} \frac{1}{1 - (1-\lambda_*\gamma)^2}  \\
& \le \frac{\gamma \|A^{-1}\|_2 \|\nabla^3 G(\theta^*)\|_2 \|\Gamma\|_2}{2\sqrt{n} (\lambda_*\gamma)}  = \frac{\|A^{-1}\|_2 \|\nabla^3 G(\theta^*)\|_2 \|\Gamma\|_2}{2\lambda_*\sqrt{n}} \to 0 
\end{align*}
and for the second term
\begin{align*}
& \left\| A^{-1} \frac{1}{\sqrt{n}} \sum_{j=1}^{\lfloor nt \rfloor} - \frac{1}{2} \nabla^3 G(\theta^*) \Big[ \gamma^3 \sum_{i=0}^{j-2} (I_d - \gamma A)^i A \Gamma A (I_d - \gamma A)^i \Big] \right\|_2  \\
& \le \frac{\gamma^3}{2\sqrt{n}} \|A^{-1}\|_2 \|\nabla^3 G(\theta^*)\|_2 \|A \Gamma A\|_2 \sum_{j=1}^{\lfloor nt \rfloor} \sum_{i=0}^{j-2} (1-\lambda_*\gamma)^{2i} \\
& \le \frac{\gamma^3}{2\sqrt{n}} \|A^{-1}\|_2 \|\nabla^3 G(\theta^*)\|_2 \|A \Gamma A\|_2 \sum_{j=1}^{\lfloor nt \rfloor} \frac{1}{\lambda_*\gamma} = \frac{\lfloor nt \rfloor \gamma^2}{2\lambda_*\sqrt{n}} \|A^{-1}\|_2 \|\nabla^3 G(\theta^*)\|_2 \|A \Gamma A\|_2  \\
& \le \frac{t (\sqrt{n}\gamma) \gamma}{2\lambda_*} \|A^{-1}\|_2 \|\nabla^3 G(\theta^*)\|_2 \|A \Gamma A\|_2  \to \frac{t \cdot c \cdot 0}{2\lambda_*} \|A^{-1}\|_2 \|\nabla^3 G(\theta^*)\|_2 \|A \Gamma A\|_2 = 0 
\end{align*}
Hence, $A^{-1} \frac{1}{\sqrt{n}} \sum_{j=1}^{\lfloor nt \rfloor} - \frac{1}{2} \nabla^3 G(\theta^*)\mathbb{E}[S_{j-1} S_{j-1}^\top] \rightarrow  c t A^{-1} b$ for $b = - \frac{1}{2} \nabla^3 G(\theta^*)[\Gamma]$.

So we finally analyze the remaining $-(\nabla G(\theta^* + H_{j-1}) - A H_{j-1})$ term in  $\mathbb{E}[R_j] = -(\nabla G(\theta^* + H_{j-1}) - A H_{j-1}) - \frac{1}{2} \mathbb{E}\Big[ \nabla^3 G(\tilde{\theta}_{j-1})[S_{j-1}, S_{j-1}] \Big] $.
Now since $H_j := \mathbb{E}[\Delta_j]$ we have that $H_j = H_{j-1} - \gamma \mathbb{E}[\nabla G(\theta_{j-1})]$ and define $R_{det}(x):= \nabla G(\theta^* + x) - A x$ so
\begin{align*}
& R_{det}(x) - R_{det}(y) = \nabla G(\theta^* + x) - \nabla G(\theta^* + y) - A(x-y)  \\
& = \int_0^1 \nabla^2 G(\theta^* + y + \tau(x-y))(x-y) d\tau - A(x-y)    \\
& = \int_0^1 \big( \nabla^2 G(\theta^* + y + \tau(x-y)) - \nabla^2 G(\theta^*) \big) (x-y) d\tau  \\
& \implies \|R_{det}(x) - R_{det}(y)\|_2 \le \int_0^1 \left\| \nabla^2 G(\theta^* + y + \tau(x-y)) - \nabla^2 G(\theta^*) \right\|_2 \|x-y\|_2 d\tau  \\
& \le \int_0^1 C_L \| y + \tau(x-y) \|_2 \|x-y\|_2 d\tau   \text{, by the same logic as above for lipschitzness of the Hessian} \\
& \leq \frac{C_L}{2} \big( \|x\|_2 + \|y\|_2 \big) \|x-y\|_2  
\end{align*}
Now define the solution to gradient flow ODE $\dot{x}(s) = - \nabla G(\theta^* + x(s))$ with $x(0) = \Delta_0$ so 
\begin{align*}
& \frac{d}{ds} \|x(s)\|_2^2 = -2 \langle x(s), \nabla G(\theta^* + x(s)) - \nabla G(\theta^*) \rangle \le -2\ \mu \|x(s)\|_2^2   \text{, by } \mu \text{-strong convexity of } G \\
& \implies \|x(s)\|_2 \le e^{-\ \mu s} \|\Delta_0\|_2 
\end{align*}
so 

\begin{align*}
& x(k \gamma) - x((k-1) \gamma) = \int_{(k-1) \gamma}^{k \gamma} \dot{x}(s) ds = - \int_{(k-1) \gamma}^{k \gamma} \nabla G(\theta^* + x(s)) ds  \\
& = -\gamma \nabla G(\theta^* + x((k-1) \gamma)) - \int_{(k-1) \gamma}^{k \gamma} \big( \nabla G(\theta^* + x(s)) - \nabla G(\theta^* + x((k-1) \gamma)) \big) ds \implies \\
& |x((k-1) \gamma) - \gamma \nabla G(\theta^* + x((k-1) \gamma)) - x(k \gamma)| = |\int_{(k-1) \gamma}^{k \gamma} \big( \nabla G(\theta^* + x(s)) - \nabla G(\theta^* + x((k-1) \gamma)) \big) ds| \\
& \le \int_{(k-1) \gamma}^{k \gamma} L_2 \|x(s) - x((k-1) \gamma)\|_2 ds  \le \int_{(k-1) \gamma}^{k \gamma} L_2 \left( \int_{(k-1) \gamma}^s \|\dot{x}(\tau)\|_2 d\tau \right) ds \\
& \le \int_{(k-1) \gamma}^{k \gamma} L_2 \left( \int_{(k-1) \gamma}^s L_2 \|x(\tau)\|_2 d\tau \right) ds   \text{, since } \|\dot{x}(\tau)\|_2 = \|\nabla G(\theta^* + x(\tau)) - \nabla G(\theta^*) \| \le L_2 \|x(\tau)\|_2 \\
& \le \int_{(k-1) \gamma}^{k \gamma} L_2^2 \|\Delta_0\|_2 e^{-\ \mu (k-1) \gamma} (s - (k-1) \gamma) ds   \text{, by bounding } \|x(\tau)\|_2 \le \|\Delta_0\|_2 e^{-\ \mu (k-1) \gamma}\\
& = \frac{1}{2} L_2^2 \|\Delta_0\|_2 e^{-\ \mu (k-1) \gamma} \gamma^2 
\end{align*}
We now bound the difference between $H_k$ and $x(k \gamma)$, and define $r_k:= (\mathbb{E}[\nabla G(\theta_{k-1})] - \nabla G(\theta^* + H_{k-1}))$ and $\delta_k:= x((k-1) \gamma) - \gamma \nabla G(\theta^* + x((k-1) \gamma)) - x(k \gamma)$ so
\begin{align*}
& | H_k - x(k \gamma) | = |H_{k-1} - \gamma \nabla G(\theta^* + H_{k-1}) - \gamma r_k - \big( x((k-1) \gamma) - \gamma \nabla G(\theta^* + x((k-1) \gamma)) - \delta_k \big)|  \\
& = |\big( I_d - \gamma \nabla G \big)(\theta^* + H_{k-1}) - \big( I_d - \gamma \nabla G \big)(\theta^* + x((k-1) \gamma)) - \gamma r_k + \delta_k|   \\
&\le \big| \big( I_d - \gamma \nabla G \big)(\theta^* + H_{k-1}) - \big( I_d - \gamma \nabla G \big)(\theta^* + x((k-1) \gamma)) \big| + \gamma |r_k| + |\delta_k|  \\
& \le (1 - \lambda \gamma) |H_{k-1} - x((k-1) \gamma)| + \gamma |r_k| + |\delta_k|   \text{, by the same proof as Theorem 2.3 of \cite{NEURIPS2025_7aae9e3e} } \\
& \le (1 - \lambda \gamma) \|H_{k-1} - x((k-1) \gamma)\|_2 + \gamma \big( C_{\delta 1} \gamma + C_{\delta 2} \gamma (1-\alpha_0\gamma)^{(k-1)/2} + C_{\delta 3} \gamma^{1.5} \big) + \frac{L_2^2 \|\Delta_0\|_2}{2}  \gamma^2 e^{-\ \mu (k-1)\gamma}   
\end{align*}

where the bound on $\gamma |r_k|$ comes from
\begin{align*}
& \mathbb{E}[\nabla G(\theta_{j-1})] = \nabla G(\theta^* + H_{j-1}) + \mathbb{E}\Big[ \nabla^2 G(\theta^* + H_{j-1}) S_{j-1} + \frac{1}{2} \nabla^3 G(\tilde{\theta}_{j-1})[S_{j-1}, S_{j-1}] \Big]  \\
& = \nabla G(\theta^* + H_{j-1}) + \frac{1}{2} \mathbb{E}\Big[ \nabla^3 G(\tilde{\theta}_{j-1})[S_{j-1}, S_{j-1}] \Big]  \\
& \implies \left\| \mathbb{E}[\nabla G(\theta_{j-1})] - \nabla G(\theta^* + H_{j-1}) \right\|_2 = \frac{1}{2} \left\| \mathbb{E}\Big[ \nabla^3 G(\tilde{\theta}_{j-1})[S_{j-1}, S_{j-1}] \Big] \right\|_2  \\
& \le \frac{1}{2} \|\nabla^3 G(\theta^*)\|_2 \mathbb{E}\|S_{j-1}\|_2^2 + \frac{1}{2} \left\| \mathbb{E}\Big[ \big( \nabla^3 G(\tilde{\theta}_{j-1}) - \nabla^3 G(\theta^*) \big) [S_{j-1}, S_{j-1}] \Big] \right\|_2 \\
& \le \frac{1}{2} \|\nabla^3 G(\theta^*)\|_2 C_S^{1/2} \gamma + \frac{L_4}{2} \sqrt{C_S} \gamma \left( \rho_2(1-\alpha_0\gamma)^{(j-1)/2} + \sqrt{3} M_2 \alpha_0^{-1/2} \sqrt{\gamma} \right) \\ &\text{, by our prior bound on  } \left\| \mathbb{E}\Big[ \big( \nabla^3 G(\tilde{\theta}_{j-1}) - \nabla^3 G(\theta^*) \big) [S_{j-1}, S_{j-1}] \Big] \right\|_2, \\
& = C_{\delta 1} \gamma + C_{\delta 2} \gamma (1-\alpha_0\gamma)^{(j-1)/2} + C_{\delta 3} \gamma^{1.5}  \\
&\text{, for } C_{\delta 1} = \frac{1}{2} \|\nabla^3 G(\theta^*)\|_2 C_S^{1/2} \text{, } C_{\delta 2} = \frac{L_4 \rho_2 \sqrt{C_S}}{2} \text{, and } C_{\delta 3} = \frac{L_4 \sqrt{3C_S} M_2}{2\sqrt{\alpha_0}}
\end{align*}

so for $\gamma\le1/\mu$, bounding $(1-\alpha_0\gamma)^{(k-1)/2}\le1$ yields
\begin{align*}
\|H_{j-1}-x((j-1)\gamma)\|_2
&\le\frac{C_{\delta1}+C_{\delta2}+C_{\delta3}\sqrt\gamma}{\lambda}\gamma
+\frac{L_2^2\|\Delta_0\|_2}{2}\frac{\gamma^2}{1-e^{-\mu\gamma}}\\
&\le C_E\gamma,
\end{align*}
where
\[
C_E=\frac{C_{\delta1}+C_{\delta2}+C_{\delta3}\sqrt{\bar\gamma}}{\lambda}
+\frac{L_2^2\|\Delta_0\|_2}{\mu}.
\]
 and
\begin{align*}
& \|H_{j-1}\|_2 = \|\mathbb{E}[\Delta_{j-1}]\|_2 \le \mathbb{E}\|\Delta_{j-1}\|_2 \le \left( \mathbb{E}\|\Delta_{j-1}\|_2^2 \right)^{1/2}   \\
& \le \left( \rho_2^2 (1-\alpha_0\gamma)^{j-1} + 3 M_2^2 \alpha_0^{-1} \gamma \right)^{1/2}   \text{, by Theorem 2.3 of \cite{NEURIPS2025_7aae9e3e}} \\
& \le \rho_2 (1-\alpha_0\gamma)^{(j-1)/2} + M_2 \sqrt{3} \alpha_0^{-1/2} \sqrt{\gamma}
\end{align*}
with this we bound for $\gamma \leq \frac{1}{\mu}$
\begin{align*}
& \left\| A^{-1} \frac{1}{\sqrt{n}} \sum_{j=1}^{\lfloor nt \rfloor} \big( R_{det}(H_{j-1}) - R_{det}(x((j-1)\gamma)) \big) \right\|_2 \le \frac{\|A^{-1}\|_2}{\sqrt{n}} \sum_{j=1}^{\lfloor nt \rfloor} \big\| R_{det}(H_{j-1}) - R_{det}(x((j-1)\gamma)) \big\|_2    \\
& \le \frac{\|A^{-1}\|_2}{\sqrt{n}} \sum_{j=1}^{\lfloor nt \rfloor} \frac{C_L}{2} \big( \|H_{j-1}\|_2 + \|x((j-1)\gamma)\|_2 \big) \|H_{j-1} - x((j-1)\gamma)\|_2   \text{, from the above bound on } R_{det} \\
& \le \frac{\|A^{-1}\|_2 C_L C_E \gamma}{2\sqrt{n}} \sum_{j=1}^{\lfloor nt \rfloor} \left( \rho_2 (1-\alpha_0\gamma)^{(j-1)/2} + M_2 \sqrt{3} \alpha_0^{-1/2} \sqrt{\gamma} + \|\Delta_0\|_2 e^{-\mu (j-1)\gamma} \right)   \text{, by our above bounds} \\
& = \frac{\|A^{-1}\|_2 C_L C_E \gamma \rho_2}{2\sqrt{n}} \sum_{j=1}^{\lfloor nt \rfloor} (1-\alpha_0\gamma)^{(j-1)/2} + \frac{\|A^{-1}\|_2 C_L C_E \gamma M_2 \sqrt{3}}{2 \sqrt{\alpha_0} \sqrt{n}} \sum_{j=1}^{\lfloor nt \rfloor} \sqrt{\gamma} \\
&\quad+ \frac{\|A^{-1}\|_2 C_L C_E \gamma \|\Delta_0\|_2}{2\sqrt{n}} \sum_{j=1}^{\lfloor nt \rfloor} e^{-\ \mu (j-1)\gamma} \\
& \le \frac{\|A^{-1}\|_2 C_L C_E \gamma \rho_2}{2\sqrt{n}} \frac{1}{1-(1-\alpha_0\gamma)^{1/2}} + \frac{\|A^{-1}\|_2 C_L C_E \gamma M_2 \sqrt{3}}{2 \sqrt{\alpha_0} \sqrt{n}} \lfloor nt \rfloor \sqrt{\gamma} \\
&\quad+ \frac{\|A^{-1}\|_2 C_L C_E \gamma \|\Delta_0\|_2}{2\sqrt{n}} \frac{1}{1 - e^{-\ \mu\gamma}} \\
& \le \frac{\|A^{-1}\|_2 C_L C_E \gamma \rho_2}{2\sqrt{n}} \frac{2}{\alpha_0\gamma} + \frac{\|A^{-1}\|_2 C_L C_E M_2 \sqrt{3}}{2 \sqrt{\alpha_0}} t (\sqrt{n}\gamma)\sqrt{\gamma} \\
&\quad+ \frac{\|A^{-1}\|_2 C_L C_E \gamma \|\Delta_0\|_2}{2\sqrt{n}} \frac{2}{\mu \gamma}   \\
&\text{,since } 1-(1-x)^{1/2} \ge x/2 \text{ and } 1-e^{-\mu \gamma} \ge \frac{\mu \gamma}{2} \text{ for } \mu \gamma \leq 1, \\
& = \frac{\|A^{-1}\|_2 C_L C_E \rho_2}{\alpha_0 \sqrt{n}} + \frac{\|A^{-1}\|_2 C_L C_E M_2 \sqrt{3}}{2 \sqrt{\alpha_0}} t (\sqrt{n}\gamma)\sqrt{\gamma} + \frac{\|A^{-1}\|_2 C_L C_E \|\Delta_0\|_2}{\mu \sqrt{n}} \\
& \to 0 + \frac{\|A^{-1}\|_2 C_L C_E M_2 \sqrt{3}}{2 \sqrt{\alpha_0}} t \cdot c \cdot 0 + 0 = 0
\end{align*}

With this it only remains to analyze $\frac{1}{\sqrt{n}} \sum_{j=1}^{\lfloor nt \rfloor} R_{det}(x((j-1)\gamma))$. We begin with a discretization error bound,

\begin{align*}
& \left\| x((j-1)\gamma) - x(s) \right\|_2 = \left\| \int_{(j-1)\gamma}^s \dot{x}(\tau) d\tau \right\|_2 \le \int_{(j-1)\gamma}^s \left\| \nabla G(\theta^* + x(\tau)) \right\|_2 d\tau  \\
& \le \int_{(j-1)\gamma}^s L_2 \|x(\tau)\|_2 d\tau \le \int_{(j-1)\gamma}^s L_2 \|\Delta_0\|_2 e^{-\mu \tau} d\tau  \text{, by stochastic Lipschitzness (by Assumption 2)} \\
& \le L_2 \|\Delta_0\|_2 e^{-\mu (j-1)\gamma} (s - (j-1)\gamma) 
\end{align*}
and with this
\begin{align*}
& \left\| \gamma R_{det}(x((j-1)\gamma)) - \int_{(j-1)\gamma}^{j \gamma} R_{det}(x(s)) ds \right\|_2 = \left\| \int_{(j-1)\gamma}^{j \gamma} \big( R_{det}(x((j-1)\gamma)) - R_{det}(x(s)) \big) ds \right\|_2 \\
& \le \int_{(j-1)\gamma}^{j \gamma} \frac{C_L}{2} \big( \|x((j-1)\gamma)\|_2 + \|x(s)\|_2 \big) \|x((j-1)\gamma) - x(s)\|_2 ds  \text{, by the above bound for } R_{det}\\
& \le \int_{(j-1)\gamma}^{j \gamma} \frac{C_L}{2} \big( 2\|\Delta_0\|_2 e^{-\mu (j-1)\gamma} \big) \big( L_2 \|\Delta_0\|_2 e^{-\mu (j-1)\gamma} (s - (j-1)\gamma) \big) ds  \\
& = C_L L_2 \|\Delta_0\|_2^2 e^{-2\mu (j-1)\gamma} \int_{(j-1)\gamma}^{j \gamma} (s - (j-1)\gamma) ds  \\
& = \frac{1}{2} C_L L_2 \|\Delta_0\|_2^2 e^{-2\mu (j-1)\gamma} \gamma^2 
\end{align*}
and hence
\begin{align*}
& \left\| \frac{1}{\sqrt{n}} \sum_{j=1}^{\lfloor nt \rfloor} R_{det}(x((j-1)\gamma)) - \frac{1}{c} \int_0^\infty R_{det}(x(s)) ds \right\|_2 \\
&= \left\| \frac{1}{\sqrt{n}\gamma} \gamma \sum_{j=1}^{\lfloor nt \rfloor} R_{det}(x((j-1)\gamma)) - \frac{1}{c} \int_0^\infty R_{det}(x(s)) ds \right\|_2 \\
& \le \left| \frac{1}{\sqrt{n}\gamma} - \frac{1}{c} \right| \left\| \gamma \sum_{j=1}^{\lfloor nt \rfloor} R_{det}(x((j-1)\gamma)) \right\|_2 + \frac{1}{c} \left\| \gamma \sum_{j=1}^{\lfloor nt \rfloor} R_{det}(x((j-1)\gamma)) - \int_0^{\lfloor nt \rfloor \gamma} R_{det}(x(s)) ds \right\|_2 \\
&+\frac{1}{c} \left\| \int_{\lfloor nt \rfloor \gamma}^\infty R_{det}(x(s)) ds \right\|_2  \leq  \\
& \left| \frac{1}{\sqrt{n}\gamma} - \frac{1}{c} \right| \left( \gamma \sum_{j=1}^{\lfloor nt \rfloor} \frac{C_L}{2} \|\Delta_0\|_2^2 e^{-2\mu (j-1)\gamma} \right) + \frac{1}{c} \sum_{j=1}^{\lfloor nt \rfloor} \frac{1}{2} C_L L_2 \|\Delta_0\|_2^2 \gamma^2 e^{-2\mu (j-1)\gamma} \\
&\quad+ \frac{1}{c} \int_{\lfloor nt \rfloor \gamma}^\infty \frac{C_L}{2} \|\Delta_0\|_2^2 e^{-2\mu s} ds  \text{, from above}  \\
& \le \left| \frac{1}{\sqrt{n}\gamma} - \frac{1}{c} \right| \frac{C_L \|\Delta_0\|_2^2}{2} \gamma \frac{1}{1 - e^{-2\mu\gamma}} + \frac{C_L L_2 \|\Delta_0\|_2^2 \gamma^2}{2 c} \frac{1}{1 - e^{-2\mu\gamma}} + \frac{C_L \|\Delta_0\|_2^2}{4 c \mu} e^{-2\mu \lfloor nt \rfloor \gamma}   \\
& \le \left| \frac{1}{\sqrt{n}\gamma} - \frac{1}{c} \right| \frac{C_L \|\Delta_0\|_2^2}{2\mu} + \frac{C_L L_2 \|\Delta_0\|_2^2}{2 c \mu} \gamma + \frac{C_L \|\Delta_0\|_2^2}{4 c \mu} e^{-2\mu (nt-1) \gamma} \\
&\text{, by  } 1 - e^{-2\mu\gamma} \ge \mu \gamma \text{ for } \gamma \leq \frac{1}{2 \mu } \\
& = \left| \frac{1}{\sqrt{n}\gamma} - \frac{1}{c} \right| \frac{C_L \|\Delta_0\|_2^2}{2\mu} + \frac{C_L L_2 \|\Delta_0\|_2^2}{2 c \mu} \gamma + \frac{C_L \|\Delta_0\|_2^2}{4 c \mu} e^{-2\mu t (\sqrt{n}\gamma) \sqrt{n} + 2\mu\gamma}  \\
& \to 0
\end{align*}

$$\implies \lim_{n \to \infty} A^{-1} \frac{1}{\sqrt{n}} \sum_{j=1}^{\lfloor nt \rfloor} -(\nabla G(\theta^* + H_{j-1}) - A H_{j-1}) = - \frac{1}{c} A^{-1} \int_0^\infty R_{det}(x(s)) \, ds$$ and so combining all of this proves the theorem.
\end{proof}

\begin{theorem} \label{fclt-burn in}

Let $n_0(\gamma)$ be a Burn in index satisfying $n_0(\gamma) \ge \frac{2}{\gamma \lambda} |\log(\sqrt{n}\gamma)|$ and $n_0(\gamma) = \mathcal{O}( \frac{2}{\gamma \lambda} |\log(\sqrt{n}\gamma)|)$. Define the post Burn in partial sums $S_m = \sum_{k=n_0(\gamma)+1}^{n_0(\gamma)+m} (\theta_k - \theta^*)$. If $\sqrt{n}\gamma \to 0$ and $n\gamma \to \infty$, then:

\begin{align*}
\frac{1}{\sqrt{n}} S_{\lfloor nt \rfloor} \Rightarrow -A^{-1} S^{1/2} B(t) \quad \text{in } D([0,1], \mathbb{R}^d).
\end{align*}
\end{theorem}
\begin{proof}
Again by the same linearization decomposition from \cite{NEURIPS2025_7aae9e3e}

\begin{align*}
\frac{1}{\sqrt{n}} \sum_{k=1}^m \Delta_{n_0(\gamma)+k} &= \frac{1}{\sqrt{n}} \sum_{k=1}^m (I_d - \gamma A)^{n_0(\gamma)+k} \Delta_0 - \frac{1}{\sqrt{n}} \sum_{k=1}^m L_{n_0(\gamma)+k} + \frac{1}{\sqrt{n}} \sum_{k=1}^m (I_{1,n_0(\gamma)+k} + I_{2,n_0(\gamma)+k}).
\end{align*}

Because of the Burn in the dependence on the initialization term goes to 0 since

\begin{align*}
\left\| \frac{1}{\sqrt{n}} \sum_{k=1}^m (I_d - \gamma A)^{n_0(\gamma)+k} \Delta_0 \right\|_2 &= \left\| \frac{1}{\sqrt{n}} (I_d - \gamma A)^{n_0(\gamma)+1} \big[ I_d - (I_d - \gamma A)^m \big] (\gamma A)^{-1} \Delta_0 \right\|_2 \\
&\le \frac{2}{\sqrt{n}\gamma \lambda_*} (1-\gamma \lambda_*)^{n_0(\gamma)} \|\Delta_0\|_2 \\
&\le \frac{2}{\sqrt{n}\gamma \lambda_*} \exp(-\gamma \lambda_* n_0(\gamma)) \|\Delta_0\|_2 \\
&\le \frac{2}{\sqrt{n}\gamma \lambda_*} (n\gamma^2) \|\Delta_0\|_2 \\
&= \frac{2 \sqrt{n} \gamma}{\lambda_*} \|\Delta_0\|_2 \rightarrow 0
\end{align*}

The second to last line follows since we assume $n_0(\gamma) \ge \frac{2}{\gamma \lambda} |\log(\sqrt{n}\gamma)|$ so
\begin{align*}
\exp(-\gamma \lambda_* n_0(\gamma)) &\le \exp\left(-\gamma \lambda_* \frac{2}{\gamma \lambda} |\log(\sqrt{n}\gamma)| \right) \\
&\leq \exp\left(-2 |\log(\sqrt{n}\gamma)| \right) \\
&= (\sqrt{n}\gamma)^2 \\
&= n\gamma^2.
\end{align*}
The last line follows since we assume $\sqrt{n}\gamma \to 0$.

Now for the Taylor expansion remainder term:

\begin{align*}
\sum_{k=1}^m I_{1,n_0(\gamma)+k}
&= \gamma \sum_{k=1}^m \sum_{j=1}^{n_0(\gamma)+k} (I_d-\gamma A)^{n_0(\gamma)+k-j} R_j \\
&= \gamma \sum_{k=1}^m \sum_{j=1}^{n_0(\gamma)} (I_d-\gamma A)^{n_0(\gamma)+k-j} R_j + \gamma \sum_{k=1}^m \sum_{j=n_0(\gamma)+1}^{n_0(\gamma)+k} (I_d-\gamma A)^{n_0(\gamma)+k-j} R_j \\
&= \gamma \sum_{k=1}^m \sum_{j=1}^{n_0(\gamma)} (I_d-\gamma A)^k (I_d-\gamma A)^{n_0(\gamma)-j} R_j + \gamma \sum_{k=1}^m \sum_{i=1}^{k} (I_d-\gamma A)^{k-i} R_{n_0(\gamma)+i} \\
&= \left(\sum_{k=1}^m (I_d-\gamma A)^k\right) \gamma \sum_{j=1}^{n_0(\gamma)} (I_d-\gamma A)^{n_0(\gamma)-j} R_j + \gamma \sum_{i=1}^m \left(\sum_{k=i}^m (I_d-\gamma A)^{k-i}\right)R_{n_0(\gamma)+i} \\
&= \left(\sum_{k=1}^m (I_d-\gamma A)^k\right) I_{1,n_0(\gamma)} + \sum_{i=1}^m A^{-1}\!\left(I_d-(I_d-\gamma A)^{m-i+1}\right)R_{n_0(\gamma)+i} \\
&= \left(\sum_{k=1}^m (I_d-\gamma A)^k\right) I_{1,n_0(\gamma)} + A^{-1}\sum_{i=1}^m R_{n_0(\gamma)+i} - A^{-1}\sum_{i=1}^m (I_d-\gamma A)^{m-i+1}R_{n_0(\gamma)+i}.
\end{align*}

So by triangle inequality and noting $|(I_d-\gamma A)^{m-i+1} | \leq 1$
\begin{align*}
\mathbb{E} \left[ \max_{1 \le m \le n} \frac{1}{\sqrt{n}} \left\| \sum_{k=1}^m I_{1,n_0(\gamma)+k} \right\|_2 \right] &\le \frac{1}{\sqrt{n}\gamma\lambda_*} \mathbb{E}\|I_{1,n_0(\gamma)}\|_2 + \frac{2}{\sqrt{n}\lambda_*} \sum_{i=1}^n \mathbb{E}\|R_{n_0(\gamma)+i}\|_2.
\end{align*}

We now derive a bound on $|R_n|$.  Following \cite{NEURIPS2025_7aae9e3e}, (by local smoothness Assumption 3.2 in the paper), when $|\theta_{n-1}-\theta^*| \le \kappa$, we apply Taylor's theorem and we have $|R_n| \le L|\theta_{n-1}-\theta^*|^2$. For $|\theta_{n-1}-\theta^*| > \kappa$, the Lipschitz continuity of the gradient implies $|R_n| \le 2L_2|\theta_{n-1}-\theta^*| \le 2L_2\kappa^{-1}|\theta_{n-1}-\theta^*|^2$. So we finally have $|R_n| \le C_0|\Delta_{n-1}|^2$ where $C_0=\max\{L,\,2L_2\kappa^{-1}\}$. So $\mathbb{E}\|R_j\|_2 \le C_0 \mathbb{E}\|\Delta_{j-1}\|_2^2 \le C_0 \left( \rho_2^2(1-\alpha_0\gamma)^{j-1} + \frac{3M_2^2}{\alpha_0}\gamma \right)$ by  theorem 2.3 of \cite{NEURIPS2025_7aae9e3e}. Hence,

\begin{align*}
\frac{1}{\sqrt{n}\gamma\lambda_*} \mathbb{E}\|I_{1,n_0(\gamma)}\|_2 &\le \frac{1}{\sqrt{n}\gamma\lambda_*} \gamma \sum_{j=1}^{n_0(\gamma)} (1-\gamma\lambda_*)^{n_0(\gamma)-j} C_0 \left( \rho_2^2(1-\alpha_0\gamma)^{j-1} + \frac{3M_2^2}{\alpha_0}\gamma \right) \\
&\le \frac{C_0}{\sqrt{n}\lambda_*} \left( n_0(\gamma) \rho_2^2 (1-\gamma\lambda)^{n_0(\gamma)-1} + \frac{3M_2^2}{\alpha_0\gamma\lambda_*} \gamma \right) \\
&= \frac{C_0 \rho_2^2}{\lambda_*(1-\gamma\lambda)} \frac{n_0(\gamma)}{\sqrt{n}} (1-\gamma\lambda)^{n_0(\gamma)} + \frac{3C_0 M_2^2}{\sqrt{n} \alpha_0 \lambda_*^2} \rightarrow 0
\end{align*}

since as before we have that $(1-\gamma\lambda)^{n_0(\gamma)} \le n\gamma^2$ so $\frac{n_0(\gamma)}{\sqrt{n}} (1-\gamma\lambda)^{n_0(\gamma)}  \leq \sqrt{n}\gamma^2 n_0(\gamma)$. Because $n_0(\gamma) \gamma \propto |\log(\sqrt{n}\gamma)|$, this is $\mathcal{O}(\sqrt{n}\gamma |\log(\sqrt{n}\gamma)|)$ which converges to zero since we assume $\sqrt{n}\gamma \to 0$.  

For the second term:
\begin{align*}
\frac{2}{\sqrt{n}\lambda_*} \sum_{i=1}^n \mathbb{E}\|R_{n_0(\gamma)+i}\|_2 &\le \frac{2C_0}{\sqrt{n}\lambda_*} \sum_{i=1}^n \left( \rho_2^2(1-\alpha_0\gamma)^{n_0(\gamma)+i-1} + \frac{3M_2^2}{\alpha_0}\gamma \right) \\
&\le \frac{2C_0 \rho_2^2}{\sqrt{n}\lambda_*} (1-\alpha_0\gamma)^{n_0(\gamma)} \sum_{i=1}^\infty (1-\alpha_0\gamma)^{i-1} + \frac{2C_0}{\sqrt{n}\lambda_*} n \frac{3M_2^2}{\alpha_0}\gamma \\
&\le \frac{2C_0 \rho_2^2}{\sqrt{n}\lambda_*} (n\gamma^2) \frac{1}{\alpha_0\gamma} + \frac{6C_0 M_2^2}{\alpha_0 \lambda_*} \sqrt{n}\gamma \\
&= \frac{2C_0 \rho_2^2 \sqrt{n}\gamma}{\alpha_0 \lambda_*} + \frac{6C_0 M_2^2}{\alpha_0 \lambda_*} \sqrt{n}\gamma \rightarrow 0
\end{align*}
since we require $\sqrt{n}\gamma \to 0$. Therefore the Taylor term is negligible by Markov's inequality. Now we need to account for our burn in $n_0(\gamma)$ iterates in the limit.
Writing $Q=I_d-\gamma A$, then we consider
\begin{align*}
I_{2,n_0(\gamma)+k}
&=Q^kI_{2,n_0(\gamma)}+I_{2,k}^{\rm new},\\
I_{2,k}^{\rm new}
&=\gamma\sum_{j=1}^kQ^{k-j}D_{n_0(\gamma)+j},\\
L_{n_0(\gamma)+k}
&=Q^kL_{n_0(\gamma)}+L_k^{\rm new},\\
L_k^{\rm new}
&=\gamma\sum_{j=1}^kQ^{k-j}Z_{n_0(\gamma)+j}.
\end{align*}
We consider $U=I_2$ and $U=L$ and our existing bounds give $\mathbb E\|U_{n_0(\gamma)}\|_2^2\le C\gamma$.
Thus the contribution of this term before burn in is negligible since
\begin{align*}
&\mathbb E\max_{m\le n}
\left\|
\frac1{\sqrt n}\sum_{k=1}^mQ^kU_{n_0(\gamma)}
\right\|_2^2\\
&\le\frac1n
\left(\sum_{k=1}^{\infty}(1-\gamma\lambda_*)^k\right)^2
\mathbb E\|U_{n_0(\gamma)}\|_2^2\\
&\le\frac{C}{n\gamma\lambda_*^2}
\longrightarrow0,
 \text{since $n\gamma\to\infty$.}
\end{align*}

For the new $I_2$ error, by the same convolution bound and Doob maximal bound argument as before:
\begin{align*}
&\mathbb E\max_{m\le n}
\left\|\frac1{\sqrt n}\sum_{k=1}^m
I_{2,k}^{\rm new}\right\|_2^2\\
&\le\frac Cn\,
\mathbb E\max_{m\le n}
\left\|\sum_{j=1}^mD_{n_0(\gamma)+j}\right\|_2^2,
 \text{by Lemma~\ref{lem:max_convolution_bound},}\\
&\le\frac{4C}{n}\sum_{j=1}^n
\mathbb E\|D_{n_0(\gamma)+j}\|_2^2,
 \text{by Doob and martingale orthogonality,}\\
&\le\frac{16CL_2^2}{n}\sum_{j=1}^n
\left\{
\rho_2^2(1-\alpha_0\gamma)^{n_0(\gamma)+j-1}
+\frac{3M_2^2}{\alpha_0}\gamma
\right\},
 \text{by the existing moment bound,}\\
&\le\frac{16CL_2^2\rho_2^2
(1-\alpha_0\gamma)^{n_0(\gamma)}}
{\alpha_0 n\gamma}
+\frac{48CL_2^2M_2^2}{\alpha_0}\gamma
\longrightarrow0.
\end{align*}
Here $C$ is independent of $n$ and $\gamma$.

Since the $Z_j$ are i.i.d.,
\begin{align*}
(L_k^{\rm new})_{k\le n}
&\overset{d}{=}(L_k)_{k\le n},\\
\frac1{\sqrt n}\sum_{k=1}^{\lfloor nt\rfloor}
L_k^{\rm new}
&\Rightarrow A^{-1}S^{1/2}B(t),
 \text{by Lemma~\ref{sum FCLT lin iterates}.}
\end{align*}
Since we have proved all other errors converge to zero, by Slutsky's theorem we have the result.
\end{proof}

\begin{theorem} \label{romberg fclt}

We run two SGD iterations at the same initialization and with the same noise but one with step size $\gamma$ and the other $2\gamma$. Let $\hat{\theta}_k = 2\theta_k(\gamma) - \theta_k(2\gamma)$. Let $n \to \infty$ and $\gamma \to 0$ such that $n\gamma \to \infty$, $n\gamma^3 \to 0$, and $\sqrt{n}\gamma \to \infty$. Let $\hat{S}_m = \sum_{k=1}^m (\hat{\theta}_k - \theta^*)$. Then:

$$\frac{1}{\sqrt{n}} \hat{S}_{\lfloor nt \rfloor} \Rightarrow A^{-1} S^{1/2} B(t) \quad \text{in } D([0,1], \mathbb{R}^d)$$
\end{theorem}

\begin{proof}
Let $\Delta_k(x) = \theta_k(x) - \theta^*$ then
\begin{align*}
& \hat{\Delta}_k := 2\theta_k(\gamma) - \theta_k(2\gamma) - \theta^* \\
&= 2\Delta_k(\gamma) - \Delta_k(2\gamma) \\
&= \big( 2(I_d - \gamma A)^k - (I_d - 2\gamma A)^k \big) \Delta_0 - \big( 2L_k(\gamma) - L_k(2\gamma) \big) + \big( 2I_{1,k}(\gamma) - I_{1,k}(2\gamma) \big) + \big( 2I_{2,k}(\gamma) - I_{2,k}(2\gamma) \big)
\end{align*}

So we analyze each term in $\frac{1}{\sqrt{n}} \hat{S}_{\lfloor nt \rfloor} = \frac{1}{\sqrt{n}} \sum_{k=1}^{\lfloor nt \rfloor} \hat{\Delta}_k$. We start with the stochastic error for which,
\begin{align*}
& \max_{1 \le m \le n} \frac{1}{\sqrt{n}} \left| \sum_{k=1}^m \big( 2I_{2,k}(\gamma) - I_{2,k}(2\gamma) \big) \right|_2 \\
&\le 2\max_{1 \le m \le n} \frac{1}{\sqrt{n}} \left| \sum_{k=1}^m I_{2,k}(\gamma) \right|_2 + \max_{1 \le m \le n} \frac{1}{\sqrt{n}} \left| \sum_{k=1}^m I_{2,k}(2\gamma) \right|_2 \\
&\xrightarrow{\ \mathbb{P}\ } 0 + 0 = 0 \text{, by Lemma \ref{Sum stochastic error to 0}, since } n\gamma \to \infty,\ n(2\gamma) \to \infty,\ \gamma \to 0
\end{align*}

For the initialization error, 
\begin{align*}
& \left| \frac{1}{\sqrt{n}} \sum_{k=1}^{\lfloor nt \rfloor} \big( 2(I_d - \gamma A)^k \Delta_0 - (I_d - 2\gamma A)^k \Delta_0 \big) \right|_2 \\
&\le \left| \frac{2}{\sqrt{n}} (I_d - \gamma A) \big(I_d - (I_d - \gamma A)^{\lfloor nt \rfloor}\big) (\gamma A)^{-1} \Delta_0 \right|_2 + \left| \frac{1}{\sqrt{n}} (I_d - 2\gamma A) \big(I_d - (I_d - 2\gamma A)^{\lfloor nt \rfloor}\big) (2\gamma A)^{-1} \Delta_0 \right|_2 \\
&\le \frac{2}{\sqrt{n}\gamma} \left| A^{-1} \right| (1 - \gamma \lambda) \left| \Delta_0 \right|_2 + \frac{1}{2\sqrt{n}\gamma}  \left| A^{-1} \right| (1 - 2 \gamma \lambda) \left| \Delta_0 \right|_2\\
&\le \frac{2 |A^{-1}|_2 |\Delta_0|_2}{\sqrt{n}\gamma} + \frac{|A^{-1}|_2 |\Delta_0|_2}{2\sqrt{n}\gamma} \\
&= \frac{5 |A^{-1}|_2 |\Delta_0|_2}{2 \sqrt{n}\gamma} \\
&\to 0 \quad \text{, since } \sqrt{n}\gamma \to \infty
\end{align*}

For the linear sequence, by the same Fubini swap as in Lemma \ref{sum FCLT lin iterates} we have that
\begin{align*}
& \frac{-1}{\sqrt{n}} \sum_{k=1}^{\lfloor nt \rfloor} \big( 2L_k(\gamma) - L_k(2\gamma) \big) \\
&= \frac{-2}{\sqrt{n}} \left( A^{-1} \sum_{j=1}^{\lfloor nt \rfloor} Z_j - A^{-1} \sum_{j=1}^{\lfloor nt \rfloor} (I_d - \gamma A)^{\lfloor nt \rfloor-j+1} Z_j \right) \\
&\quad+ \frac{1}{\sqrt{n}} \left( A^{-1} \sum_{j=1}^{\lfloor nt \rfloor} Z_j - A^{-1} \sum_{j=1}^{\lfloor nt \rfloor} (I_d - 2\gamma A)^{\lfloor nt \rfloor-j+1} Z_j \right) \\
&\text{(Fubini sum trick)} \\
&= -A^{-1} \frac{1}{\sqrt{n}} \sum_{j=1}^{\lfloor nt \rfloor} Z_j + \frac{2}{\sqrt{n}} A^{-1} \sum_{j=1}^{\lfloor nt \rfloor} (I_d - \gamma A)^{\lfloor nt \rfloor-j+1} Z_j - \frac{1}{\sqrt{n}} A^{-1} \sum_{j=1}^{\lfloor nt \rfloor} (I_d - 2\gamma A)^{\lfloor nt \rfloor-j+1} Z_j \\
&\Rightarrow A^{-1} S^{1/2} B(t) + 0 - 0 \text{, by the exact same proof as Lemma \ref{sum FCLT lin iterates}} \\
&\stackrel{d}{=} A^{-1} S^{1/2} B(t) 
\end{align*}

Finally for the nonlinear error term, $\frac{1}{\sqrt{n}} \sum_{k=1}^{\lfloor nt \rfloor} \big( 2I_{1,k}(\gamma) - I_{1,k}(2\gamma) \big)$, we have that
\begin{align*}
& \max_{1 \le m \le n} \frac{1}{\sqrt{n}} \left| \sum_{k=1}^m \big( 2I_{1,k}(\gamma) - I_{1,k}(2\gamma) \big) \right|_2 \\
&\le 2\max_{1 \le m \le n} \frac{1}{\sqrt{n}} \left| \sum_{k=1}^m \tilde{I}_{1,k}(\gamma) \right|_2 + \max_{1 \le m \le n} \frac{1}{\sqrt{n}} \left| \sum_{k=1}^m \tilde{I}_{1,k}(2\gamma) \right|_2 + \max_{1 \le m \le n} \left| \frac{1}{\sqrt{n}} \sum_{k=1}^m \big( 2\mathbb{E}[I_{1,k}(\gamma)] - \mathbb{E}[I_{1,k}(2\gamma)] \big) \right|_2 \\
&\xrightarrow{\ \mathbb{P}\ } 0 + 0 + \max_{1 \le m \le n} \left| \frac{1}{\sqrt{n}} \sum_{k=1}^m \big( 2\mathbb{E}[I_{1,k}(\gamma)] - \mathbb{E}[I_{1,k}(2\gamma)] \big) \right|_2 \quad \text{, by Lemma \ref{Sum centered taylor error to 0}}
\end{align*}

Again by the same Fubini sum swap from Theorem \ref{thm:fclt_init_bias},
\begin{align*}
& \frac{1}{\sqrt{n}} \sum_{k=1}^{\lfloor nt \rfloor} \big( 2\mathbb{E}[I_{1,k}(\gamma)] - \mathbb{E}[I_{1,k}(2\gamma)] \big) \\
&= A^{-1} \frac{1}{\sqrt{n}} \sum_{j=1}^{\lfloor nt \rfloor} \big( 2\mathbb{E}[R_j(\gamma)] - \mathbb{E}[R_j(2\gamma)] \big) - 2A^{-1} \frac{1}{\sqrt{n}} \sum_{j=1}^{\lfloor nt \rfloor} (I_d - \gamma A)^{\lfloor nt \rfloor-j+1} \mathbb{E}[R_j(\gamma)] \\
&+ A^{-1} \frac{1}{\sqrt{n}} \sum_{j=1}^{\lfloor nt \rfloor} (I_d - 2\gamma A)^{\lfloor nt \rfloor-j+1} \mathbb{E}[R_j(2\gamma)]
\end{align*}

For $\eta\in\{\gamma,2\gamma\}$, the Fubini boundary is uniformly bounded by
\[
\max_{m\le n}\left\|\frac{A^{-1}}{\sqrt n}\sum_{j=1}^m
(I_d-\eta A)^{m-j+1}\mathbb E R_j(\eta)\right\|_2
\le\frac{\|A^{-1}\|_2C_0\rho_2^2}{e\lambda\sqrt n\eta}
+\frac{3\|A^{-1}\|_2C_0M_2^2}{\lambda\alpha_0\sqrt n}=o(1).
\]
Thus it only remains to look at $A^{-1} \frac{1}{\sqrt{n}} \sum_{j=1}^{\lfloor nt \rfloor} \big( 2\mathbb{E}[R_j(\gamma)] - \mathbb{E}[R_j(2\gamma)] \big)$. Here we heavily use the results from Theorem \ref{thm:fclt_init_bias} and expand $$A^{-1} \frac{1}{\sqrt{n}} \sum_{j=1}^{\lfloor nt \rfloor} \mathbb{E}[R_j(\eta)]= -\frac12 A^{-1} \frac{1}{\sqrt{n}} \sum_{j=1}^{\lfloor nt \rfloor} \nabla^3 G(\theta^*) \mathbb{E}[L_{j-1}(\eta)L_{j-1}(\eta)^\top] + A^{-1}\frac{1}{\sqrt{n}} \sum_{j=1}^{\lfloor nt \rfloor}\text{Rem}_j(\eta).$$ Then for each of the two step sizes $\eta$ from Theorem \ref{thm:fclt_init_bias},
\begin{align*}
& \left| A^{-1} \frac{1}{\sqrt{n}} \sum_{j=1}^{\lfloor nt \rfloor} \text{Rem}_j(\eta) + \frac{1}{\sqrt{n}} \sum_{j=1}^{\lfloor nt \rfloor} A^{-1} R_{det}(x((j-1) \eta)) \right|_2 \\
&\le \frac{\|A^{-1}\|_2}{2\sqrt{n}} \sum_{j=1}^{\lfloor nt \rfloor} \left| \mathbb{E}\Big[ \big( \nabla^3 G(\tilde{\theta}_{j-1}) - \nabla^3 G(\theta^*) \big) [S_{j-1}, S_{j-1}] \Big] \right|_2 + \\
&\frac{\|A^{-1}\|_2}{2\sqrt{n}} \sum_{j=1}^{\lfloor nt \rfloor} |\nabla^3 G(\theta^*)|_2 |\mathcal{E}_{j-1}|_2 + \left| A^{-1} \frac{1}{\sqrt{n}} \sum_{j=1}^{\lfloor nt \rfloor} \big( R_{det}(H_{j-1}) - R_{det}(x((j-1)\eta)) \big) \right|_2 \\
&\le \|A^{-1}\|_2\left( \frac{L_4 \rho_2 \sqrt{C_S}}{\alpha_0 \sqrt{n}} + t \frac{L_4 \sqrt{3C_S} M_2}{2\alpha_0^{1/2}} \sqrt{n}\, \eta^{1.5} \right) \\
&+ \frac{\|A^{-1}\|_2|\nabla^3 G(\theta^*)|_2}{2} \left( \frac{2 C_{\phi 1}}{\lambda^2 \sqrt{n}} + \frac{32 C_{R3}}{\lambda^3 \sqrt{n}} + 2t\sqrt{n}(C_{\phi 2} + C_{R4}) \eta^2 + \frac{2^{5/2} C_{X1}}{\lambda^{3/2} \sqrt n} + \frac{8 C_{X2}}{\lambda^2 \sqrt{n}} + 2t\sqrt{n}\, C_{X3} \eta^{1.5} \right) \\
&+ \left( \frac{|A^{-1}|_2 C_L C_E \rho_2}{\alpha_0 \sqrt{n}} + \frac{|A^{-1}|_2 C_L C_E M_2 \sqrt{3}}{2\sqrt{\alpha_0}} t\sqrt{n}\, \eta^{1.5} + \frac{|A^{-1}|_2 C_L C_E |\Delta_0|_2}{\mu \sqrt{n}} \right) \text{by the bounds of Theorem } \ref{thm:fclt_init_bias} \\
&\to 0 \quad \text{, since } n\eta^3 \to 0 \Rightarrow \sqrt{n}\eta^{1.5}\to 0,\ \text{and } n\eta\to\infty \Rightarrow \sqrt{n\eta}\to\infty
\end{align*}

and for the final deterministic part
\begin{align*}
& \left| \frac{1}{\sqrt{n}} \sum_{j=1}^{\lfloor nt \rfloor} A^{-1} R_{det}(x((j-1) \eta)) \right|_2 \\
&\le \frac{|A^{-1}|_2}{\sqrt{n}} \sum_{j=1}^{\lfloor nt \rfloor} \frac{C_L}{2} |\Delta_0|_2^2 e^{-2\mu (j-1)\eta}  \text{, by Theorem \ref{thm:fclt_init_bias}} \\
&\le \frac{|A^{-1}|_2 C_L |\Delta_0|_2^2}{2\sqrt{n}} \cdot \frac{1}{1 - e^{-2\mu\eta}} \\
&\le \frac{|A^{-1}|_2 C_L |\Delta_0|_2^2}{2\sqrt{n} \eta} \cdot \frac{2}{2\mu} \quad \text{(} 1-e^{-2\mu\eta} \ge \mu\eta \text{ for } \mu\eta\le1/2\text{)} \\
&= \frac{|A^{-1}|_2 C_L |\Delta_0|_2^2}{2\mu\sqrt{n}\eta} \\
&\to 0 \quad \text{, when } \sqrt{n}\eta\to\infty, \eta \to 0
\end{align*}

Therefore, all the nonlinear remainder terms converge to zero and so applying Lemma \ref{linear covariance expansion}:
\begin{align*}
& \left| -\frac{1}{2} A^{-1} \frac{1}{\sqrt{n}} \sum_{j=1}^{\lfloor nt \rfloor} \nabla^3 G(\theta^*) \big[ 2\mathbb{E}[L_{j-1}(\gamma)L_{j-1}(\gamma)^\top] - \mathbb{E}[L_{j-1}(2\gamma)L_{j-1}(2\gamma)^\top] \big] \right|_2 \\
&= \Bigg| -\frac{1}{2} A^{-1} \frac{1}{\sqrt{n}} \sum_{j=1}^{\lfloor nt \rfloor} \nabla^3 G(\theta^*) \Bigg[ 2\Big( \gamma\Gamma - \gamma(I_d-\gamma A)^{j-1}\Gamma(I_d-\gamma A)^{j-1} + \gamma^3 \sum_{i=0}^{j-2}(I_d-\gamma A)^i A\Gamma A(I_d-\gamma A)^i \Big) \\
&- \Big( 2\gamma\Gamma - 2\gamma(I_d-2\gamma A)^{j-1}\Gamma(I_d-2\gamma A)^{j-1} + 8\gamma^3\sum_{i=0}^{j-2}(I_d-2\gamma A)^i A\Gamma A(I_d-2\gamma A)^i \Big) \Bigg] \Bigg|_2 \\
&= \Bigg| -\frac{1}{2} A^{-1} \frac{1}{\sqrt{n}} \sum_{j=1}^{\lfloor nt \rfloor} \nabla^3 G(\theta^*) \Bigg[ -2\gamma(I_d-\gamma A)^{j-1}\Gamma(I_d-\gamma A)^{j-1} + 2\gamma(I_d-2\gamma A)^{j-1}\Gamma(I_d-2\gamma A)^{j-1} \\
&+ 2\gamma^3\sum_{i=0}^{j-2}(I_d-\gamma A)^i A\Gamma A(I_d-\gamma A)^i - 8\gamma^3\sum_{i=0}^{j-2}(I_d-2\gamma A)^i A\Gamma A(I_d-2\gamma A)^i \Bigg] \Bigg|_2 \\
&\le \frac{|A^{-1}|_2 |\nabla^3 G(\theta^*)|_2 |\Gamma|_2}{2\sqrt{n}} \sum_{j=1}^{\lfloor nt \rfloor} \Big( 2\gamma(1-\lambda_*\gamma)^{2j-2} + 2\gamma(1-2\lambda_*\gamma)^{2j-2} \Big) \\
&+ \frac{|A^{-1}|_2 |\nabla^3 G(\theta^*)|_2 |A\Gamma A|_2}{2\sqrt{n}} \sum_{j=1}^{\lfloor nt \rfloor} \Big( 2\gamma^3 \sum_{i=0}^{j-2}(1-\lambda_*\gamma)^{2i} + 8\gamma^3\sum_{i=0}^{j-2}(1-2\lambda_*\gamma)^{2i} \Big) \\
&\le \frac{|A^{-1}|_2 |\nabla^3 G(\theta^*)|_2 |\Gamma|_2}{2\sqrt{n}} \left( \frac{2\gamma}{\lambda_*\gamma} + \frac{2\gamma}{2\lambda_*\gamma} \right) + \frac{|A^{-1}|_2 |\nabla^3 G(\theta^*)|_2 |A\Gamma A|_2}{2\sqrt{n}} \sum_{j=1}^{\lfloor nt \rfloor} \left( \frac{2\gamma^3}{\lambda_*\gamma} + \frac{8\gamma^3}{2\lambda_*\gamma} \right) \\
&\le \frac{3 |A^{-1}|_2 |\nabla^3 G(\theta^*)|_2 |\Gamma|_2}{2\lambda_* \sqrt{n}} + \frac{3t |A^{-1}|_2 |\nabla^3 G(\theta^*)|_2 |A\Gamma A|_2}{\lambda_*} \sqrt{n}\gamma^2 \\
&\to 0 + 0 = 0 \quad \text{, since } \sqrt{n}\gamma^2 = (\sqrt{n}\gamma^{1.5})\sqrt{\gamma} \to 0.
\end{align*}

Therefore, putting all this together
\begin{align*}
& \frac{1}{\sqrt{n}} \hat{S}_{\lfloor nt \rfloor} \\
&= \frac{1}{\sqrt{n}} \sum_{k=1}^{\lfloor nt \rfloor} \big( 2(I_d - \gamma A)^k - (I_d - 2\gamma A)^k \big) \Delta_0 - \frac{1}{\sqrt{n}} \sum_{k=1}^{\lfloor nt \rfloor} \big( 2L_k(\gamma) - L_k(2\gamma) \big) \\
&+ \frac{1}{\sqrt{n}} \sum_{k=1}^{\lfloor nt \rfloor} \big( 2I_{1,k}(\gamma) - I_{1,k}(2\gamma) \big) + \frac{1}{\sqrt{n}} \sum_{k=1}^{\lfloor nt \rfloor} \big( 2I_{2,k}(\gamma) - I_{2,k}(2\gamma) \big) \\
&\Rightarrow 0 - \big( -A^{-1}S^{1/2}B(t) \big) + 0 + 0 \text{, by multivariate Slutsky's theorem} \\
&= A^{-1} S^{1/2} B(t)
\end{align*}
in $D([0,1], \mathbb{R}^d)$.
\end{proof}
\section{Trajectory FCLT}
\begin{theorem}\label{thm:trajectory_fclt}
Assume Assumptions 2.1 and 2.2 of \cite{NEURIPS2025_7aae9e3e} hold for $p = 2$. Also impose Assumption~\ref{ass:smooth} and~\eqref{eq:steps-main} at order $r=2$. Let $\gamma\to0$ and choose $n_0(\gamma)=2\lceil\nu\log(1/\gamma)/(2\gamma)\rceil$ for a constant $\nu>3/\lambda$ with the paper-wide $\lambda$. For the change in index $k(t) = n_0(\gamma) + \lfloor t/\gamma \rfloor$, define
$$Y_\gamma(t) = \frac{\theta_{k(t)} - \theta^*}{\sqrt{\gamma}}.$$
Then, $Y_\gamma(\cdot)$ converges weakly in the Skorokhod space $D([0, T], \mathbb{R}^d)$ to the stationary Ornstein-Uhlenbeck process $X(\cdot)$ satisfying:
$$dX_t = -A X_t dt + S^{1/2} dB_t,$$
where $A = \nabla^2 G(\theta^*)$, $S = \mathbb{E}[\nabla g(\theta^*, X) \nabla g(\theta^*, X)^\top]$, and $B_t$ is a standard $d$-dimensional Brownian motion. The process is stationary with invariant measure $\mathcal{N}(0, \Gamma)$, where $A\Gamma + \Gamma A = S$.
\end{theorem}

\begin{proof}
Consider the decomposition of $\Delta_k = \theta_k - \theta^*$ from  \cite{NEURIPS2025_7aae9e3e}, which is 
\begin{align*}
\frac{1}{\sqrt{\gamma}} \Delta_k &= \frac{1}{\sqrt{\gamma}} ((I_d - \gamma A)^k \Delta_0 - L_k + I_{1,k} + I_{2,k}),
\end{align*}
where $L_k = \gamma \sum_{j=1}^k (I_d - \gamma A)^{k-j} Z_j$, $I_{1,k} = \gamma \sum_{j=1}^k (I_d - \gamma A)^{k-j} R_j$, and $I_{2,k} = \gamma \sum_{j=1}^k (I_d - \gamma A)^{k-j} D_j$. $R_j = A\Delta_{j-1} - \nabla G(\theta_{j-1})$, and $D_j = \nabla G(\theta_{j-1}) - \nabla g(\theta_{j-1}, X_j) + \nabla g(\theta^*, X_j)$. Let $k(t) = n_0(\gamma) + \lfloor t/\gamma \rfloor$ and $N = n_0(\gamma) + \lfloor T/\gamma \rfloor$. 

We will first prove that all the remainder terms converge uniformly to 0 in probability and then show the linear term converges to our desired OU process.

We start with the initialization error which we bound by
\begin{align*}
&\sup_{t \in [0,T]} \left\| \frac{1}{\sqrt{\gamma}} (I_d - \gamma A)^{k(t)} \Delta_0 \right\|_2 \\
&\le \frac{1}{\sqrt{\gamma}} (1 - \gamma \lambda_*)^{n_0(\gamma)} \|\Delta_0\|_2 \text{, since $n_0(\gamma) \leq k(t)$ $ \forall t\in [0,T]$,} \\
&\le \frac{1}{\sqrt{\gamma}} \exp(-\lambda_* n_0(\gamma) \gamma) \|\Delta_0\|_2 \\
&\le \frac{1}{\sqrt{\gamma}} \exp(-\lambda_* \nu \log(1/\gamma)) \|\Delta_0\|_2 \text{, since $n_0(\gamma) \gamma \geq \nu \log(1/\gamma)$ by assumption,} \\
&= \gamma^{\lambda_* \nu - 1/2} \|\Delta_0\|_2.
\end{align*}
Again we assume that $\nu > 3/\lambda \ge 3/\lambda_*$, so $\lambda_*\nu - 1/2 > 1$, and taking $\gamma \to 0$ we see that $\frac{1}{\sqrt{\gamma}} (I_d - \gamma A)^{k(t)} \Delta_0$ converges to zero uniformly in $t$.

\bigskip

We now handle the $I_{2,k}$ term. Recall that we defined $\rho_2^2 = \|\Delta_0\|_2^2$ and let $m = \lfloor n_0(\gamma)/2 \rfloor$. For $k \ge n_0(\gamma)$ we decompose $I_{2,k}$ as
\begin{align*}
I_{2,k} &= (I_d - \gamma A)^{k - m} \left( \gamma \sum_{j=1}^{m} (I_d - \gamma A)^{m-j} D_j \right) + \gamma \sum_{j=m + 1}^k (I_d - \gamma A)^{k-j} D_j \\
&= (I_d - \gamma A)^{k - m} I_{2, m} + \gamma \sum_{j=m + 1}^k (I_d - \gamma A)^{k-j} D_j.
\end{align*}

Starting with a bound for $I_{2,m}$
\begin{align*}
\sup_{k \ge n_0(\gamma)} \frac{1}{\sqrt{\gamma}} \left\| (I_d - \gamma A)^{k - m} I_{2, m} \right\|_2 
&\le \frac{1}{\sqrt{\gamma}} (1 - \gamma \lambda_*)^{n_0(\gamma)/2} \| I_{2, m} \|_2 \\
&\le \gamma^{(\lambda_* \nu - 1)/2} \| I_{2, m} \|_2 \text{, by the same logic as above.}
\end{align*}
 By the stochastic Lipschitzness assumption on the gradients(Assumption 2.2) and Theorem 2.3 of \cite{NEURIPS2025_7aae9e3e}
 $\mathbb{E}\|D_j\|_2^2 \le 4L_2^2 \mathbb{E}\|\Delta_{j-1}\|_2^2 \le 4L_2^2 (\rho_2^2(1-\alpha_0\gamma)^{j-1} + \frac{3M_2^2}{\alpha_0}\gamma)$ which is a bound we will employ below. Since the $D_j$ are uncorrelated and mean zero
\begin{align*}
&\mathbb{E}\| I_{2, m} \|_2^2 \\
&\le \gamma^2 \sum_{j=1}^{m} \| (I_d - \gamma A)^{m-j} \|_2^2 \mathbb{E}\|D_j\|_2^2 \\
&\le 4L_2^2 \gamma^2 \sum_{j=1}^{m} (1 - \gamma \lambda_*)^{2(m-j)} \left( \rho_2^2(1-\alpha_0\gamma)^{j-1} + \frac{3M_2^2}{\alpha_0}\gamma \right) \text{, by the above bound on $\mathbb{E}\|D_j\|_2^2$} \\
&\le 4L_2^2\rho_2^2\gamma^2m(1-\lambda\gamma)^{m-1}
+\frac{12L_2^2M_2^2}{\alpha_0}\gamma^3\sum_{j=0}^{\infty}(1-\gamma\lambda_*)^j\\
&\le 4L_2^2 \rho_2^2 \gamma^2 m \exp(- (m-1) \gamma \lambda) + \frac{12 L_2^2 M_2^2}{\alpha_0 \lambda_*} \gamma^2 \\
&\le 4L_2^2 \rho_2^2 \gamma^2 \frac{e^{\gamma \lambda}}{e \gamma \lambda} + \frac{12 L_2^2 M_2^2}{\alpha_0 \lambda_*} \gamma^2 \text{, since $x \exp(-x) \leq \exp(-1)$ by the supporting hyperplane theorem.}
\end{align*}

Hence $\mathbb E\left[\sup_{k \ge n_0(\gamma)} \frac{1}{\gamma} \left\| (I_d - \gamma A)^{k - m} I_{2, m} \right\|^2_2\right] \leq \gamma^{(\lambda_* \nu - 1)}(4L_2^2 \rho_2^2 \gamma \frac{e^{\gamma \lambda}}{e \lambda} + \frac{12 L_2^2 M_2^2}{\alpha_0 \lambda_*} \gamma^2) \underset{\gamma \rightarrow 0 }{\rightarrow} 0$.

Now to bound $\gamma \sum_{j=m + 1}^k (I_d - \gamma A)^{k-j} D_j$ denote the martingale $\tilde{M}_k = \sum_{j=m + 1}^k D_j$ so
\begin{align*}
&\gamma \sum_{j=m+1}^{k}(I_d-\gamma A)^{k-j}D_j \\
&= \gamma \sum_{j=m+1}^{k}(I_d-\gamma A)^{k-j}(\tilde M_j-\tilde M_{j-1}), \text{ with $M_m := 0$} \\
&= \gamma \sum_{j=m+1}^{k}(I_d-\gamma A)^{k-j}\tilde M_j
   - \gamma \sum_{j=m}^{k-1}(I_d-\gamma A)^{k-1-j}\tilde M_j \\
&= \gamma \tilde M_k
   + \gamma \sum_{j=m+1}^{k-1}
   \Big[(I_d-\gamma A)^{k-j}-(I_d-\gamma A)^{k-1-j}\Big]\tilde M_j \\
&= \gamma \tilde M_k
   + \gamma \sum_{j=m+1}^{k-1}
   \Big[-\gamma A(I_d-\gamma A)^{k-j-1} +(I_d-\gamma A)^{k-j-1}-(I_d-\gamma A)^{k-1-j}\Big]\tilde M_j \\
&= \gamma \tilde M_k
   - \gamma^2 A \sum_{j=m+1}^{k-1}(I_d-\gamma A)^{k-1-j}\tilde M_j .
\end{align*}
Hence,
\begin{align*}
\max_{n_0(\gamma) \le k \le N} \|\gamma \sum_{j=m + 1}^k (I_d - \gamma A)^{k-j} D_j\|_2 
&\le \gamma \max_{j \le N} \|\tilde{M}_j\|_2 + \gamma^2 \lambda^* \sum_{i=0}^\infty (1-\gamma \lambda_*)^i \max_{j \le N} \|\tilde{M}_j\|_2 \\
&\le \gamma \left( 1 + \frac{\lambda^*}{\lambda_*} \right) \max_{j \le N} \|\tilde{M}_j\|_2.
\end{align*}
Now by Doob's maximal $\mathcal{L}^2$ inequality and using that the $D_j$ are uncorrelated we get that
\begin{align*}
\mathbb{E} \left[ \max_{j \le N} \|\tilde{M}_j\|_2^2 \right] 
&\le 4 \sum_{j=m + 1}^N \mathbb{E}\|D_j\|_2^2 \\
&\le 16 L_2^2 \sum_{j=m + 1}^N \left( \rho_2^2(1-\alpha_0\gamma)^{n_0(\gamma)/2} + \frac{3M_2^2}{\alpha_0}\gamma \right) \text{, by the bound on $D_j$ from above} \\
&\le 16 L_2^2 N \left( \rho_2^2 \exp (-\alpha_0 \gamma n_0(\gamma) \frac{1}{2})  + \frac{3M_2^2}{\alpha_0}\gamma \right) \\
&\le 16 L_2^2 N \left( \rho_2^2 \gamma^{\alpha_0 \nu / 2} + \frac{3M_2^2}{\alpha_0}\gamma \right) \text{, since we assume $-\gamma n_0(\gamma) \leq \nu \log(\gamma)$} \\
&\leq  16 L_2^2  \left( \frac{\nu \log(1/\gamma) + T+2\gamma}{\gamma} \right)  \left( \rho_2^2 \gamma^{\alpha_0 \nu / 2} + \frac{3M_2^2}{\alpha_0}\gamma \right) \text{, by definition of $N$} \\
&= 16 L_2^2  \left( \nu \log(1/\gamma) + T+2\gamma \right)  \left( \rho_2^2 \gamma^{\frac{\alpha_0 \nu}{2} -1} + \frac{3M_2^2}{\alpha_0} \right) 
\end{align*}
Therefore,
\begin{align*}
&\mathbb{P}\!\left(\sup_{n_0(\gamma)\le k\le N}\frac{1}{\sqrt{\gamma}}\left\|\gamma\sum_{j=m+1}^{k}(I_d-\gamma A)^{k-j}D_j\right\|_2>\varepsilon\right) \\
&\le \frac{1}{\varepsilon^2}\mathbb{E}\!\left[\sup_{n_0(\gamma)\le k\le N}\frac{1}{\gamma}\left\|\gamma\sum_{j=m+1}^{k}(I_d-\gamma A)^{k-j}D_j\right\|_2^2\right] \\
&\le \frac{1}{\varepsilon^2}\frac{1}{\gamma}\gamma^2\left(1+\frac{\lambda^*}{\lambda_*}\right)^2\mathbb{E}\!\left[\max_{j\le N}\|\widetilde M_j\|_2^2\right] \\
&\le \frac{16L_2^2}{\varepsilon^2}\left(1+\frac{\lambda^*}{\lambda_*}\right)^2\gamma\bigl(\nu\log(1/\gamma)+T+2\gamma\bigr)\left(\rho_2^2\gamma^{\frac{\alpha_0\nu}{2}-1}+\frac{3M_2^2}{\alpha_0}\right) \\
&\le \frac{16L_2^2}{\varepsilon^2}\left(1+\frac{\lambda^*}{\lambda_*}\right)^2\gamma\bigl(\nu\log(1/\gamma)+T+2\gamma\bigr)\left(\rho_2^2+\frac{3M_2^2}{\alpha_0}\right) \\
& \underset{\gamma \rightarrow 0}{\to} 0.
\end{align*}

We now proceed to show the $I_{1,k}$ term goes to 0.

We apply the same split at $m = \lfloor n_0(\gamma)/2 \rfloor$.

\begin{align*}
I_{1,k}
&=(I_d-\gamma A)^{k-m}\left(\gamma\sum_{j=1}^{m}(I_d-\gamma A)^{m-j}R_j\right)+\gamma\sum_{j=m+1}^{k}(I_d-\gamma A)^{k-j}R_j\\
&=(I_d-\gamma A)^{k-m}I_{1,m}+\gamma\sum_{j=m+1}^{k}(I_d-\gamma A)^{k-j}R_j.
\end{align*}

We first derive a bound on $|R_n|$.  Following \cite{NEURIPS2025_7aae9e3e}, (by local smoothness Assumption 3.2 in the paper), when $|\theta_{n-1}-\theta^*| \le \kappa$, we apply Taylor's theorem and we have $|R_n| \le L|\theta_{n-1}-\theta^*|^2$. For $|\theta_{n-1}-\theta^*| > \kappa$, the Lipschitz continuity of the gradient implies $|R_n| \le 2L_2|\theta_{n-1}-\theta^*| \le 2L_2\kappa^{-1}|\theta_{n-1}-\theta^*|^2$. So we finally have $|R_n| \le C_0|\Delta_{n-1}|^2$ where $C_0=\max\{L,\,2L_2\kappa^{-1}\}$.

Hence
\begin{align*}
\mathbb E\|I_{1,m}\|_2
&\le\gamma\sum_{j=1}^{m}\|(I_d-\gamma A)^{m-j}\|_2\mathbb E\|R_j\|_2\\
&\le C_0\gamma\sum_{j=1}^{m}\mathbb E\|\Delta_{j-1}\|_2^2\\
&\le C_0\gamma\sum_{j=1}^{m}\left(\rho_2^2(1-\alpha_0\gamma)^{j-1}+\frac{3M_2^2}{\alpha_0}\gamma\right)\text{, by Theorem 2.3 of \cite{NEURIPS2025_7aae9e3e},}\\
&\le C_0\rho_2^2\gamma\sum_{j=0}^{\infty}(1-\alpha_0\gamma)^j+\frac{3C_0M_2^2}{\alpha_0}\gamma^2m\\
&\le\frac{C_0\rho_2^2}{\alpha_0}+\frac{3C_0M_2^2}{\alpha_0}\gamma^2m\\
&\le\frac{C_0\rho_2^2}{\alpha_0}+\frac{3C_0M_2^2}{2\alpha_0}n_0(\gamma)\gamma^2\\
&\le\frac{C_0\rho_2^2}{\alpha_0}+\frac{3C_0M_2^2}{2\alpha_0}\gamma\{\nu\log(1/\gamma)+2\gamma\}.
\end{align*}

Therefore
\begin{align*}
&\mathbb E\left[\sup_{k\ge n_0(\gamma)}\frac{1}{\sqrt{\gamma}}\|(I_d-\gamma A)^{k-m}I_{1,m}\|_2\right]\\
&\le\gamma^{(\lambda_*\nu-1)/2}\mathbb E\|I_{1,m}\|_2\\
&\le\gamma^{(\lambda_*\nu-1)/2}\left(\frac{C_0\rho_2^2}{\alpha_0}+\frac{3C_0M_2^2}{2\alpha_0}\gamma\{\nu\log(1/\gamma)+2\gamma\}\right)\\
&\underset{\gamma\to0}{\longrightarrow}0,
\end{align*}
since $\nu>3/\lambda\ge3/\lambda_*$ implies $(\lambda_*\nu-1)/2>1$.

Also
\begin{align*}
&\sup_{n_0(\gamma)\le k\le N}\left\|\gamma\sum_{j=m+1}^{k}(I_d-\gamma A)^{k-j}R_j\right\|_2\\
&\le\gamma\sum_{j=m+1}^{N}\|R_j\|_2\\
&\le\gamma\sum_{j=m+1}^{N}\|R_j\|_2\\
&\le C_0\gamma\sum_{j=m+1}^{N}\|\Delta_{j-1}\|_2^2.
\end{align*}

Therefore
\begin{align*}
&\mathbb P\!\left(\sup_{n_0(\gamma)\le k\le N}\frac{1}{\sqrt{\gamma}}\left\|\gamma\sum_{j=m+1}^{k}(I_d-\gamma A)^{k-j}R_j\right\|_2>\varepsilon\right)\\
&\le\frac{1}{\varepsilon\sqrt{\gamma}}\mathbb E\left[\sup_{n_0(\gamma)\le k\le N}\left\|\gamma\sum_{j=m+1}^{k}(I_d-\gamma A)^{k-j}R_j\right\|_2\right]\\
&\le\frac{C_0}{\varepsilon\sqrt{\gamma}}\gamma\sum_{j=m+1}^{N}\mathbb E\|\Delta_{j-1}\|_2^2\\
&\le\frac{C_0}{\varepsilon\sqrt{\gamma}}\gamma\sum_{j=m+1}^{N}\left(\rho_2^2(1-\alpha_0\gamma)^m+\frac{3M_2^2}{\alpha_0}\gamma\right)\\
&\le\frac{C_0}{\varepsilon\sqrt{\gamma}}N\gamma\left(\rho_2^2e^{-\alpha_0\gamma m}+\frac{3M_2^2}{\alpha_0}\gamma\right)\\
&\le\frac{C_0}{\varepsilon\sqrt{\gamma}}N\gamma\left(\rho_2^2\gamma^{\alpha_0\nu/2}+\frac{3M_2^2}{\alpha_0}\gamma\right)\\
&\le\frac{C_0}{\varepsilon}\bigl(\nu\log(1/\gamma)+T+2\gamma\bigr)\left(\rho_2^2\gamma^{\alpha_0\nu/2-1/2}+\frac{3M_2^2}{\alpha_0}\gamma^{1/2}\right)\\
&\underset{\gamma\to0}{\longrightarrow}0.
\end{align*}

Hence
$$
\sup_{t\in[0,T]}\frac{1}{\sqrt{\gamma}}\|I_{1,k(t)}\|_2\xrightarrow{\mathbb P}0.
$$

\begin{lemma}
Let $W_k = -\frac{1}{\sqrt{\gamma}} L_k$. Under the assumptions of Theorem \ref{thm:trajectory_fclt}, $W_{k(\cdot)}$ converges weakly in $D([0, T], \mathbb{R}^d)$ to the stationary Ornstein-Uhlenbeck process $X(\cdot)$ satisfying$$dX_t = -A X_t dt + S^{1/2} dB_t, \quad X_0 \sim \mathcal{N}(0, \Gamma),$$where $A\Gamma + \Gamma A = S$.\end{lemma}
\begin{proof}
\begin{align*}
W_k &= -\frac{1}{\sqrt{\gamma}} \left( (I_d - \gamma A) L_{k-1} + \gamma Z_k \right) \\
&= (I_d - \gamma A) \left(-\frac{1}{\sqrt{\gamma}} L_{k-1} \right) - \sqrt{\gamma} Z_k \\
&= W_{k-1} - \gamma A W_{k-1} - \sqrt{\gamma} Z_k.
\end{align*}
So
\begin{align*}
W_{k(t)} - W_{n_0(\gamma)} &= \sum_{j=n_0(\gamma)+1}^{k(t)} (W_j - W_{j-1}) \\
&= -A \gamma \sum_{j=n_0(\gamma)+1}^{k(t)} W_{j-1} - \sqrt{\gamma} \sum_{j=n_0(\gamma)+1}^{k(t)} Z_j.
\end{align*}
and writing the sum as an integral to make the sde structure more apparent
\begin{align*}
 &\gamma \sum_{j=n_0(\gamma)+1}^{k(t)} W_{j-1} = \sum_{l=0}^{\lfloor t/\gamma \rfloor - 1} \gamma W_{n_0(\gamma) + l} \\
&= \sum_{l=0}^{\lfloor t/\gamma \rfloor - 1} \int_{l\gamma}^{(l+1)\gamma} W_{k(s)} ds \text{, since $W_{k(s)}$ is constant on $[l\gamma, (l+1)\gamma)$,}\\
&= \int_0^{\lfloor t/\gamma \rfloor \gamma} W_{k(s)} ds.\end{align*}
Therefore, plugging this into our above bound:
\begin{align*}
W_{k(t)} &= W_{n_0(\gamma)} - A \int_0^{\lfloor t/\gamma \rfloor \gamma} W_{k(s)} ds - \sqrt{\gamma} \sum_{j=n_0(\gamma)+1}^{k(t)} Z_j \\
&= W_{n_0(\gamma)} - A \int_0^t W_{k(s)} ds + A \int_{\lfloor t/\gamma \rfloor \gamma}^t W_{k(s)} ds - \sqrt{\gamma} \sum_{j=n_0(\gamma)+1}^{k(t)} Z_j.
\end{align*}

We now analyze each term in the above integral equation. The $Z_j$ are i.i.d. with mean zero and covariance $S$. Thus, by multivariate Donsker's theorem(see Theorem 4.3.5 in \cite{whitt2002stochastic}) with number of steps $\lfloor \frac{1}{\gamma} \rfloor$, 
$-\sqrt{\gamma} \sum_{j=n_0(\gamma)+1}^{k(\cdot)} Z_j \Rightarrow S^{1/2} B(\cdot).
$  in $D([0,T], \mathbb{R}^d)$.

For the initial condition, Theorem 3.4 of \cite{NEURIPS2025_7aae9e3e} proves that $-\frac{1}{\sqrt{\gamma}} L_{k(0)} = -\frac{1}{\sqrt{\gamma}} L_{n_0(\gamma)} \Rightarrow \mathcal{N}(0, \Gamma)$. Because the increments $Z_j$ for $j > n_0(\gamma)$ are independent of the initial state $W_{n_0(\gamma)}$, we have \begin{align*}\left( W_{n_0(\gamma)}, \ -\sqrt{\gamma} \sum_{j=n_0(\gamma)+1}^{k(\cdot)} Z_j \right) \Rightarrow \left( X_0, \ S^{1/2} B(\cdot) \right),\end{align*}where $X_0 \sim \mathcal{N}(0, \Gamma)$.

The integral remainder is asymptotically negligible since
\begin{align*}
\left|
A\int_{\lfloor t/\gamma\rfloor\gamma}^{t} W_{k(s)}\,ds
\right|_2
&\le
\lambda^*
\left(t-\lfloor t/\gamma\rfloor\gamma\right)
\sup_{s\in[0,T]} |W_{k(s)}|_2 \\
&\le
\gamma \lambda^*
\sup_{s\in[0,T]} |W_{k(s)}|_2 \\
&\le
\gamma \lambda^* \exp(\lambda^* T)
\left(
|W_{n_0(\gamma)}|_2
+
\sup_{s\in[0,T]}
\left|
\sqrt{\gamma}\sum_{j=n_0(\gamma)+1}^{k(s)} Z_j
\right|_2
\right)
\end{align*}

since we have the following uniform bound on the $W_k$
\begin{align*}
|W_k|_2
&=
\left|
W_{n_0(\gamma)}
- A\gamma \sum_{j=n_0(\gamma)+1}^{k} W_{j-1}
- \sqrt{\gamma}\sum_{j=n_0(\gamma)+1}^{k} Z_j
\right|_2 \\
&\le
\left(
|W_{n_0(\gamma)}|_2
+
\max_{n_0(\gamma)\le k\le N}
\left|
\sqrt{\gamma}\sum_{j=n_0(\gamma)+1}^{k} Z_j
\right|_2
\right)
+
\gamma |A|_2 \sum_{i=n_0(\gamma)}^{k-1} |W_i|_2 \\
&\le
\left(
|W_{n_0(\gamma)}|_2
+
\max_{n_0(\gamma)\le k\le m}
\left|
\sqrt{\gamma}\sum_{j=n_0(\gamma)+1}^{k} Z_j
\right|_2
\right)
\exp\!\left(\gamma \lambda^* (N-n_0(\gamma))\right).
\end{align*}
,by the discrete version of Gronwall's inequality(\cite{Gronwal_proof}) which we can apply as the first term is increasing in m and everything is positive,

$$\leq \left(
|W_{n_0(\gamma)}|_2
+
\max_{n_0(\gamma)\le k\le m}
\left|
\sqrt{\gamma}\sum_{j=n_0(\gamma)+1}^{k} Z_j
\right|_2
\right)
\exp\!\left(\lambda^* T \right).$$

so
$$
\sup_{t\in[0,T]} |W_{k(t)}|_2 \leq \left( |W_{n_0(\gamma)}|_2 + \sup_{t\in[0,T]} \left| \sqrt{\gamma}\sum_{j=n_0(\gamma)+1}^{k(t)} Z_j \right|_2 \right)\exp(\lambda^* T).
$$

Hence,
\begin{align*}
&\mathbb{P}\!\left(
\sup_{t\in[0,T]}
\left|
A\int_{\lfloor t/\gamma\rfloor\gamma}^{t} W_{k(s)}\,ds
\right|_2
>
\epsilon
\right) \\
&\le
\mathbb{P}\!\left(
|W_{n_0(\gamma)}|_2
+
\sup_{s\in[0,T]}
\left|
\sqrt{\gamma}\sum_{j=n_0(\gamma)+1}^{k(s)} Z_j
\right|_2
>
\frac{\epsilon}{\gamma \lambda^* \exp(\lambda^* T)}
\right).
\end{align*}

But we have already proven both $W_{n_0(\gamma)}, \ -\sqrt{\gamma} \sum_{j=n_0(\gamma)+1}^{k(\cdot)} Z_j $ converge and are thus tight with the second term being tight directly by Donsker's theorem and the first by Prokhorov's Theorem (Theorems 5.1 and 5.2, \cite{billingsley1999convergence}). Therefore, sending $\gamma \to 0$ sends $\frac{\epsilon}{\gamma \lambda^* \exp(\lambda^* T)}$ to $\infty$ and thus the probability to 0.

Now we define
$$
\Phi:\mathbb{R}^d\times D([0,T],\mathbb{R}^d)
\to
D([0,T],\mathbb{R}^d)
$$
which maps any initial state $x_0$ and path $w(\cdot)$ to the strong solution $x(\cdot)$ of the ODE:
$$
x(t)    = x_0 - \int_0^t A x(s)\,ds + w(t).
$$

If $w_n\to w$ in the Skorokhod $J_1$ topology and $w$
is continuous, then the convergence is uniform.
The following bound establishes continuity of
$\Phi$ at inputs with continuous $w$. So for inputs $\Phi(x_{0,1}, w_1(\cdot)):= x_1(t)$ and $\Phi(x_{0,2}, w_2(\cdot)):=  x_2(t)$ 
\begin{align*}
|x_1(t)-x_2(t)|_2
&=
\left|x_{0,1}-x_{0,2}+w_1(t)-w_2(t)-\int_0^t A(x_1(s)-x_2(s))\,ds\right|_2 \\
&\le \left(|x_{0,1}-x_{0,2}|_2+\sup_{s\in[0,T]}|w_1(s)-w_2(s)|_2\right)+\int_0^t \lambda^* |x_1(s)-x_2(s)|_2\,ds.
\end{align*}

which by the continuous Gronwall's inequality leads to the uniform bound
\begin{align*}
\sup_{t\in[0,T]} |x_1(t)-x_2(t)|_2
&\le \exp(\lambda^* T)\left(|x_{0,1}-x_{0,2}|_2+\sup_{s\in[0,T]}|w_1(s)-w_2(s)|_2\right).
\end{align*}
This bound implies continuity of $\Phi$ in the Skorokhod topology at every input with continuous $w$.
So letting $r_\gamma(t)=A\int_{\lfloor t/\gamma\rfloor\gamma}^{t}W_{k(s)}\,ds$ and $w_\gamma(t)=-\sqrt\gamma\sum_{j=n_0(\gamma)+1}^{k(t)}Z_j$ and $W_{k(\cdot)}=\Phi(W_{n_0(\gamma)},w_\gamma+r_\gamma)$, then the preceding bound gives
\[
\|W_{k(\cdot)}-\Phi(W_{n_0(\gamma)},w_\gamma)\|_\infty
\le e^{\lambda^*T}\|r_\gamma\|_\infty=o_{\mathbb P}(1).
\] Thus by Slutsky's theorem and the continuous mapping theorem(Theorem 3.1 in \cite{billingsley1999convergence}), having just shown $\Phi$ is continuous, we have that $$\lim_{\gamma \to 0} W_{k(t)} = \lim_{\gamma \to 0} \Phi \left(W_{n_0(\gamma)},-\sqrt{\gamma}\sum_{j=n_0(\gamma)+1}^{k(\cdot)}Z_j\right)(t) \Rightarrow \Phi \left(X_0,S^{1/2}B(\cdot)\right)$$ and $\Phi \left(X_0,S^{1/2}B(\cdot)\right)$ satisfies $dX_t=-AX_t\,dt+S^{1/2}\,dB_t$ and $A\Gamma+\Gamma A=S$
\end{proof}

Therefore, by one final application of Slutsky's theorem since all the remainder terms go to zero uniformly in probability we have our final FCLT.

\end{proof}

\section{KMT variants} 
\begin{lemma}[Gaussian approximation of the noise partial sums]
\label{lem:innovation-new}
Suppose $Z_j$ are i.i.d. and centered, with covariance $S$
and $M_p=\|Z_1\|_p<\infty$ for some $p>2$.
On an extension of the probability space, there are
i.i.d. $Z_j^\diamond\sim N(0,S)$ such that
\begin{equation}\label{eq:innovation-new}
\mathbb E\max_{k\le N}
\left\|\sum_{j=1}^k(Z_j-Z_j^\diamond)\right\|_2^p
\le K_{p,d,S}NM_p^p,    N\ge1,
\end{equation}
where $K_{p,d,S}$ is independent of $N$.
\end{lemma}
\begin{proof}
Apply Theorem 4 of \citet{GotzeZaitsev2010},
restricted to the range of $S$ when $S$ is singular.
The case $S=0$ is immediate.
\end{proof}

\begin{lemma}[Maximal Convolution Bound] \label{lem:max_convolution_bound}
Let $X_1, \dots, X_n \in \mathbb{R}^d$ be a sequence of random vectors, and $S_m = \sum_{j=1}^m X_j$ with $S_0 = 0$. Let $A \in \mathbb{R}^{d \times d}$ be a symmetric positive-definite matrix with eigenvalues bounded in $[\lambda_*, \lambda^*]$ where $0 < \lambda_* \le \lambda^*$. For any step size $\gamma \in (0, 1/\lambda^*)$, define:
$$Y_k = \gamma \sum_{j=1}^k (I_d - \gamma A)^{k-j} X_j$$
Then,
$$\max_{1 \le m \le n} \left\| \sum_{k=1}^m Y_k \right\|_2 \le \frac{1}{\lambda_*} \max_{1 \le m \le n} \|S_m\|_2$$
\end{lemma}

\begin{proof}
Using the same fubini swap and summation by parts,
$$\begin{aligned}
\sum_{k=1}^m Y_k
&=\gamma\sum_{j=1}^m(I_d-\gamma A)^{m-j}S_j,\\
\left\|\sum_{k=1}^m Y_k\right\|_2
&\le\gamma\sum_{j=1}^m(1-\gamma\lambda_*)^{m-j}
\max_{1\le j\le m}\|S_j\|_2
\le\frac1{\lambda_*}\max_{1\le j\le m}\|S_j\|_2.
\end{aligned}$$
Taking the maximum over $m\le n$ proves the result.
\end{proof}

\begin{theorem} \label{KMT Burn in}
Let $\Delta_{n_0(\gamma)+k}^\diamond=(I_d-\gamma A)^{n_0(\gamma)+k}\Delta_0-L_{n_0(\gamma)+k}^\diamond$.
Write $b=n_0(\gamma)$, $q_A=1-\lambda_*\gamma$, and, for each required order $r\ge2$,
\[
v_{r,j}=\rho_r^2(1-\alpha_0\gamma)^j+
\frac{3(r-1)M_r^2}{\alpha_0}\gamma.
\]
For $p\ge2$, $p'\ge1$, assume moments of orders $p$ and $2p'$ and the same step size restrictions as defined in the main body. Define
\begin{align*}
K_L&:=K_{p,d,S}
\mathbb E\|Z_1\|_2^p
\left(\frac{6(2+\lambda^*/\lambda_*)}{\lambda_*}\right)^p,\\
\mathcal B_L&:=K_L(b+n),\\
\mathcal B_M&:=3^p\Bigg[
\frac{2L_p\sqrt{p-1}}{\lambda_*}
\left(\sum_{j=1}^{b}q_A^{2(b-j)}v_{p,j-1}\right)^{1/2}\\
&\hspace{15mm}+
\frac{2pL_p(2+\lambda^*/\lambda_*)}{\lambda_*\sqrt{p-1}}
\left(\sum_{j=1}^{n}v_{p,b+j-1}\right)^{1/2}
\Bigg]^p,\\
\mathcal B_T&:=\left(\frac{3C_0}{\lambda_*}\right)^{p'}
\left[\sum_{j=1}^{b}q_A^{b-j}v_{2p',j-1}
+2\sum_{j=1}^{n}v_{2p',b+j-1}\right]^{p'}.
\end{align*}
The Gaussian sequence is constructed as in Theorem \ref{kmt for the linear component}. Then, for every $y>0$,
\[
\mathbb P\left[\max_{1\le m\le n}\left\|
\sum_{k=1}^m(\Delta_{b+k}-\Delta_{b+k}^\diamond)
\right\|_2\ge y\right]
\le\frac{\mathcal B_L+\mathcal B_M}{y^p}
+\frac{\mathcal B_T}{y^{p'}}.
\]
\end{theorem}

\begin{proof}

\begin{align*} 
& \sum_{k=1}^m (\Delta_{n_0(\gamma)+k} - \Delta_{n_0(\gamma)+k}^\diamond) = - \sum_{k=1}^m (L_{n_0(\gamma)+k} - L_{n_0(\gamma)+k}^\diamond) + \sum_{k=1}^m I_{1,n_0(\gamma)+k} + \sum_{k=1}^m I_{2,n_0(\gamma)+k} \\ 
& \implies \mathbb{P}\left[ \max_{1 \le m \le n} \left\Vert{} \sum_{k=1}^m (\Delta_{n_0(\gamma)+k} - \Delta_{n_0(\gamma)+k}^\diamond) \right\Vert{}_2 \ge y \right] \le \mathbb{P}\left[ \max_{1 \le m \le n} \left\Vert{} \sum_{k=1}^m (L_{n_0(\gamma)+k} - L_{n_0(\gamma)+k}^\diamond) \right\Vert{}_2 \ge \frac{y}{3} \right] \\ & \\
&\quad+ \mathbb{P}\left[ \max_{1 \le m \le n} \left\Vert{} \sum_{k=1}^m I_{1,n_0(\gamma)+k} \right\Vert{}_2 \ge \frac{y}{3} \right] \\
&\quad+ \mathbb{P}\left[ \max_{1 \le m \le n} \left\Vert{} \sum_{k=1}^m I_{2,n_0(\gamma)+k} \right\Vert{}_2 \ge \frac{y}{3} \right]
\end{align*}

We start with a $p$-th moment bound for $I_2$. 
\[
I_{2,b+k}=(I_d-\gamma A)^kI_{2,b}
+\gamma\sum_{j=1}^k(I_d-\gamma A)^{k-j}D_{b+j}.
\]
By Minkowski's inequality and Lemma \ref{lem:max_convolution_bound},
\begin{align*}
\left\|\max_{m\le n}\left\|\sum_{k=1}^m I_{2,b+k}\right\|_2\right\|_p
&\le\frac{\|I_{2,b}\|_p}{\gamma\lambda_*}
+\frac{2+\lambda^*/\lambda_*}{\lambda_*}
\left\|\max_{m\le n}\left\|\sum_{j=1}^mD_{b+j}\right\|_2\right\|_p.
\end{align*}
Rio's inequality for the first term, and Doob followed by Rio for the second, give
\begin{align*}
\|I_{2,b}\|_p
&\le 2L_p\sqrt{p-1}\gamma
\left(\sum_{j=1}^{b}q_A^{2(b-j)}v_{p,j-1}\right)^{1/2},\\
\left\|\max_{m\le n}\left\|\sum_{j=1}^mD_{b+j}\right\|_2\right\|_p
&\le\frac{2pL_p}{\sqrt{p-1}}
\left(\sum_{j=1}^{n}v_{p,b+j-1}\right)^{1/2}.
\end{align*}
Hence Markov's inequality gives the contribution $\mathcal B_M/y^p$.

Now for the Taylor error term,  $\|R_j\|_{p'}\le C_0\|\Delta_{j-1}\|_{2p'}^2$. Minkowski's inequality yields
\begin{align*}
\|I_{1,b}\|_{p'}&\le C_0\gamma\sum_{j=1}^{b}q_A^{b-j}v_{2p',j-1},\\
\left\|\max_{m\le n}\left\|\sum_{k=1}^m I_{1,b+k}\right\|_2\right\|_{p'}
&\le\frac{\|I_{1,b}\|_{p'}}{\gamma\lambda_*}
+\frac{2}{\lambda_*}\sum_{j=1}^n\|R_{b+j}\|_{p'}\\
&\le\frac{C_0}{\lambda_*}
\left[\sum_{j=1}^{b}q_A^{b-j}v_{2p',j-1}
+2\sum_{j=1}^{n}v_{2p',b+j-1}\right].
\end{align*}
Markov's inequality gives $\mathcal B_T/y^{p'}$. This retains the factor $\gamma^{-p'}$ when taking the $p'$-th moment of the Burn in memory.

Now,
\begin{align*} 
& \sum_{k=1}^m (L_{n_0(\gamma)+k} - L_{n_0(\gamma)+k}^\diamond) = \sum_{i=1}^{n_0(\gamma)+m} (L_i - L_i^\diamond) - \sum_{i=1}^{n_0(\gamma)} (L_i - L_i^\diamond)  \\ 
& \implies \max_{1 \le m \le n} \left\Vert \sum_{k=1}^m (L_{n_0(\gamma)+k} - L_{n_0(\gamma)+k}^\diamond) \right\Vert_2 \le 2 \max_{1 \le i \le n_0(\gamma)+n} \left\Vert \sum_{k=1}^i (L_k - L_k^\diamond) \right\Vert_2
\end{align*}
so
\begin{align*} 
& \mathbb{P}\left[ \max_{1 \le m \le n} \left\Vert \sum_{k=1}^m (L_{n_0(\gamma)+k} - L_{n_0(\gamma)+k}^\diamond) \right\Vert_2 \ge \frac{y}{3} \right] \le \mathbb{P}\left[ \max_{1 \le i \le n_0(\gamma)+n} \left\Vert \sum_{k=1}^i (L_k - L_k^\diamond) \right\Vert_2 \ge \frac{y}{6} \right] \\ 
& \le K_{p,d,S}\, (n_0(\gamma)+n)\,\mathbb{E}\Vert Z_1\Vert_2^p \left(\frac{6(2+\lambda^*/\lambda_*)}{\lambda_* y}\right)^p  \text{, by lemma \ref{kmt for the linear component}}
\end{align*}
\end{proof}

\begin{theorem} \label{KMT_Romberg}
 Let $\hat{\Delta}_k$ be the error of the Richardson Romberg SGD iterates from theorem \ref{romberg fclt}($\hat{\theta}_k = 2\theta_k(\gamma) - \theta_k(2\gamma)$, $\Delta_k(x) = \theta_k(x) - \theta^*$) $\hat{\Delta}_k^\diamond = \big( 2(I_d - \gamma A)^k - (I_d - 2\gamma A)^k \big) \Delta_0 - \big( 2L_k^\diamond(\gamma) - L_k^\diamond(2\gamma) \big).$ We assume the gradient noise $\nabla g(\theta^*, X)$ possesses bounded moments of orders $p_1\ge2$, $p_2>2$, and $2p_3>4$. Let $C_{BDG, p}$ be the constant from the Burkholder-Davis-Gundy inequality. The same Gaussian noise $Z_j^\diamond$ is used for both SGD runs.

Now define

\[
C_{4,2p_3}
=
2\sqrt2\left\{
M_{2p_3}^2+
L_{2p_3}^2
\left(
\rho_{2p_3}^2+
\frac{6(2p_3-1)M_{2p_3}^2}{\alpha_0}\bar\gamma
\right)
\right\}^{1/2},
\]
and define
\[
C_{M,p_2}
=
\left(
\frac{3(p_2-1)M_{p_2}^2}{\alpha_0}
\right)^{p_2/2}.
\]

Now for $\eta\in\{\gamma,2\gamma\}$ define
\[
\begin{aligned}
B_{Fub}(\eta)
&:=
\frac{\|A^{-1}\|_2C_0\rho_2^2}{e\lambda\eta}
+\frac{3\|A^{-1}\|_2C_0M_2^2}{\lambda\alpha_0},
\\
B_{Rem}(\eta)
&:=
\|A^{-1}\|_2\Bigg[
\frac{L_4\rho_2\sqrt{C_S}}{\alpha_0}
+n\left(
\frac{L_4\sqrt{3C_S}M_2}{2\sqrt{\alpha_0}}
+\frac{L_4C_S^{3/4}}6
\right)\eta^{3/2}
\\
&\quad+
\frac{\|\nabla^3G(\theta^*)\|_2}{2}
\left(
\frac{2C_{\phi1}}{\lambda^2}
+\frac{32C_{R3}}{\lambda^3}
+2n(C_{\phi2}+C_{R4})\eta^2
+\frac{2^{5/2}C_{X1}}{\lambda^{3/2}}
+\frac{8C_{X2}}{\lambda^2}
+2nC_{X3}\eta^{3/2}
\right)
\\
&\quad+
\frac{C_LC_E\rho_2}{\alpha_0}
+\frac{C_LC_EM_2\sqrt3}{2\sqrt{\alpha_0}}n\eta^{3/2}
+\frac{C_LC_E\|\Delta_0\|_2}{\mu}
\Bigg],
\\
B_{Det}(\eta)
&:=
\frac{\|A^{-1}\|_2C_L\|\Delta_0\|_2^2}{2\mu\eta},
\\
B_{Cov}
&:=
\frac{3\|A^{-1}\|_2
\|\nabla^3G(\theta^*)\|_2\|\Gamma\|_2}{2\lambda_*}
+\frac{3n\|A^{-1}\|_2
\|\nabla^3G(\theta^*)\|_2\|A\Gamma A\|_2}
{\lambda_*}\gamma^2.
\end{aligned}
\]
Define
\[
y_0
:=
2\big(B_{Fub}(\gamma)+B_{Rem}(\gamma)+B_{Det}(\gamma)\big)
+B_{Fub}(2\gamma)+B_{Rem}(2\gamma)+B_{Det}(2\gamma)
+B_{Cov},
\]
and further define
\[
\begin{aligned}
C_{Lin}
&:=
K_{p_1,d,S}
\mathbb E\|Z_1\|_2^{p_1}
\left(
\frac{2+\lambda^*/\lambda_*}{\lambda_*}
\right)^{p_1}
(12^{p_1}+6^{p_1}),
\\
C_{Stoch}(n,\gamma)
&:=n^{1-p_2/2}
\left(\frac{2p_2L_{p_2}}{\lambda_*\sqrt{p_2-1}}\right)^{p_2}
\Bigg[
12^{p_2}\left(
\frac{\rho_{p_2}^2}{\alpha_0\gamma}
+\frac{3(p_2-1)M_{p_2}^2}{\alpha_0}n\gamma
\right)^{p_2/2}\\
& +
6^{p_2}\left(
\frac{\rho_{p_2}^2}{2\alpha_0\gamma}
+\frac{6(p_2-1)M_{p_2}^2}{\alpha_0}n\gamma
\right)^{p_2/2}
\Bigg],
\\
C_{Rem}(n,\gamma)
&:=
n^{1-p_3/2}
\left(
\frac{2+\lambda^*/\lambda_*}{\lambda_*}
\right)^{p_3}
\left(
\frac{2p_3C_LC_{4,2p_3}}{\sqrt{p_3-1}}
\right)^{p_3}
\\
&\quad\cdot\Bigg[
12^{p_3}
\left(
\frac{\rho_{2p_3}}{\alpha_0^{3/2}\sqrt\gamma}
+\frac{2\sqrt{3(2p_3-1)}M_{2p_3}}
{\alpha_0^{3/2}}\sqrt{n\gamma}
\right)^{p_3}
\\
& +
6^{p_3}
\left(
\frac{\rho_{2p_3}}{\alpha_0^{3/2}\sqrt{2\gamma}}
+\frac{2\sqrt{3(2p_3-1)}M_{2p_3}}
{\alpha_0^{3/2}}\sqrt{2n\gamma}
\right)^{p_3}
\Bigg].
\end{aligned}
\]

Then, for any $y>y_0$,
\[
\begin{aligned}
&\mathbb P\left[
\max_{1\le m\le n}
\left\|
\sum_{k=1}^m
(\hat{\Delta}_k-\hat{\Delta}_k^\diamond)
\right\|_2\ge y
\right]
\\
& \le
\frac{C_{Lin}\,n}{(y-y_0)^{p_1}}
+
\frac{C_{Stoch}(n,\gamma)\,n^{p_2/2-1}}
{(y-y_0)^{p_2}}
+
\frac{C_{Rem}(n,\gamma)\,n^{p_3/2-1}}
{(y-y_0)^{p_3}}.
\end{aligned}
\]
\end{theorem}

\begin{proof}
We start by recalling the bounds for the combined mean term $\max_{m\le n}\|\sum_{k=1}^m(2\mathbb E I_{1,k}(\gamma)-\mathbb E I_{1,k}(2\gamma))\|_2$ from theorem \ref{romberg fclt}. $y_0$ is our bound from that theorem where we recall the fubini swap and that 
$$\begin{aligned}
& \max_{1 \le m \le n} \left\| A^{-1} \sum_{j=1}^m (I_d - \eta A)^{m-j+1} \mathbb{E}[R_j(\eta)] \right\|_2 \\
& \le \|A^{-1}\|_2 C_0 \rho_2^2 \max_{1 \le m \le n} m (1-\lambda\eta)^m + \frac{3\|A^{-1}\|_2 C_0 M_2^2}{\lambda \alpha_0}  \\
& \le \frac{\|A^{-1}\|_2 C_0 \rho_2^2}{e \lambda \eta} + \frac{3\|A^{-1}\|_2 C_0 M_2^2}{\lambda \alpha_0} =: B_{Fub}(\eta). 
\end{aligned}$$ The other constants $B_{Rem}(\eta),B_{Det}(\eta)$, and $B_{Cov}(\eta)$ are taken directly as the bounds from theorem \ref{romberg fclt} and summing them bounds the combined mean term, yielding
$$y_0 := 2\Big(B_{Fub}(\gamma) + B_{Rem}(\gamma) + B_{Det}(\gamma)\Big) + \Big(B_{Fub}(2\gamma) + B_{Rem}(2\gamma) + B_{Det}(2\gamma)\Big) + B_{Cov}.$$
Now,
$$\begin{aligned}
& \sum_{k=1}^m (\hat{\Delta}_k - \hat{\Delta}_k^\diamond) = - \sum_{k=1}^m \big( 2(L_k(\gamma) - L_k^\diamond(\gamma)) - (L_k(2\gamma) - L_k^\diamond(2\gamma)) \big) + \sum_{k=1}^m \big( 2\tilde{I}_{1,k}(\gamma) - \tilde{I}_{1,k}(2\gamma) \big) \\
& + \sum_{k=1}^m \big( 2I_{2,k}(\gamma) - I_{2,k}(2\gamma) \big) + \sum_{k=1}^m \big( 2\mathbb{E}[I_{1,k}(\gamma)] - \mathbb{E}[I_{1,k}(2\gamma)] \big) \\ 
& \implies \mathbb{P}\left[ \max_{1 \le m \le n} \left\Vert{} \sum_{k=1}^m (\hat{\Delta}_k - \hat{\Delta}_k^\diamond) \right\Vert{}_2 \ge y \right] \le \\
&\mathbb{P}\left[ \max_{1 \le m \le n} \left\Vert{} \sum_{k=1}^m \big( 2(L_k(\gamma) - L_k^\diamond(\gamma)) - (L_k(2\gamma) - L_k^\diamond(2\gamma)) \big) \right\Vert{}_2 \ge \frac{y - y_0}{3} \right] \\ & \\
&\quad+ \mathbb{P}\left[ \max_{1 \le m \le n} \left\Vert{} \sum_{k=1}^m \big( 2\tilde{I}_{1,k}(\gamma) - \tilde{I}_{1,k}(2\gamma) \big) \right\Vert{}_2 \ge \frac{y - y_0}{3} \right] \\
&\quad+ \mathbb{P}\left[ \max_{1 \le m \le n} \left\Vert{} \sum_{k=1}^m \big( 2I_{2,k}(\gamma) - I_{2,k}(2\gamma) \big) \right\Vert{}_2 \ge \frac{y - y_0}{3} \right] \\
&\text{, factoring out the } y_0 \text{ offset and applying the union bound} \end{aligned}$$

\bigskip

Now for $\eta\in\{\gamma,2\gamma\}$, $D_j(\eta)$ is a martingale difference. As in lemma \ref{Sum stochastic error to 0}, Doob's maximal $\mathcal L^{p_2}$ inequality gives
\begin{align*}
\left\|\max_{m\le n}\left\|\sum_{j=1}^mD_j(\eta)\right\|_2\right\|_{p_2}
&\le\frac{p_2}{p_2-1}
\left\|\sum_{j=1}^nD_j(\eta)\right\|_{p_2}.
\end{align*}
Then, by Rio's inequality and the same moment bound,
\begin{align*}
\left\|\sum_{j=1}^nD_j(\eta)\right\|_{p_2}
&\le\sqrt{p_2-1}
\left(\sum_{j=1}^n\|D_j(\eta)\|_{p_2}^2\right)^{1/2}\\
&\le2L_{p_2}\sqrt{p_2-1}
\left[\sum_{j=0}^{n-1}\left\{
\rho_{p_2}^2(1-\alpha_0\eta)^j
+\frac{3(p_2-1)M_{p_2}^2}{\alpha_0}\eta
\right\}\right]^{1/2}\\
&\le2L_{p_2}\sqrt{p_2-1}
\left(\frac{\rho_{p_2}^2}{\alpha_0\eta}
+\frac{3(p_2-1)M_{p_2}^2}{\alpha_0}n\eta\right)^{1/2}.
\end{align*}
By lemma \ref{lem:max_convolution_bound} and Markov's inequality, for $a>0$,
\begin{align*}
\mathbb P\left[\max_{m\le n}\left\|\sum_{k=1}^mI_{2,k}(\eta)\right\|_2\ge a\right]
&\le\frac1{a^{p_2}}
\left(\frac{2p_2L_{p_2}}{\lambda_*\sqrt{p_2-1}}\right)^{p_2}
\left(\frac{\rho_{p_2}^2}{\alpha_0\eta}
+\frac{3(p_2-1)M_{p_2}^2}{\alpha_0}n\eta\right)^{p_2/2}.
\end{align*}
Consequently,
\begin{align*}
&\mathbb P\left[\max_{m\le n}\left\|\sum_{k=1}^m
\big(2I_{2,k}(\gamma)-I_{2,k}(2\gamma)\big)\right\|_2
\ge\frac{y-y_0}{3}\right]\\
&\le\mathbb P\left[\max_{m\le n}\left\|\sum_{k=1}^mI_{2,k}(\gamma)\right\|_2
\ge\frac{y-y_0}{12}\right]
+\mathbb P\left[\max_{m\le n}\left\|\sum_{k=1}^mI_{2,k}(2\gamma)\right\|_2
\ge\frac{y-y_0}{6}\right]\\
&\le\frac{C_{Stoch}(n,\gamma)n^{p_2/2-1}}{(y-y_0)^{p_2}}.
\end{align*}
Now following the logic of Lemma
\ref{Sum centered taylor error to 0}, for each fixed
$\ell\ge1$, $\mathcal P_t(R_{t+\ell})$ is a martingale
difference sequence in $t$. Thus,
\begin{align*}
&\mathbb P\left[
 \max_{1\le m\le n}
 \left\|\sum_{k=1}^m\tilde I_{1,k}(\gamma)\right\|_2
 \ge \frac{y-y_0}{12}\right]\\
&\le \mathbb P\left[
 \max_{1\le m\le n}
 \left\|\sum_{j=1}^m\tilde R_j(\gamma)\right\|_2
 \ge \frac{\lambda_*(y-y_0)}
 {12(2+\lambda^*/\lambda_*)}\right]
 \text{, by Lemma \ref{lem:max_convolution_bound}}\\
&\le
 \frac{12^{p_3}(2+\lambda^*/\lambda_*)^{p_3}}
 {\lambda_*^{p_3}(y-y_0)^{p_3}}
 \mathbb E\max_{m\le n}
 \left\|
 \sum_{\ell=1}^{n-1}\sum_{t=1}^{(m-\ell)_+}
 \mathcal P_t(R_{t+\ell})
 \right\|_2^{p_3}
 \text{, since $\mathcal P_t(R_t)=0$}\\
&\le
 \frac{12^{p_3}(2+\lambda^*/\lambda_*)^{p_3}}
 {\lambda_*^{p_3}(y-y_0)^{p_3}}
 \left(\frac{p_3}{p_3-1}\right)^{p_3}
 \left(
 \sum_{\ell=1}^{n-1}
 \left\|
 \sum_{t=1}^{n-\ell}\mathcal P_t(R_{t+\ell})
 \right\|_{p_3}
 \right)^{p_3}\\
&\text{, by Minkowski and Doob's maximal
 $\mathcal L^{p_3}$ inequality at each fixed $\ell$,
 as in Lemma \ref{Sum centered taylor error to 0},}\\
&\le
 \frac{12^{p_3}(2+\lambda^*/\lambda_*)^{p_3}}
 {\lambda_*^{p_3}(y-y_0)^{p_3}}
 \left(\frac{p_3}{\sqrt{p_3-1}}\right)^{p_3}
 \left(
 \sum_{\ell=1}^{n-1}
 \left(
 \sum_{t=1}^{n-\ell}
 \|\mathcal P_t(R_{t+\ell})\|_{p_3}^2
 \right)^{1/2}
 \right)^{p_3}
 \text{, by Rio's inequality.}\\
&\le
 \frac{12^{p_3}(2+\lambda^*/\lambda_*)^{p_3}}
 {\lambda_*^{p_3}(y-y_0)^{p_3}}
 \left(
 \frac{2p_3C_LC_{4,2p_3}\gamma}
 {\sqrt{p_3-1}}
 \right)^{p_3}\\
& , \cdot
 \left(
 \sum_{\ell=1}^{n-1}
 \left\{
 \frac{\rho_{2p_3}(1-\alpha_0\gamma)^{\ell-1}}
 {\sqrt{\alpha_0\gamma}}
 +\sqrt{3(2p_3-1)}M_{2p_3}\alpha_0^{-1/2}
 \sqrt{n\gamma}
 (1-\alpha_0\gamma)^{(\ell-1)/2}
 \right\}
 \right)^{p_3}\\
&\le
 \frac{12^{p_3}(2+\lambda^*/\lambda_*)^{p_3}}
 {\lambda_*^{p_3}(y-y_0)^{p_3}}
 \left(
 \frac{2p_3C_LC_{4,2p_3}\gamma}
 {\sqrt{p_3-1}}
 \right)^{p_3}\\
& , \cdot
 \left(
 \frac{\rho_{2p_3}}
 {\sqrt{\alpha_0\gamma}}\frac1{\alpha_0\gamma}
 +\sqrt{3(2p_3-1)}M_{2p_3}\alpha_0^{-1/2}
 \sqrt{n\gamma}\frac2{\alpha_0\gamma}
 \right)^{p_3}\\
&\text{, since }
 1-(1-\alpha_0\gamma)^{1/2}
 \ge \tfrac12\alpha_0\gamma,\\
&=
 \frac{12^{p_3}(2+\lambda^*/\lambda_*)^{p_3}}
 {\lambda_*^{p_3}(y-y_0)^{p_3}}
 \left(
 \frac{2p_3C_LC_{4,2p_3}}
 {\sqrt{p_3-1}}
 \right)^{p_3}
 \left(
 \frac{\rho_{2p_3}}{\alpha_0^{3/2}\sqrt\gamma}
 +\frac{2\sqrt{3(2p_3-1)}M_{2p_3}}
 {\alpha_0^{3/2}}\sqrt{n\gamma}
 \right)^{p_3}.
\end{align*}
and similarly
\begin{align*}
&\mathbb P\left[
\max_{1\le m\le n}
\left\|\sum_{k=1}^m\tilde I_{1,k}(2\gamma)\right\|_2
\ge\frac{y-y_0}{6}
\right]\\
&\le
\frac{6^{p_3}(2+\lambda^*/\lambda_*)^{p_3}}
{\lambda_*^{p_3}(y-y_0)^{p_3}}
\left(\frac{2p_3C_LC_{4,2p_3}}{\sqrt{p_3-1}}\right)^{p_3}
\\
& , \cdot
\left(
\frac{\rho_{2p_3}}{\alpha_0^{3/2}\sqrt{2\gamma}}
+\frac{2\sqrt{3(2p_3-1)}M_{2p_3}}{\alpha_0^{3/2}}
\sqrt{2n\gamma}
\right)^{p_3}.
\end{align*}
which implies
\[
\mathbb P\left[
\max_{1\le m\le n}
\left\|
\sum_{k=1}^m
\big(2\tilde I_{1,k}(\gamma)-\tilde I_{1,k}(2\gamma)\big)
\right\|_2
\ge\frac{y-y_0}{3}
\right]
\le
\frac{C_{Rem}(n,\gamma)n^{p_3/2-1}}
{(y-y_0)^{p_3}}.
\]
where in the final step we combine the $\gamma \text{ and } 2\gamma$ terms into $C_{rem}(n, \gamma)$.
Finally,
$$\begin{aligned} 
& \mathbb{P}\left[ \max_{1 \le m \le n} \left\Vert{} \sum_{k=1}^m \big( 2(L_k(\gamma) - L_k^\diamond(\gamma)) - (L_k(2\gamma) - L_k^\diamond(2\gamma)) \big) \right\Vert{}_2 \ge \frac{y - y_0}{3} \right] \\ & \le \mathbb{P}\left[ \max_{1 \le m \le n} \left\Vert{} \sum_{k=1}^m (L_k(\gamma) - L_k^\diamond(\gamma)) \right\Vert{}_2 \ge \frac{y-y_0}{12} \right] \\
&\quad+ \mathbb{P}\left[ \max_{1 \le m \le n} \left\Vert{} \sum_{k=1}^m (L_k(2\gamma) - L_k^\diamond(2\gamma)) \right\Vert{}_2 \ge \frac{y-y_0}{6} \right] \\
&\text{, by triangle inequality} \\ 
& \le K_{p_1,d,S}\, n\,\mathbb{E}\Vert{}Z_1\Vert{}_2^{p_1} \left(\frac{12(2+\lambda^*/\lambda_*)}{\lambda_* (y-y_0)}\right)^{p_1} \\ 
& + K_{p_1,d,S}\, n\,\mathbb{E}\Vert{}Z_1\Vert{}_2^{p_1} \left(\frac{6(2+\lambda^*/\lambda_*)}{\lambda_* (y-y_0)}\right)^{p_1}   \text{, by lemma \ref{kmt for the linear component}} \\ 
& = \frac{C_{Lin} \cdot n}{(y-y_0)^{p_1}} 
\end{aligned}$$
where in the last line we combine everything into $C_{Lin}$
\end{proof}

\begin{theorem} \label{thm:bias_kmt}
Assume the gradient noise $\nabla g(\theta^*, X)$ possesses bounded moments for $p_1 \ge 2$, $p_2 > 2$, and $2p_3 > 4$. Let $C_{BDG, p}$ be the absolute constant from the Burkholder-Davis-Gundy inequality.

Let $b = -\frac{1}{2} \nabla^3 G(\theta^\star)[\Gamma]$, and let $x(s)$ be the solution to the gradient-flow ODE $\dot{x}(s) = -\nabla G(\theta^\star + x(s))$ with initial condition $x(0) = \Delta_0$. Define
$$ \Delta_k^\diamond := (I_d - \gamma A)^k \Delta_0 - L_k^\diamond + \gamma \sum_{j=1}^k (I_d - \gamma A)^{k-j} \Big( \gamma b - R_{det}(x((j-1)\gamma)) \Big) $$

and further define
$$\begin{aligned}
& B_{Rem} := \frac{L_4 \rho_2 \sqrt{C_S}}{\alpha_0} + n \left(\frac{L_4 \sqrt{3C_S} M_2}{2\alpha_0^{1/2}}+\frac{L_4C_S^{3/4}}6\right) \gamma^{1.5} \\
&+ \frac{\|\nabla^3 G(\theta^*)\|_2}{2} \left( \frac{2 C_{\phi 1}}{\lambda^2} + \frac{32 C_{R3}}{\lambda^3} + 2n(C_{\phi 2} + C_{R4}) \gamma^2 + \frac{2^{5/2} C_{X1}}{\lambda^{3/2}} + \frac{8 C_{X2}}{\lambda^2} + 2n C_{X3} \gamma^{1.5} \right) \\
&  + \frac{C_L C_E \rho_2}{\alpha_0} + \frac{C_L C_E M_2 \sqrt{3}}{2\sqrt{\alpha_0}} n \gamma^{1.5} + \frac{C_L C_E \|\Delta_0\|_2}{\mu} \\
& B_{Cov} := \frac{\|\nabla^3 G(\theta^*)\|_2 \|\Gamma\|_2}{2\lambda_*} + \frac{n \|\nabla^3 G(\theta^*)\|_2 \|A\Gamma A\|_2}{2\lambda_*} \gamma^2 
\end{aligned}$$
so that we may define $y_0 := \frac{2+\lambda^*/\lambda_*}{\lambda_*} \big( B_{Rem} + B_{Cov} \big)$ 
and
$$\begin{aligned}
& C_{Lin} := K_{p_1,d,S} \mathbb{E}\|Z_1\|_2^{p_1} \left(\frac{3(2+\lambda^*/\lambda_*)}{\lambda_*}\right)^{p_1} \\
& C_{Stoch}(n,\gamma) := n^{1-p_2/2}
\left(\frac{6p_2L_{p_2}}{\lambda_*\sqrt{p_2-1}}\right)^{p_2}
\left(\frac{\rho_{p_2}^2}{\alpha_0\gamma}
+\frac{3(p_2-1)M_{p_2}^2}{\alpha_0}n\gamma\right)^{p_2/2} \\
& C_{Rem}(n,\gamma) := n^{1-p_3/2}
\left(\frac{3(2+\lambda^*/\lambda_*)}{\lambda_*}\right)^{p_3}
\left(\frac{2p_3C_LC_{4,2p_3}}{\sqrt{p_3-1}}\right)^{p_3} \\
& \cdot \left(
\frac{\rho_{2p_3}}{\alpha_0^{3/2}\sqrt\gamma}
+\frac{2\sqrt{3(2p_3-1)}M_{2p_3}}{\alpha_0^{3/2}}
\sqrt{n\gamma}\right)^{p_3}
\end{aligned}$$

Then, for all $y > y_0$,
$$\begin{aligned}
& \mathbb{P}\left[ \max_{1 \le m \le n} \left\| \sum_{k=1}^m (\Delta_k - \Delta_k^\diamond) \right\|_2 \ge y \right] \le \frac{C_{Lin} \cdot n}{(y - y_0)^{p_1}} + \frac{C_{Stoch}(n, \gamma) \cdot n^{p_2/2 - 1}}{(y - y_0)^{p_2}} + \frac{C_{Rem}(n, \gamma) \cdot n^{p_3/2 - 1}}{(y - y_0)^{p_3}} 
\end{aligned}$$
\end{theorem}

\begin{proof}
Now define
$$\begin{aligned}
& \mathcal{D}_m := \sum_{k=1}^m \mathbb{E}[I_{1,k}] - \sum_{k=1}^m \gamma \sum_{j=1}^k (I_d - \gamma A)^{k-j} \Big( \gamma b - R_{det}(x((j-1)\gamma)) \Big) \\
& = \gamma \sum_{k=1}^m \sum_{j=1}^k (I_d - \gamma A)^{k-j} \Big( \mathbb{E}[R_j] - \gamma b + R_{det}(x((j-1)\gamma)) \Big) 
\end{aligned}$$
and thus
$$\begin{aligned}
& \max_{1 \le m \le n} \|\mathcal{D}_m\|_2 \le \frac{2+\lambda^*/\lambda_*}{\lambda_*} \max_{1 \le m \le n} \left\| \sum_{j=1}^m \Big( \mathbb{E}[R_j] - \gamma b + R_{det}(x((j-1)\gamma)) \Big) \right\|_2 \quad \text{, by Lemma \ref{lem:max_convolution_bound}}
\end{aligned}$$

We bound this by the identical bounds as in \ref{thm:fclt_init_bias}, from which we have
$$\begin{aligned}
& \max_{1 \le m \le n} \|\mathcal{D}_m\|_2 \le \frac{2+\lambda^*/\lambda_*}{\lambda_*} \sum_{j=1}^n \left\| \mathbb{E}[R_j] - \gamma b + R_{det}(x((j-1)\gamma)) \right\|_2 \le \frac{2+\lambda^*/\lambda_*}{\lambda_*} \big( B_{Rem} + B_{Cov} \big) =: y_0
\end{aligned}$$
where $y_0$ is simply the sum of the error bounds derived in theorem \ref{thm:fclt_init_bias}.
Using the same decomposition as before(as in theorem \ref{thm:fclt_init_bias}) for the SGD iterates, we get that
$$\begin{aligned}
& \sum_{k=1}^m (\Delta_k - \Delta_k^\diamond) = - \sum_{k=1}^m (L_k - L_k^\diamond) + \sum_{k=1}^m \tilde{I}_{1,k} + \sum_{k=1}^m I_{2,k} + \mathcal{D}_m \\
& \implies \mathbb{P}\left[ \max_{1 \le m \le n} \left\| \sum_{k=1}^m (\Delta_k - \Delta_k^\diamond) \right\|_2 \ge y \right] \le \mathbb{P}\left[ \max_{1 \le m \le n} \left\| \sum_{k=1}^m (L_k - L_k^\diamond) \right\|_2 \ge \frac{y - y_0}{3} \right] \\
& \\
&\quad+ \mathbb{P}\left[ \max_{1 \le m \le n} \left\| \sum_{k=1}^m \tilde{I}_{1,k} \right\|_2 \ge \frac{y - y_0}{3} \right] \\
&\quad+ \mathbb{P}\left[ \max_{1 \le m \le n} \left\| \sum_{k=1}^m I_{2,k} \right\|_2 \ge \frac{y - y_0}{3} \right]
\end{aligned}$$

\bigskip

For $I_2$, apply Doob followed by Rio exactly as in the $I_2$ bound in theorem \ref{KMT_Romberg}, with $\eta=\gamma$. Lemma \ref{lem:max_convolution_bound} and Markov's inequality then give
$$\begin{aligned}
&\mathbb P\left[\max_{m\le n}\left\|\sum_{k=1}^m I_{2,k}\right\|_2
\ge\frac{y-y_0}{3}\right]\\
&\le\frac1{(y-y_0)^{p_2}}
\left(\frac{6p_2L_{p_2}}{\lambda_*\sqrt{p_2-1}}\right)^{p_2}
\left(\frac{\rho_{p_2}^2}{\alpha_0\gamma}
+\frac{3(p_2-1)M_{p_2}^2}{\alpha_0}n\gamma\right)^{p_2/2}\\
&=\frac{C_{Stoch}(n,\gamma)n^{p_2/2-1}}{(y-y_0)^{p_2}}.
\end{aligned}$$

For the centered Taylor remainder $\tilde{I}_1$, we proceed similarly
For each fixed $\ell\ge1$, $\mathcal P_t(R_{t+\ell})$
is a martingale difference sequence in $t$.
$$\begin{aligned}
& \mathbb{P}\left[ \max_{1 \le m \le n} \left\| \sum_{k=1}^m \tilde{I}_{1,k} \right\|_2 \ge \frac{y-y_0}{3} \right] \le \mathbb{P}\left[ \max_{1 \le m \le n} \left\| \sum_{j=1}^m \tilde{R}_j \right\|_2 \ge \frac{\lambda_*(y-y_0)}{3(2+\lambda^*/\lambda_*)} \right] \text{, by Lemma \ref{lem:max_convolution_bound}} \\
&\le
 \frac{3^{p_3}(2+\lambda^*/\lambda_*)^{p_3}}
 {\lambda_*^{p_3}(y-y_0)^{p_3}}
 \mathbb E\max_{m\le n}
 \left\|
 \sum_{\ell=1}^{n-1}\sum_{t=1}^{(m-\ell)_+}
 \mathcal P_t(R_{t+\ell})
 \right\|_2^{p_3}
 \text{, since $\mathcal P_t(R_t)=0$}\\
&\le
 \frac{3^{p_3}(2+\lambda^*/\lambda_*)^{p_3}}
 {\lambda_*^{p_3}(y-y_0)^{p_3}}
 \left(\frac{p_3}{p_3-1}\right)^{p_3}
 \left(
 \sum_{\ell=1}^{n-1}
 \left\|
 \sum_{t=1}^{n-\ell}\mathcal P_t(R_{t+\ell})
 \right\|_{p_3}
 \right)^{p_3}\\
&\text{, by Minkowski and Doob's maximal
 $\mathcal L^{p_3}$ inequality at each fixed $\ell$,
 as in Lemma \ref{Sum centered taylor error to 0},}\\
&\le
 \frac{3^{p_3}(2+\lambda^*/\lambda_*)^{p_3}}
 {\lambda_*^{p_3}(y-y_0)^{p_3}}
 \left(\frac{p_3}{\sqrt{p_3-1}}\right)^{p_3}
 \left(
 \sum_{\ell=1}^{n-1}
 \left(
 \sum_{t=1}^{n-\ell}
 \|\mathcal P_t(R_{t+\ell})\|_{p_3}^2
 \right)^{1/2}
 \right)^{p_3}
 \text{, by Rio's inequality.}\\
&\le
 \frac{3^{p_3}(2+\lambda^*/\lambda_*)^{p_3}}
 {\lambda_*^{p_3}(y-y_0)^{p_3}}
 \left(
 \frac{2p_3C_LC_{4,2p_3}\gamma}
 {\sqrt{p_3-1}}
 \right)^{p_3}\\
& , \cdot
 \left(
 \sum_{\ell=1}^{n-1}
 \left\{
 \frac{\rho_{2p_3}(1-\alpha_0\gamma)^{\ell-1}}
 {\sqrt{\alpha_0\gamma}}
 +\sqrt{3(2p_3-1)}M_{2p_3}\alpha_0^{-1/2}
 \sqrt{n\gamma}
 (1-\alpha_0\gamma)^{(\ell-1)/2}
 \right\}
 \right)^{p_3}\\
&\le
 \frac{3^{p_3}(2+\lambda^*/\lambda_*)^{p_3}}
 {\lambda_*^{p_3}(y-y_0)^{p_3}}
 \left(
 \frac{2p_3C_LC_{4,2p_3}\gamma}
 {\sqrt{p_3-1}}
 \right)^{p_3}\\
& , \cdot
 \left(
 \frac{\rho_{2p_3}}
 {\sqrt{\alpha_0\gamma}}\frac1{\alpha_0\gamma}
 +\sqrt{3(2p_3-1)}M_{2p_3}\alpha_0^{-1/2}
 \sqrt{n\gamma}\frac2{\alpha_0\gamma}
 \right)^{p_3}\\
&\text{, since }
 1-(1-\alpha_0\gamma)^{1/2}
 \ge \tfrac12\alpha_0\gamma,\\
&=
 \frac{3^{p_3}(2+\lambda^*/\lambda_*)^{p_3}}
 {\lambda_*^{p_3}(y-y_0)^{p_3}}
 \left(
 \frac{2p_3C_LC_{4,2p_3}}
 {\sqrt{p_3-1}}
 \right)^{p_3}
 \left(
 \frac{\rho_{2p_3}}{\alpha_0^{3/2}\sqrt\gamma}
 +\frac{2\sqrt{3(2p_3-1)}M_{2p_3}}
 {\alpha_0^{3/2}}\sqrt{n\gamma}
 \right)^{p_3} \\
& = \frac{C_{Rem}(n, \gamma) \cdot n^{p_3/2 - 1}}{(y - y_0)^{p_3}} \text{, defining } C_{Rem}(n,\gamma)
\end{aligned}$$

Finally, by \ref{kmt for the linear component}
$$\begin{aligned}
& \mathbb{P}\left[ \max_{1 \le m \le n} \left\| \sum_{k=1}^m (L_k - L_k^\diamond) \right\|_2 \ge \frac{y-y_0}{3} \right] \le K_{p_1,d,S}\, n\,\mathbb{E}\|Z_1\|_2^{p_1} \left(\frac{3(2+\lambda^*/\lambda_*)}{\lambda_* (y-y_0)}\right)^{p_1} \\
& = \frac{C_{Lin} \cdot n}{(y-y_0)^{p_1}} \text{, defining } C_{Lin}
\end{aligned}$$

\end{proof}
\section{Trajectory KMT}\label{app:trajectory-kmt}

Let $b$ be an even nonnegative integer,
$J=\lceil T/\gamma\rceil$, and $N=b+J$.
Define
\begin{align}
E_1
&=C_0\sqrt\gamma\left\{
q_A^{b/2}\sum_{j=1}^{b/2}q_A^{b/2-j}v_{j-1}
+\sum_{j=b/2+1}^{N}v_{j-1}\right\},
\label{eq:E1}\\
E_2
&=8L_2^2\gamma q_A^b
\sum_{j=1}^{b/2}q_A^{b-2j}v_{j-1}
+32L_2^2c_A^2\gamma
\sum_{j=b/2+1}^{N}v_{j-1}
\label{eq:E2}
\end{align}
where $v_j=v_{2,j}$ and $c_A=1+\lambda^*/\lambda_*$,
and for $0<\delta<1$, let
\begin{align*}
D^2
&=2\operatorname{tr}(S)\left\{
\frac{(\lambda^*)^2}{\lambda_*}
+\frac{2(\lambda^*)^4}{\lambda_*^3}\right\},\\
g_{\gamma,b,T}(\delta)
&=e^{-\lambda_*b\gamma}
\sqrt{2d\|\Gamma\|_2\log(2d/\delta)}\\
&\quad+\gamma\left(D+\lambda^*\sqrt{\|\Gamma\|_2}\right)
\sqrt{2d\log(2d(J+1)/\delta)}\\
&\quad+(1+\gamma\lambda^*)\|S^{1/2}\|_2
\sqrt{2d\gamma\log(4d(J+1)/\delta)}.
\end{align*}
\begin{theorem}[Quenched trajectory KMT approximation]
\label{thm:trajectory-kmt-new}
We use the assumptions from Theorem~\ref{thm:trajectory_fclt}, and
assume additionally $M_p<\infty$ for some $p>2$,
$\gamma\lambda^*\le1/2$ and $0<\alpha_0\gamma<1$.
For any even burn in $b\ge0$ and $T>0$, and
\[
Y_\gamma(t)=\frac{\theta_{b+\lfloor t/\gamma\rfloor}-\theta^*}
{\sqrt\gamma}
\]
we can construct a stationary OU process on the same space
\[
dX_t=-AX_t\,dt+S^{1/2}\,dB_t,   
X_0\sim N(0,\Gamma),\quad A\Gamma+\Gamma A=S,
\]
such that, for every $z,u,v>0$ and $0<\delta<1$,
\begin{align}
&\mathbb P\left\{\sup_{0\le t\le T}\|Y_\gamma(t)-X_t\|_2>
\frac{q_A^b\rho_2}{\sqrt\gamma}+c_A\sqrt\gamma z+u+v
+g_{\gamma,b,T}(\delta)\right\}\notag\\
& \le\frac{K_{p,d,S}NM_p^p}{z^p}
+\frac{E_1}{u}+\frac{E_2}{v^2}+4\delta.
\label{eq:trajectory-finite}
\end{align}
In particular, for fixed $T$ and
\[
b=2\left\lceil\frac{\nu}{2\gamma}\log(1/\gamma)\right\rceil,
 ,  \nu>3/\min(\lambda_*,\alpha_0),
\]
\begin{equation}\label{eq:trajectory-rate}
\sup_{t\le T}\|Y_\gamma(t)-X_t\|_2
=O_{\mathbb P}\!\left(
\gamma^{1/2-1/p}\{\log(1/\gamma)\}^{1/p}
+\sqrt\gamma\log(1/\gamma)\right)=o_{\mathbb P}(1).
\end{equation}
\end{theorem}
\begin{proof}
Let $k(t)=b+\lfloor t/\gamma\rfloor$ and use the decomposition from the proof of Theorem~\ref{thm:trajectory_fclt} with burn in $b/2$.
By the same logic/intermediate result(for all results below) from Theorem~\ref{thm:trajectory_fclt}, we get
\begin{align*}
\sup_{t\le T}\frac{\|Q^{k(t)}\Delta_0\|_2}{\sqrt\gamma}
&\le\frac{q_A^b\rho_2}{\sqrt\gamma},
 \text{as in the initialization bound,}\\
\mathbb E\max_{b\le k\le N}\frac{\|I_{1,k}\|_2}{\sqrt\gamma}
&\le E_1,
 \text{as in the bound for $I_1$,}\\
\mathbb E\max_{b\le k\le N}\frac{\|I_{2,k}\|_2^2}{\gamma}
&\le E_2,
 \text{as in the Doob bound for $I_2$.}
\end{align*}

Let $U_k=\sum_{j=1}^k(Z_j-Z_j^\diamond)$.
By the same summation by parts used in
Theorem~\ref{kmt for the linear component},
\begin{align*}
L_k-L_k^\diamond
&=\gamma U_k-\gamma^2A\sum_{j=1}^{k-1}Q^{k-1-j}U_j,\\
\max_{k\le N}\frac{\|L_k-L_k^\diamond\|_2}{\sqrt\gamma}
&\le c_A\sqrt\gamma\max_{k\le N}\|U_k\|_2,
 \text{calculating the same geometric sum,}\\
\mathbb P\!\left\{
\max_{k\le N}\frac{\|L_k-L_k^\diamond\|_2}{\sqrt\gamma}
>c_A\sqrt\gamma z\right\}
&\le\frac{K_{p,d,S}NM_p^p}{z^p},
 \text{by \eqref{eq:innovation-new} and Markov's inequality.}
\end{align*}

Pick a Brownian motion $\widetilde B$ with
\[
S^{1/2}\{\widetilde B(j\gamma)-\widetilde B((j-1)\gamma)\}
=-\sqrt\gamma Z_j^\diamond.
\]
We can make this construction by using independent Brownian bridges between $ \{ j \gamma: j \in [N] \}$ points and independent
Brownian coordinates in the null space of $S$.
For an independent $\xi\sim N(0,\Gamma)$, let
\begin{align*}
V(s)&=\int_0^s e^{-A(s-r)}S^{1/2}\,d\widetilde B(r),\\
\overline X(s)&=e^{-As}\xi+V(s),    X_t=\overline X(b\gamma+t),
 \text{so $X$ is stationary OU.}
\end{align*}
At this point it only remains to bound in probability the difference between our linear Gaussian process and the OU process. We start by deriving a relevant bound for the eigenvalue difference of the linear operators we will be integrating over. First, for $0\le a\gamma\le1/2$,
\begin{align*}
&0\le-\log(1-a\gamma)-a\gamma
 =\int_0^{a\gamma}\frac{u}{1-u}\,du\\
&\le\int_0^{a\gamma}2u\,du
=a^2\gamma^2,
 \text{since $1-u\ge1/2$.}
\end{align*}
Thus, for an integer $r\ge0$,
\begin{align*}
&0\le e^{-ar\gamma}-(1-a\gamma)^r
=e^{-ar\gamma}
  \left[1-e^{-r\{-\log(1-a\gamma)-a\gamma\}}\right]\\
&\le r\{-\log(1-a\gamma)-a\gamma\}e^{-ar\gamma},
 \text{since $1-e^{-x}\le x$,}\\
&\le ra^2\gamma^2e^{-ar\gamma}.
\end{align*}
and
\begin{align*}
|e^{-ar\gamma}-e^{-a(r\gamma+s)}|
&=e^{-ar\gamma}(1-e^{-as})\\
&\le a\gamma e^{-ar\gamma}.
\end{align*}
so
\begin{align*}
|(1-a\gamma)^r-e^{-a(r\gamma+s)}|
&\le |(1-a\gamma)^r-e^{-ar\gamma}|
   +|e^{-ar\gamma}-e^{-a(r\gamma+s)}|\\
&\le\gamma(a+r\gamma a^2)e^{-ar\gamma}.
\end{align*}
Since $A$ is symmetric, $Q^r$ and $e^{-A(r\gamma+s)}$
are diagonalizable in a common orthonormal eigenbasis. Hence,
\begin{align*}
\|Q^r-e^{-A(r\gamma+s)}\|_2
&=\max_{a\in\operatorname{spectrum}(A)}
  |(1-a\gamma)^r-e^{-a(r\gamma+s)}|\\
&\le\gamma\{\lambda^*+r\gamma(\lambda^*)^2\}
 e^{-\lambda_*r\gamma},
 \text{since $\lambda_*\le a\le\lambda^*$.}
\end{align*}

We now want to rewrite both our OU process and linear Gaussian process using the same Brownian motion.
So, using our construction of the Brownian motion
\begin{align*}
-\frac{L_k^\diamond}{\sqrt\gamma}
&=-\sqrt\gamma\sum_{j=1}^kQ^{k-j}Z_j^\diamond\\
&=\sum_{j=1}^kQ^{k-j}S^{1/2}
 \{\widetilde B(j\gamma)-\widetilde B((j-1)\gamma)\}\\
&=\sum_{j=1}^k\int_{(j-1)\gamma}^{j\gamma}
 Q^{k-j}S^{1/2}\,d\widetilde B(u).
\end{align*}
Therefore, subtracting
$V(k\gamma)=\int_0^{k\gamma}
e^{-A(k\gamma-u)}S^{1/2}\,d\widetilde B(u)$ we get
\begin{align*}
-\frac{L_k^\diamond}{\sqrt\gamma}-V(k\gamma)
&=\sum_{j=1}^k\int_{(j-1)\gamma}^{j\gamma}
 \{Q^{k-j}-e^{-A(k\gamma-u)}\}
 S^{1/2}\,d\widetilde B(u).
\end{align*}
On the $j$th integration interval, let $r=k-j$ and $s=j\gamma-u$.
Then $k\gamma-u=r\gamma+s$ and $0\le s\le\gamma$.
Therefore,
\begin{align*}
&\mathbb E\|-L_k^\diamond/\sqrt\gamma-V(k\gamma)\|_2^2\\
&=\sum_{j=1}^k\int_{(j-1)\gamma}^{j\gamma}
 \|\{Q^{k-j}-e^{-A(k\gamma-u)}\}S^{1/2}\|_F^2\,du,
 \text{by It\^o's isometry,}\\
&\le\operatorname{tr}(S)
 \sum_{r=0}^{k-1}\int_0^\gamma
 \|Q^r-e^{-A(r\gamma+s)}\|_2^2\,ds,
 \text{since $\|CS^{1/2}\|_F^2\le\|C\|_2^2\operatorname{tr}(S)$,}\\
&\le\operatorname{tr}(S)\gamma^3
 \sum_{r=0}^{k-1}
 \{\lambda^*+r\gamma(\lambda^*)^2\}^2
 e^{-2\lambda_*r\gamma}\\
&\le\operatorname{tr}(S)\gamma^3
 \sum_{r\ge0}
 \{\lambda^*+r\gamma(\lambda^*)^2\}^2
 e^{-2\lambda_*r\gamma}.
\end{align*}

Using $(x+y)^2\le2x^2+2y^2$ and
$\sum_{r\ge0}r^2x^r=x(1+x)/(1-x)^3$, we obtain
\begin{align*}
&\mathbb E\|-L_k^\diamond/\sqrt\gamma-V(k\gamma)\|_2^2\\
&\le2\operatorname{tr}(S)\gamma^3
\left\{
\frac{(\lambda^*)^2}{1-e^{-2\lambda_*\gamma}}
+\frac{\gamma^2(\lambda^*)^4
 e^{-2\lambda_*\gamma}(1+e^{-2\lambda_*\gamma})}
 {(1-e^{-2\lambda_*\gamma})^3}
\right\}\\
&\le2\operatorname{tr}(S)\gamma^3
\left\{
\frac{(\lambda^*)^2}{\lambda_*\gamma}
+\frac{2\gamma^2(\lambda^*)^4}{(\lambda_*\gamma)^3}
\right\}\\
&=2\operatorname{tr}(S)
\left\{
\frac{(\lambda^*)^2}{\lambda_*}
+\frac{2(\lambda^*)^4}{\lambda_*^3}
\right\}\gamma^2
=D^2\gamma^2.
\end{align*}

Where in the penultimate step we use that $\lambda_*\gamma\le1/2$ implies
\begin{align*}
&1-e^{-2\lambda_*\gamma}
=\int_0^{2\lambda_*\gamma}e^{-u}\,du\\
&\ge 2\lambda_*\gamma-2(\lambda_*\gamma)^2
\ge\lambda_*\gamma,
 \text{since $e^{-u}\ge1-u$.}
\end{align*} Also, for
$j=0,\ldots,J-1$ and $0\le s\le\gamma$,
\begin{align*}
\|X_{j\gamma+s}-X_{j\gamma}\|_2
&\le\gamma\lambda^*\|X_{j\gamma}\|_2\\
&\quad +(1+\gamma\lambda^*)\|S^{1/2}\|_2
\sup_{0\le r\le\gamma}
\|\widetilde B((b+j)\gamma+r)-\widetilde B((b+j)\gamma)\|_2,
\quad\text{by integration by parts.}
\end{align*}
Now the
$-L_k^\diamond/\sqrt\gamma-V(k\gamma)$ are centered
Gaussians, and their coordinate variances are at most
$D^2\gamma^2$ by our preceding mean square error bound. Thus, using the well known theory for Gaussian tail bounds and a union bound we obtain
\begin{align*}
&\mathbb P\left\{
\max_{0\le j\le J}
\left\|-\frac{L_{b+j}^\diamond}{\sqrt\gamma}
      -V((b+j)\gamma)\right\|_2
>\gamma D\sqrt{2d\log(2d(J+1)/\delta)}
\right\}\\
&\le2d(J+1)
\exp\{-\log(2d(J+1)/\delta)\}
=\delta,
 \text{by a union bound.}
\end{align*}
Similarly, stationarity gives $X_{j\gamma}\sim N(0,\Gamma)$,
so
\begin{align*}
&\mathbb P\left\{
\max_{0\le j\le J}\|X_{j\gamma}\|_2
>\sqrt{2d\|\Gamma\|_2\log(2d(J+1)/\delta)}
\right\}\\
&\le2d(J+1)
\exp\{-\log(2d(J+1)/\delta)\}
=\delta.
\end{align*}

To bound the Brownian increments
\begin{align*}
&\mathbb P\left\{
\max_{0\le j<J}\sup_{0\le s\le\gamma}
\|\widetilde B((b+j)\gamma+s)
  -\widetilde B((b+j)\gamma)\|_2>x
\right\}\\
&\le
\sum_{j=0}^{J-1}\sum_{\ell=1}^d
\mathbb P\left\{
\sup_{0\le s\le\gamma}
|\widetilde B_\ell((b+j)\gamma+s)
 -\widetilde B_\ell((b+j)\gamma)|>x/\sqrt d
\right\}\\
&\le4dJ\exp\left(-\frac{x^2}{2d\gamma}\right),
 \text{by reflection and the Gaussian tail bound.}
\end{align*}
Taking $x=\sqrt{2d\gamma\log(4d(J+1)/\delta)}$ gives
\begin{align*}
&\mathbb P\left\{
\max_{0\le j<J}\sup_{0\le s\le\gamma}
\|\widetilde B((b+j)\gamma+s)
  -\widetilde B((b+j)\gamma)\|_2
>\sqrt{2d\gamma\log(4d(J+1)/\delta)}
\right\}\\
&\le\frac{J}{J+1}\delta\le\delta.
\end{align*}

So conditioning on these four events not happening and letting
$j=\lfloor t/\gamma\rfloor$ and recalling
$X_{j\gamma}=e^{-A(b+j)\gamma}\xi+V((b+j)\gamma)$,
\begin{align*}
\left\|-\frac{L_{k(t)}^\diamond}{\sqrt\gamma}-X_t\right\|_2
&\le
\left\|-\frac{L_{b+j}^\diamond}{\sqrt\gamma}
      -V((b+j)\gamma)\right\|_2
+\|e^{-A(b+j)\gamma}\xi\|_2
+\|X_{j\gamma}-X_t\|_2\\
&\le
\gamma D\sqrt{2d\log(2d(J+1)/\delta)}
+e^{-\lambda_*b\gamma}
 \sqrt{2d\|\Gamma\|_2\log(2d/\delta)}\\
&\quad+\gamma\lambda^*
 \sqrt{2d\|\Gamma\|_2\log(2d(J+1)/\delta)}\\
&\quad+(1+\gamma\lambda^*)\|S^{1/2}\|_2
 \sqrt{2d\gamma\log(4d(J+1)/\delta)},
 \text{by the preceding OU increment bound,}\\
&:=g_{\gamma,b,T}(\delta).
\end{align*}
Therefore,
\begin{align*}
&\mathbb P\left\{
\sup_{t\le T}
\|-L_{k(t)}^\diamond/\sqrt\gamma-X_t\|_2
>g_{\gamma,b,T}(\delta)
\right\}\\
&\le\delta+\delta+\delta+\delta
=4\delta,
 \text{by a union bound.}
\end{align*}
Triangle inequality, Markov's inequality for $I_1,I_2$
and a union bound then imply \eqref{eq:trajectory-finite}.

So for our stated burn in, the same geometric estimates as in
Theorem~\ref{thm:trajectory_fclt} give errors
\begin{align*}
N&=O(\gamma^{-1}\log(1/\gamma)),\\
E_1&=O(\sqrt\gamma\log(1/\gamma)),   
E_2=O(\gamma\log(1/\gamma)),\\
\frac{q_A^b\rho_2}{\sqrt\gamma}
&\le\rho_2\gamma^{\lambda_*\nu-1/2}=o(\sqrt\gamma),\\
g_{\gamma,b,T}(\delta)
&=O_\delta\!\left(\sqrt{\gamma\log(1/\gamma)}\right).
\end{align*}
Substituting this in \eqref{eq:trajectory-finite} proves \eqref{eq:trajectory-rate}.
\end{proof}

\section{Consequences for monitoring and accumulated loss}\label{app:applications}

\begin{proposition}[Transfer of boundary crossing probabilities]\label{prop:crossing}
Suppose $U,V$ are stochastic processes on $[0,T]$(or $[T]$) satisfying
$\mathbb P\{\|U-V\|_\infty>r\}\le\delta$.
Let $f:\mathbb R^d\to\mathbb R$ be $L_f$-Lipschitz and $h$ a deterministic boundary function. For
$M_f(U)=\sup_{t\le T}\{f(U_t)-h(t)\}$,
\begin{equation}\label{eq:crossing}
\mathbb P\{M_f(V)>L_fr\}-\delta
\le\mathbb P\{M_f(U)>0\}
\le\mathbb P\{M_f(V)>-L_fr\}+\delta.
\end{equation}
so
\[
\left|\mathbb P\{M_f(U)>0\}-\mathbb P\{M_f(V)>0\}\right|
\le\delta+\mathbb P\{|M_f(V)|\le L_fr\}.
\]

\end{proposition}
\begin{proof}
On the event, $|M_f(U)-M_f(V)|\le L_fr$, it is clear
$\{M_f(V)>L_fr\}\subseteq\{M_f(U)>0\}\subseteq\{M_f(V)>-L_fr\}$. Then we can take expectations of the corresponding indicators to prove the theorem.
\end{proof}

\begin{corollary}\label{cor:monitor}
Suppose $\mathbb P\{\|U-V\|_\infty>r\}\le\delta$. If a deterministic boundary $h$ satisfies
$\mathbb P\{\exists t\le T:f(V_t)>h(t)\}\le\alpha$, then the rule
\[
\tau=\inf\{t\le T:f(U_t)>h(t)+L_fr\}
\]
satisfies $\mathbb P\{\tau\le T\}\le\alpha+\delta$.
For example taking, $f(x)=\|x\|_2$ and $U=Y_\gamma$,then with probability at least $1-\alpha-\delta$,
\[
\|\theta_{b+\lfloor t/\gamma\rfloor}-\theta^*\|_2
\le\sqrt\gamma\{h(t)+r\} , (0\le t\le T).
\]
\end{corollary}
\begin{proof}
Apply the set inequalities in Proposition~\ref{prop:crossing}.
\end{proof}

\begin{corollary}[CUSUM functional]\label{cor:cusum}
For a unit vector $a$, define
\[
\mathcal C(U)=\sup_{0\le s\le t\le T}a^\top(U_t-U_s).
\]
In the setting of Proposition~\ref{prop:crossing},
\[
\mathbb P\{\mathcal C(V)>h+2r\}-\delta
\le\mathbb P\{\mathcal C(U)>h\}
\le\mathbb P\{\mathcal C(V)>h-2r\}+\delta.
\]
\end{corollary}
\begin{proof}
$|\mathcal C(U)-\mathcal C(V)|\le2\|U-V\|_\infty$.
\end{proof}

\begin{theorem}[Distributional approximation of regret]\label{thm:loss}
Let $T=J\gamma$ and
\[
\mathcal L_\gamma(T)=\sum_{k=b}^{b+J-1}\{G(\theta_k)-G(\theta^*)\},
   \mathcal Q(T)=\frac12\int_0^T X_t^\top A X_t\,dt.
\]
Suppose $\mathbb P\{\|Y_\gamma-X\|_\infty>r\}\le\delta$ and
$\mathbb P\{\|X\|_\infty>M\}\le\beta$. If $\sqrt\gamma(M+r)\le\kappa$, then, with probability at least $1-\delta-\beta$,
\begin{equation}\label{eq:loss-transfer}
|\mathcal L_\gamma(T)-\mathcal Q(T)|
\le T\left\{\frac{\lambda^*}{2}r(2M+r)
+\frac{L}{6}\sqrt\gamma(M+r)^3\right\}=:e_\gamma.
\end{equation}
Consequently, for every $z$,
\[
\mathbb P\{\mathcal Q(T)>z+e_\gamma\}-\delta-\beta
\le\mathbb P\{\mathcal L_\gamma(T)>z\}
\le\mathbb P\{\mathcal Q(T)>z-e_\gamma\}+\delta+\beta.
\]
With the Burn in of Theorem~\ref{thm:trajectory-kmt-new}, for fixed $T$ and $J=\lfloor T/\gamma\rfloor$,
\[
\mathcal L_\gamma(J\gamma)\Rightarrow\frac12\int_0^T X_t^\top A X_t\,dt.
\]
\end{theorem}
\begin{proof}
For $\|x\|_2\le\kappa$,
\begin{align*}
&\left|G(\theta^*+x)-G(\theta^*)-\frac12x^\top Ax\right|\\
&=\left|\int_0^1(1-s)x^\top
\{\nabla^2G(\theta^*+sx)-A\}x\,ds\right|\\
&\le L\|x\|_2^3\int_0^1s(1-s)\,ds
=\frac L6\|x\|_2^3.
\end{align*}
On $\{\|Y_\gamma-X\|_\infty\le r,\ \|X\|_\infty\le M\}$,
\begin{align*}
&\|Y_\gamma\|_\infty\le M+r,\\
&\implies |Y_\gamma(t)^\top AY_\gamma(t)-X_t^\top AX_t|\\
&=\left|(Y_\gamma(t)-X_t)^\top A(Y_\gamma(t)+X_t)\right|,
\quad\text{since $A$ is symmetric,}\\
&\le\lambda^*r(2M+r).
\end{align*}
Since $T=J\gamma$,
\begin{align*}
&\mathcal L_\gamma(T)
=\frac1\gamma\int_0^T
\{G(\theta^*+\sqrt\gamma Y_\gamma(t))-G(\theta^*)\}\,dt,
\quad\text{by the definition of $Y_\gamma$,}\\
&\implies |\mathcal L_\gamma(T)-\mathcal Q(T)|\\
&\le\frac L6\sqrt\gamma\int_0^T\|Y_\gamma(t)\|_2^3\,dt
+\frac12\int_0^T
|Y_\gamma(t)^\top AY_\gamma(t)-X_t^\top AX_t|\,dt\\
&\le T\left\{\frac{\lambda^*}{2}r(2M+r)
+\frac L6\sqrt\gamma(M+r)^3\right\}
=e_\gamma.
\end{align*}
Therefore,
\begin{align*}
&\mathbb P\{|\mathcal L_\gamma(T)-\mathcal Q(T)|>e_\gamma\}\\
&\le\mathbb P\{\|Y_\gamma-X\|_\infty>r\}
+\mathbb P\{\|X\|_\infty>M\}
\le\delta+\beta,
\quad\text{by a union bound,}\\
&\mathbb P\{\mathcal Q(T)>z+e_\gamma\}-\delta-\beta\\
&\le\mathbb P\{\mathcal L_\gamma(T)>z\}\\
&\le\mathbb P\{\mathcal Q(T)>z-e_\gamma\}+\delta+\beta,
\quad\text{by the preceding error bound.}
\end{align*}

For fixed $T$ and $J=\lfloor T/\gamma\rfloor$, on
$\{\|X\|_\infty\le M\}$,
\begin{align*}
&0\le\mathcal Q(T)-\mathcal Q(J\gamma)
=\frac12\int_{J\gamma}^T X_t^\top AX_t\,dt
\le\frac{\lambda^*\gamma M^2}{2}.
\end{align*}
Fix $\varepsilon>0$ and $M>0$, and choose $r>0$ so that
$T\lambda^*r(2M+r)/2<\varepsilon/2$.
For sufficiently small $\gamma$,
\begin{align*}
&\sqrt\gamma(M+r)\le\kappa,\qquad
\frac{TL}{6}\sqrt\gamma(M+r)^3
+\frac{\lambda^*\gamma M^2}{2}<\frac{\varepsilon}{2}\\
&\implies
\mathbb P\{|\mathcal L_\gamma(J\gamma)-\mathcal Q(T)|
>\varepsilon\}\\
&\le\mathbb P\{\|Y_\gamma-X\|_\infty>r\}
+\mathbb P\{\|X\|_\infty>M\}\\
&\implies
\limsup_{\gamma\to0}
\mathbb P\{|\mathcal L_\gamma(J\gamma)-\mathcal Q(T)|
>\varepsilon\}
\le\mathbb P\{\|X\|_\infty>M\},
\quad\text{by Theorem~\ref{thm:trajectory-kmt-new}.}
\end{align*}
The last probability tends to zero as $M\to\infty$,
since $X$ has continuous paths on $[0,T]$. Hence
\begin{align*}
&\mathcal L_\gamma(J\gamma)-\mathcal Q(T)
\xrightarrow{\mathbb P}0\\
&\implies\mathcal L_\gamma(J\gamma)\Rightarrow\mathcal Q(T),
\quad\text{by Slutsky's theorem.}
\end{align*}
\end{proof}

% BEGIN NEW APPENDIX SECTIONS
% Paste after the existing appendices. The covariance-estimation assumption
% ass:cov-new is stated once in the accompanying main body.
% Uses the existing theorem environments, natbib, and standard AMS packages.

\section{Covariance estimation and inference}
\label{app:feasible-new}

\subsection{Estimating the Gaussian comparison law}
We will denote in the section $H(\theta,X)=\nabla g(\theta,X)$ and
$J(\theta,X)=\nabla^2g(\theta,X)$.
Let $\Sigma=A^{-1}SA^{-1}$. For a deterministic Burn in $b_n\ge0$, define
\begin{align}
\widehat A_m&=\frac1m\sum_{k=1}^m J(\theta_{b_n+k-1},X_{b_n+k}),   
H_k=H(\theta_{b_n+k-1},X_{b_n+k}),\notag\\
\widehat S_m&=\frac1m\sum_{k=1}^mH_kH_k^\top,   
\widehat\Sigma_m=\widehat A_m^{-1}\widehat S_m\widehat A_m^{-1}.
\label{eq:estimators-new}
\end{align}
On the exceptional event that $\widehat A_m$ is not positive definite, use $I_d$ only for inversion and solving the Lyapunov equation. This event has probability tending to zero uniformly in the respective settings of each theorem. 

\begin{theorem}[Uniform covariance consistency]\label{thm:cov-new}
Assume the paper's second moment noise assumptions and Assumption~\ref{ass:cov-new}, and further assume $\gamma\to0$, $n\gamma\to\infty$, and the initialization is fixed. For each fixed $\tau\in(0,1)$ and any deterministic $b_n\ge0$,
\[
\max_{\lceil\tau n\rceil\le m\le n}
\big\{\|\widehat A_m-A\|_2+\|\widehat S_m-S\|_2
+\|\widehat\Sigma_m-\Sigma\|_2\big\}
\longrightarrow_{\mathbb P}0.
\]
Let $\widehat\Gamma_m$ solve
$\widehat A_m\widehat\Gamma_m+\widehat\Gamma_m\widehat A_m=\widehat S_m$.
Then $\widehat\Gamma_m$ converges uniformly to $\Gamma$ on $\lceil\tau n\rceil\le m\le n$ as well.
If additionally $M_4<\infty$ and $\mathbb E\|J(\theta^*,X)\|_2^2<\infty$, the maximum defined  above, including the $\Gamma$ error, is
\begin{equation}\label{eq:cov-rate-new}
O_{\mathbb P}\left(n^{-1/2}+\sqrt\gamma+
\frac{\rho_2(1-\alpha_0\gamma)^{b_n/2}
+\rho_2^2(1-\alpha_0\gamma)^{b_n}}{n\gamma}\right).
\end{equation}
The fourth moments are needed only for this stated rate, not for consistency.
\end{theorem}
\begin{proof}
Let $q=1-\alpha_0\gamma$ and
$a_k=(\mathbb E\|\Delta_{b_n+k-1}\|_2^2)^{1/2}$.
The existing moment bound gives
\begin{align*}
&a_k
\le\rho_2q^{(b_n+k-1)/2}
+\frac{\sqrt3M_2}{\sqrt{\alpha_0}}\sqrt\gamma,\\
&a_k^2
\le\rho_2^2q^{b_n+k-1}
+\frac{3M_2^2}{\alpha_0}\gamma,\\
&\sum_{k=1}^n a_k
\le\frac{2\rho_2q^{b_n/2}}{\alpha_0\gamma}
+\frac{\sqrt3nM_2}{\sqrt{\alpha_0}}\sqrt\gamma,
\quad\text{since $1-\sqrt q\ge\alpha_0\gamma/2$,}\\
&\sum_{k=1}^n a_k^2
\le\frac{\rho_2^2q^{b_n}}{\alpha_0\gamma}
+\frac{3nM_2^2}{\alpha_0}\gamma.
\end{align*}

Let $Z_k^*=H(\theta^*,X_{b_n+k})$ and
$J_k^*=J(\theta^*,X_{b_n+k})$.
So by the stochastic Lipschitz assumption we have
\begin{align*}
&\mathbb E\|H_k-Z_k^*\|_2^2\le L_2^2a_k^2,\\
&\mathbb E\|J(\theta_{b_n+k-1},X_{b_n+k})-J_k^*\|_2
\le L_J\mathbb E\|\Delta_{b_n+k-1}\|_2
\le L_Ja_k,\\
&\mathbb E\|H_kH_k^\top-Z_k^*Z_k^{*\top}\|_2\\
&=\mathbb E\|
(H_k-Z_k^*)H_k^\top+Z_k^*(H_k-Z_k^*)^\top
\|_2\\
&\le
(\mathbb E\|H_k-Z_k^*\|_2^2)^{1/2}
\left\{(\mathbb E\|H_k\|_2^2)^{1/2}+M_2\right\},
\quad\text{by Cauchy--Schwarz,}\\
&\le L_2a_k(2M_2+L_2a_k),
\quad\text{by Minkowski's inequality.}
\end{align*}
Consequently,
\begin{align*}
&\mathbb E\max_{\lceil\tau n\rceil\le m\le n}
\left\{
\left\|\widehat A_m-\frac1m\sum_{k=1}^mJ_k^*\right\|_2
+\left\|\widehat S_m-\frac1m\sum_{k=1}^m
Z_k^*Z_k^{*\top}\right\|_2
\right\}\\
&\le\frac1{\tau n}\sum_{k=1}^n
\left\{L_Ja_k+L_2a_k(2M_2+L_2a_k)\right\}\\
&\le
\frac{L_J+2L_2M_2}{\tau}
\left\{\frac{2\rho_2q^{b_n/2}}{\alpha_0n\gamma}
+\frac{\sqrt3M_2}{\sqrt{\alpha_0}}\sqrt\gamma\right\}\\
&\quad+\frac{L_2^2}{\tau}
\left\{\frac{\rho_2^2q^{b_n}}{\alpha_0n\gamma}
+\frac{3M_2^2}{\alpha_0}\gamma\right\}
\longrightarrow0.
\end{align*}
Then by Markov's inequality, both
errors converge uniformly to zero in probability.

For either $V_k=J(\theta^*,X_k)$ or
$V_k=H(\theta^*,X_k)H(\theta^*,X_k)^\top$,
the entries are integrable. The i.i.d. strong law of large numbers,
applied entrywise, gives
\begin{align*}
&\left\|\frac1m\sum_{k=1}^m(V_k-\mathbb EV_k)\right\|_2
\longrightarrow0
\quad\text{almost surely},\\
&\max_{\lceil\tau n\rceil\le m\le n}
\left\|\frac1m\sum_{k=1}^m(V_k-\mathbb EV_k)\right\|_2\\
&\le\sup_{m\ge\lceil\tau n\rceil}
\left\|\frac1m\sum_{k=1}^m(V_k-\mathbb EV_k)\right\|_2
\longrightarrow0
\quad\text{almost surely}.
\end{align*}
Since $b_n$ is deterministic,
\begin{align*}
&(V_{b_n+1},\ldots,V_{b_n+n})
\overset d=(V_1,\ldots,V_n)\\
&\implies
\max_{\lceil\tau n\rceil\le m\le n}
\left\|\frac1m\sum_{k=1}^m
(V_{b_n+k}-\mathbb EV_1)\right\|_2
\xrightarrow{\mathbb P}0.
\end{align*}
Combining these bounds proves uniform consistency
of $\widehat A_m$ and $\widehat S_m$.

Under the additional moment assumptions,
\begin{align*}
&\mathbb E\max_{\lceil\tau n\rceil\le m\le n}
\left\|\frac1m\sum_{k=1}^m(V_k-\mathbb EV_k)\right\|_2^2\\
&\le\frac1{\tau^2n^2}
\mathbb E\max_{m\le n}
\left\|\sum_{k=1}^m(V_k-\mathbb EV_k)\right\|_F^2\\
&\le\frac4{\tau^2n}
\mathbb E\|V_1-\mathbb EV_1\|_F^2,
\quad\text{by Doob and independence,}\\
&\mathbb E\|J(\theta^*,X)-A\|_F^2
\le d\,\mathbb E\|J(\theta^*,X)\|_2^2,\\
&\mathbb E\|H(\theta^*,X)H(\theta^*,X)^\top-S\|_F^2
\le M_4^4.
\end{align*}
These bounds and Markov's inequality applied to the errors give
\begin{align*}
&\max_{\lceil\tau n\rceil\le m\le n}
\left\{\|\widehat A_m-A\|_2+\|\widehat S_m-S\|_2\right\}\\
&=O_{\mathbb P}\left(
n^{-1/2}+\sqrt\gamma+
\frac{\rho_2q^{b_n/2}+\rho_2^2q^{b_n}}{n\gamma}
\right),
\quad q=1-\alpha_0\gamma.
\end{align*}.

For the remaining estimators, let
\[
e_A=\max_{\lceil\tau n\rceil\le m\le n}
\|\widehat A_m-A\|_2,\qquad
e_S=\max_{\lceil\tau n\rceil\le m\le n}
\|\widehat S_m-S\|_2.
\]
On $\{e_A\le\lambda_*/2\}$, uniformly over these $m$,
\begin{align*}
&\lambda_{\min}(\widehat A_m)
\ge\lambda_*-e_A\ge\lambda_*/2,\\
&\|\widehat A_m^{-1}\|_2\le2/\lambda_*,
\qquad \|A^{-1}\|_2\le1/\lambda_*,\\
&\|\widehat A_m^{-1}-A^{-1}\|_2
=\|\widehat A_m^{-1}(A-\widehat A_m)A^{-1}\|_2
\le\frac{2e_A}{\lambda_*^2},\\
&\|\widehat\Sigma_m-\Sigma\|_2\\
&\le
\|\widehat A_m^{-1}(\widehat S_m-S)\widehat A_m^{-1}\|_2
+\|(\widehat A_m^{-1}-A^{-1})S\widehat A_m^{-1}\|_2\\
&\quad+\|A^{-1}S(\widehat A_m^{-1}-A^{-1})\|_2\\
&\le\frac{4e_S}{\lambda_*^2}
+\frac{6\|S\|_2e_A}{\lambda_*^3}.
\end{align*}

On the event $e_A\le\lambda_*/2$, we have
$\lambda_{\min}(\widehat A_m)\ge\lambda_*/2$. Hence,
\begin{align*}
&\widehat\Gamma_m-\Gamma\\
&=\int_0^\infty e^{-\widehat A_mt}
\{\widehat S_m-S-(\widehat A_m-A)\Gamma
-\Gamma(\widehat A_m-A)\}
e^{-\widehat A_mt}\,dt,\\
&\|\widehat\Gamma_m-\Gamma\|_2
\le\int_0^\infty
\|e^{-\widehat A_mt}\|_2^2
\{e_S+2e_A\|\Gamma\|_2\}\,dt\\
&\le\int_0^\infty e^{-\lambda_*t}
\{e_S+2e_A\|\Gamma\|_2\}\,dt,
\quad\text{since $\lambda_{\min}(\widehat A_m)\ge\lambda_*/2$,}\\
&=\frac{e_S+2e_A\|\Gamma\|_2}{\lambda_*}.
\end{align*}
Likewise,
\begin{align*}
&\Gamma=\int_0^\infty e^{-At}Se^{-At}\,dt,\\
&\|\Gamma\|_2
\le\|S\|_2\int_0^\infty e^{-2\lambda_*t}\,dt
=\frac{\|S\|_2}{2\lambda_*},\\
&\implies
\|\widehat\Gamma_m-\Gamma\|_2
\le\frac{e_S}{\lambda_*}
+\frac{\|S\|_2e_A}{\lambda_*^2}.
\end{align*}
Since $\mathbb P\{e_A>\lambda_*/2\}\to0$, these
inequalities prove both uniform consistency and
the stated rate for $\widehat\Sigma_m,\widehat\Gamma_m$.
\end{proof}

 Suppose known bounds satisfy
\[
\mathbb E\|J(\theta^*,X)-A\|_F^2\le K_A^2,
\quad
\mathbb E\|ZZ^\top-S\|_F^2\le K_S^2,
\]
where $\|\cdot\|_F$ is the Frobenius norm($K_S=M_4^2$ is sufficient for these bounds).
Then using a known upper bound for $\rho_2$, define
\begin{align*}
&u_k=\rho_2q^{(b_n+k-1)/2}
+\sqrt{3/\alpha_0}\,M_2\sqrt\gamma,\\
&d_{A,n}=\frac{L_J}{\tau n}\sum_{k=1}^n u_k,
\quad
d_{S,n}=\frac1{\tau n}\sum_{k=1}^n
L_2u_k(2M_2+L_2u_k),\\
&r_{A,n}(\delta)=\frac{4K_A}{\tau\sqrt{n\delta}}
+\frac{4d_{A,n}}{\delta},
\quad
r_{S,n}(\delta)=\frac{4K_S}{\tau\sqrt{n\delta}}
+\frac{4d_{S,n}}{\delta}.
\end{align*}
Then, for $\delta\in(0,1)$,
\begin{equation}\label{eq:cov-confidence-new}
\mathbb P\left\{
\begin{aligned}
&\max_{\lceil\tau n\rceil\le m\le n}
\|\widehat A_m-A\|_2\le r_{A,n}(\delta),\\
&\max_{\lceil\tau n\rceil\le m\le n}
\|\widehat S_m-S\|_2\le r_{S,n}(\delta)
\end{aligned}
\right\}\ge1-\delta.
\end{equation}
This is because the same preceding Doob and replacement bounds(from the consistency proof) give
\begin{align*}
&\mathbb P\{\text{either bound fails}\}\\
&\le
\sum_{F\in\{A,S\}}
\left\{
\frac{4K_F^2/(\tau^2n)}
     {[4K_F/(\tau\sqrt{n\delta})]^2}
+\frac{d_{F,n}}{4d_{F,n}/\delta}
\right\}
=\delta,
\quad\text{by Markov's inequality and a union bound.}
\end{align*}
At $m=n$, intersect our above confidence set with $A\succ0$ and $S\succeq0$ to obtain a confidence set
for $(A,S)$.

\subsection{Simultaneous intervals}
Let $\widetilde\theta_k$  be the kept sequence from some centered FCLT already proved in the paper. Then define
\begin{equation}\label{eq:centered-FCLT-new}
U_n(t)=\frac1{\sqrt n}\sum_{k=1}^{\lfloor nt\rfloor}(\widetilde\theta_k-\theta^*)
\Rightarrow C B_d(t),    CC^\top=\Sigma\succ0,
\end{equation}
in $D([0,1],\mathbb R^d)$, and let $\bar\theta_m=m^{-1}\sum_{k=1}^m\widetilde\theta_k$.

\begin{corollary}[Feasible simultaneous intervals]\label{cor:bands-new}
Fix a nonzero deterministic $v\in\mathbb R^d$, $\tau\in(0,1)$ and $\delta\in(0,1)$. Let $q_{\tau,1-\delta}$ be the $(1-\delta)$-quantile of
$\sup_{\tau\le t\le1}|B_1(t)|/\sqrt t$. Under \eqref{eq:centered-FCLT-new} and Theorem~\ref{thm:cov-new}, the intervals
\[
v^\top\bar\theta_m\ \pm\
q_{\tau,1-\delta}\sqrt{\frac{v^\top\widehat\Sigma_m v}{m}},
 ,  \lceil\tau n\rceil\le m\le n,
\]
cover $v^\top\theta^*$ simultaneously with probability tending to $1-\delta$.
\end{corollary}
\begin{proof}
\begin{align*}
&\max_{\lceil\tau n\rceil\le m\le n}
\frac{\sqrt m\,|v^\top(\bar\theta_m-\theta^*)|}
{\sqrt{v^\top\widehat\Sigma_m v}}\\
&=\sup_{\tau\le t\le1}
\frac{|v^\top U_n(t)|}{\sqrt{t\,v^\top\Sigma v}}
+o_{\mathbb P}(1),
\quad\text{by uniform consistency and $\tau>0$,}\\
&\Rightarrow\sup_{\tau\le t\le1}\frac{|B_1(t)|}{\sqrt t},
\quad\text{by the FCLT and continuous mapping theorem.}
\end{align*}

The limiting supremum has a continuous distribution, since $B_1(t)/\sqrt t$ becomes a stationary OU process under the time change $t=e^s$ and the absolute value maximum of which over our domain is known to have a continuous distribution.

Hence,
\begin{align*}
&\mathbb P\left\{
\max_{\lceil\tau n\rceil\le m\le n}
\frac{\sqrt m\,|v^\top(\bar\theta_m-\theta^*)|}
{\sqrt{v^\top\widehat\Sigma_m v}}
\le q_{\tau,1-\delta}\right\}
\longrightarrow1-\delta.
\end{align*}
\end{proof}

\section{Estimating crossing and loss probabilities}
\label{app:risk-estimation-new}

\subsection{Bias correction and crossing probabilities}
Let $\beta=A^{-1}b$, $J_0=\int_0^\infty x(s)\,ds$,
$W(t)=-A^{-1}S^{1/2}B_d(t)$ and $f_c(t)=ct\beta+J_0/c$.
Fix $\epsilon\in(0,1)$, a direction $v$ and a continuous boundary $h$.
The biased FCLT, with $\sqrt n\gamma\to c\in(0,\infty)$, gives
\begin{align*}
&\max_{\lceil\epsilon n\rceil\le m\le n}
\left\{v^\top n^{-1/2}\sum_{k=1}^m\Delta_k-h(m/n)\right\}\\
&\Rightarrow\sup_{\epsilon\le t\le1}
\{v^\top[W(t)+f_c(t)]-h(t)\},
\quad\text{in Skorohod space by the continuous mapping theorem,}\\
&\implies
\mathbb P\left\{\max_{\lceil\epsilon n\rceil\le m\le n}
\left[v^\top n^{-1/2}\sum_{k=1}^m\Delta_k-h(m/n)\right]>0\right\}\\
&\longrightarrow
\mathbb P\left\{\sup_{\epsilon\le t\le1}
[v^\top\{W(t)+f_c(t)\}-h(t)]>0\right\},
\quad\text{if the limiting maximum has no point mass at zero.}
\end{align*}
To apply the continuous mapping theorem here we note that uniform path convergence
implies convergence of maxima, and $h$ is uniformly continuous.
Under the centered Burn in or Richardson-Romberg FCLT, replace $f_c$ by zero.
For the trajectory theorem we similarly have
\begin{align*}
&\left|\max_{0\le k\le N}v^\top\Delta_{b+k}/\sqrt\gamma
-\sup_{0\le t\le N\gamma}v^\top X_t\right|
\le\|v\|_2r,
\quad\text{on $\{\|Y_\gamma-X\|_\infty\le r\}$,}\\
&\implies
\mathbb P\left\{\sup_{t\le N\gamma}v^\top X_t>
\frac h{\sqrt\gamma}+\|v\|_2r\right\}-\delta\\
&\le\mathbb P\left\{\max_{0\le k\le N}v^\top\Delta_{b+k}>h\right\}\\
&\le\mathbb P\left\{\sup_{t\le N\gamma}v^\top X_t>
\frac h{\sqrt\gamma}-\|v\|_2r\right\}+\delta.
\end{align*}

\subsection{Gaussian risk probabilities}
For $\psi=(A,S)$, let $X_\psi$ be stationary OU with
$A\Gamma_\psi+\Gamma_\psi A=S$. Consider either
\begin{align*}
&F_\psi(X)=\sup_{0\le t\le T}\{v^\top X(t)-h(t)\},\\
&F_\psi(X)=\frac12\int_0^T X(t)^\top AX(t)\,dt,
\end{align*}
where $v,h,T$ are fixed, $h$ is continuous and $0<T<\infty$.
\begin{proposition}[Consistent estimated path risk]
\label{prop:risk-plugin-new}
Suppose $\widehat A$ is symmetric, $\widehat S\succeq0$,
$\widehat A\to_{\mathbb P}A\succ0$,
$\widehat S\to_{\mathbb P}S\succeq0$ and
$\mathbb P\{F_\psi(X_\psi)=z\}=0$. Then
\begin{align*}
&\widehat p(z)=\mathbb P_{\widehat\psi}
\{F_{\widehat\psi}(X_{\widehat\psi})>z\}
\longrightarrow_{\mathbb P}
 p(z)=\mathbb P_\psi\{F_\psi(X_\psi)>z\}.
\end{align*}
Define $\widehat p$ arbitrarily outside
$\{\widehat A\succ0,\widehat S\succeq0\}$(the probability of $\{\widehat A\succ0,\widehat S\succeq0\}$ tends to one).
If the corresponding SGD statistic converges to $F_\psi(X_\psi)$,
then $\widehat p(z)$ also consistently estimates  the equivalent SGD event probability.
\end{proposition}
\begin{proof}
For deterministic $\psi_j\to\psi$, use the same independent
$\xi\sim N(0,I_d)$ and Brownian motion $B$ for all processes. Then
\begin{align*}
&X_\psi(t)=e^{-At}\Gamma_\psi^{1/2}\xi+S^{1/2}B(t)
-\int_0^t Ae^{-A(t-s)}S^{1/2}B(s)\,ds,
\quad\text{by integration by parts,}\\
&\implies \|X_{\psi_j}-X_\psi\|_\infty\\
&\le\sup_{t\le T}
\|e^{-A_jt}\Gamma_{\psi_j}^{1/2}-e^{-At}\Gamma_\psi^{1/2}\|_2\|\xi\|_2
+\|S_j^{1/2}-S^{1/2}\|_2\|B\|_\infty\\
&\quad+T\sup_{t\le T}
\|A_je^{-A_jt}S_j^{1/2}-Ae^{-At}S^{1/2}\|_2\|B\|_\infty
\longrightarrow0\quad\text{almost surely.}
\end{align*}
Here the Lyapunov bound in Theorem~\ref{thm:cov-new} gives
$\Gamma_{\psi_j}\to\Gamma_\psi$ as matrix square roots and exponentials
are continuous. For the two functionals, respectively,
\begin{align*}
&|F_{\psi_j}(X_{\psi_j})-F_\psi(X_\psi)|
\le\|v\|_2\|X_{\psi_j}-X_\psi\|_\infty\longrightarrow0,\\
&|F_{\psi_j}(X_{\psi_j})-F_\psi(X_\psi)|\\
&\le\frac T2\|A_j-A\|_2\|X_{\psi_j}\|_\infty^2
+\frac T2\|A\|_2\|X_{\psi_j}-X_\psi\|_\infty
(\|X_{\psi_j}\|_\infty+\|X_\psi\|_\infty)
\longrightarrow0.
\end{align*}
Consequently,
\begin{align*}
&|p_{\psi_j}(z)-p_\psi(z)|\\
&\le\mathbb E\left|
\mathbf1\{F_{\psi_j}(X_{\psi_j})>z\}
-\mathbf1\{F_\psi(X_\psi)>z\}\right|
\longrightarrow0,
\quad\text{by DCT and that there are no point masses,}\\
&\implies\widehat p(z)\longrightarrow_{\mathbb P}p(z),
\quad\text{by continuity of $\psi\mapsto p_\psi(z)$.}
\end{align*}
For the SGD statistic $T_n\Rightarrow F_\psi(X_\psi)$,
\begin{align*}
&|\widehat p(z)-\mathbb P\{T_n>z\}|\\
&\le|\widehat p(z)-p(z)|
+|p(z)-\mathbb P\{T_n>z\}|
\longrightarrow_{\mathbb P}0.
\end{align*}
\end{proof}

\begin{lemma}\label{lem:OU-continuity}
Let $X$ be stationary OU with symmetric $A\succ0$ and
$\Gamma\succ0$. For $T>0$, $v\ne0$ and continuous $h$, both
\begin{align*}
&\sup_{0\le t\le T}\{v^\top X_t-h(t)\},
\quad\frac12\int_0^T X_t^\top AX_t\,dt
\end{align*}
have no point mass so we can apply the previous theorems.
\end{lemma}
\begin{proof}
Choose a unit eigenvector $e$ with $Ae=\lambda e$ and $v^\top e>0$.
Let $Z=e^\top\Gamma^{-1}X_0/(e^\top\Gamma^{-1}e)$ and $R=X_0-eZ$.
Then
\begin{align*}
&\operatorname{var}(Z)=\frac1{e^\top\Gamma^{-1}e}>0,
\quad\operatorname{cov}(R,Z)
=\frac e{e^\top\Gamma^{-1}e}-e\operatorname{var}(Z)=0,\\
&\implies Z\text{ is independent of }R\text{ and }B,
\text{since they are jointly Gaussian,}\\
&X_t=e^{-\lambda t}eZ+Y_t,
\quad Y_t=e^{-At}R+\int_0^t e^{-A(t-s)}S^{1/2}\,dB_s,
\quad Z\text{ independent of }Y.
\end{align*}
Conditional on $Y$, write the maximum as $M(Z)$. For $z_2>z_1$,
\begin{align*}
&M(z_2)-M(z_1)
\ge(v^\top e)e^{-\lambda T}(z_2-z_1)>0,\\
&\implies\mathbb P\{M(Z)=a\mid Y\}=0,
\quad\text{since each level has at most one preimage.}
\end{align*}
For the loss functional,
\begin{align*}
&\frac12\int_0^T X_t^\top AX_t\,dt\\
&=\frac{1-e^{-2\lambda T}}4Z^2
+Z\int_0^T e^{-\lambda t}e^\top AY_t\,dt
+\frac12\int_0^T Y_t^\top AY_t\,dt,\\
&\frac{1-e^{-2\lambda T}}4>0
\implies
\mathbb P\left\{\frac12\int_0^T X_t^\top AX_t\,dt=a\mid Y\right\}=0,
\quad\text{since each level has at most two preimages.}
\end{align*}
Taking expectations proves both assertions.
\end{proof}

\subsection{Finite sample risk intervals}
Suppose $\mathbb P_\psi\{|T_\psi-T_\psi^\diamond|>e_\psi\}
\le\delta_\psi$. Let $\mathcal C$
be an initial confidence set satisfying
$\mathbb P_{\psi_0}\{\psi_0\in\mathcal C\}\ge1-\alpha_{\rm est}$,
and let  $p_\psi^\diamond(z)=\mathbb P_\psi\{T_\psi^\diamond>z\}$.
\begin{proposition}[Confidence interval for a risk probability]
\label{prop:risk-CI-new}
With probability at least $1-\alpha_{\rm est}$,
\begin{align*}
&\max\left\{0,\inf_{\psi\in\mathcal C}
[p_\psi^\diamond(z+e_\psi)-\delta_\psi]\right\}\\
&\le\mathbb P_{\psi_0}\{T_{\psi_0}>z\}\\
&\le\min\left\{1,\sup_{\psi\in\mathcal C}
[p_\psi^\diamond(z-e_\psi)+\delta_\psi]\right\}.
\end{align*}
\end{proposition}
\begin{proof}
On $\{|T_\psi-T_\psi^\diamond|\le e_\psi\}$,
\begin{align*}
&T_\psi^\diamond>z+e_\psi\implies T_\psi>z
\implies T_\psi^\diamond>z-e_\psi,\\
&\implies p_\psi^\diamond(z+e_\psi)-\delta_\psi
\le\mathbb P_\psi\{T_\psi>z\}
\le p_\psi^\diamond(z-e_\psi)+\delta_\psi,\\
&\psi_0\in\mathcal C\implies
\inf_{\psi\in\mathcal C}[p_\psi^\diamond(z+e_\psi)-\delta_\psi]
\le\mathbb P_{\psi_0}\{T_{\psi_0}>z\}\\
&\le\sup_{\psi\in\mathcal C}[p_\psi^\diamond(z-e_\psi)+\delta_\psi].
\end{align*}
We know it has to be in $[0,1]$ so we add these bounds and use
$\mathbb P_{\psi_0}\{\psi_0\in\mathcal C\}\ge1-\alpha_{\rm est}$.
\end{proof}
Equation~\eqref{eq:cov-confidence-new} supplies a confidence set
for $(A,S)$ which we can plug into the theorem. 

\subsection{Bankruptcy bounds with cash, borrowing, and short positions}
\label{app:portfolio-wealth}
Use the set up of Section~\ref{portfolio example}.
Let $R_j$ be i.i.d. with mean $\mu_R$, covariance $S$ and
$M_p=\|R_1-\mu_R\|_p<\infty$ for $p>2$. Fix symmetric $A\succ0$
and $\theta_0$. We will condition on past data if $A$ has already been independently estimated.
Let $b\ge0$ be an integer burn in (you can think of this as some training length),  $n\ge1$ be the number of trading iterations,
and $W_0>0$ the initial capital. Positions $\theta_k$ represent holdings in dollars (not necessarily proportions of assets). For $0\le r_{\rm lend}\le r_{\rm borrow}$, let
\begin{align*}
&f(c)=r_{\rm lend}(c)_+-r_{\rm borrow}(-c)_+,
\quad(c)_+=\max(c,0),\\
&|f(c)-f(c')|\le r_{\rm borrow}|c-c'|.
\end{align*}

\paragraph{Cash accounting.}
We follow \citet{Boyd2017Trading}, and define our SGD iterates, cash balance and wealth as
\begin{equation}\label{eq:portfolio-new}
\begin{aligned}
&\theta_j=Q\theta_{j-1}+\gamma R_j,
\quad Q=I_d-\gamma A,\\
&C_m=W_m-\mathbf1^\top\theta_{b+m},\\
&W_m=W_{m-1}+\theta_{b+m-1}^\top R_{b+m}+f(C_{m-1}),
\quad1\le m\le n.
\end{aligned}
\end{equation}
Letting
$H_m=\theta_{b+m-1}\odot(\mathbf1+R_{b+m})$ (pointwise product), we have
\begin{align*}
&C_m=C_{m-1}+f(C_{m-1})
-\mathbf1^\top(\theta_{b+m}-H_m),\\
&W_m=C_m+\mathbf1^\top\theta_{b+m}\\
&=C_{m-1}+f(C_{m-1})+\mathbf1^\top H_m\\
&=W_{m-1}+\theta_{b+m-1}^\top R_{b+m}+f(C_{m-1}).
\end{align*}
Under our model purchases reduce cash, short positions add cash and a
negative holding. $\theta_j$ above corresponds to using constant rate SGD to do portfolio selection.

Here $\theta^*=A^{-1}\mu_R$, $Z_j=\mu_R-R_j$ and
$\Delta_j=Q\Delta_{j-1}-\gamma Z_j$. 
Use Lemma~\ref{lem:innovation-new} and define
\begin{align*}
&\Delta_j^\diamond=Q^j\Delta_0
-\gamma\sum_{i=1}^jQ^{j-i}Z_i^\diamond,
\quad\theta_j^\diamond=\theta^*+\Delta_j^\diamond.
\end{align*}
With the same Gaussian vectors for returns and allocations, set
$W_0^\diamond=W_0$ and
\begin{equation}\label{eq:wealth-reference}
\begin{aligned}
&C_m^\diamond=W_m^\diamond-\mathbf1^\top\theta_{b+m}^\diamond,\\
&W_m^\diamond=W_{m-1}^\diamond
+(\theta^*)^\top\mu_R-(\theta^*)^\top Z_{b+m}^\diamond\\
&\quad+\mu_R^\top\Delta_{b+m-1}^\diamond+f(C_{m-1}^\diamond),
\quad1\le m\le n.
\end{aligned}
\end{equation}

\subsubsection{Bankruptcy Results}
\begin{proposition}[Wealth crossing (Bankruptcy) transfer]\label{prop:wealth-new}
Under these assumptions, suppose $0<\gamma\lambda^*\le1/2$.
Let  $q=1-\gamma\lambda_*$, $N=b+n$ and define
\begin{align*}
&C_W=2\left(\|\theta^*\|_2+\frac{\|\mu_R\|_2}{\lambda_*}\right),\\
&C_{W,n}=C_W+r_{\rm borrow}\sqrt d\,n\gamma
\left(1+\frac{\lambda^*}{\lambda_*}\right),\\
&B_{p,b,n}=\frac{pM_p}{\sqrt{p-1}\sqrt{1-q^2}}
\left(q^b\|\Delta_0\|_2+\gamma\sqrt{n(p-1)}M_p\right).
\end{align*}
For $0<\delta<1$, let
\begin{equation}\label{eq:wealth-radius}
\begin{aligned}
&e_\delta=(1+r_{\rm borrow})^n\left[
C_{W,n}\left(\frac{2K_{p,d,S}NM_p^p}{\delta}\right)^{1/p}
+\left(\frac2\delta\right)^{1/p}B_{p,b,n}\right].
\end{aligned}
\end{equation}
If $r_{\rm borrow}=O(1/n)$, the amplification factor $(1+r_{\rm borrow})^n$ remains bounded, so borrowing does not itself prevent informative bounds at long horizons.
Then we can construct the simulatable wealth process given at the start of this section and show
\begin{align*}
&\mathbb P\left\{\max_{0\le m\le n}|W_m-W_m^\diamond|>e_\delta\right\}
\le\delta.
\end{align*}
Consequently, for $h<W_0$,
\begin{align*}
&\mathbb P\{\min_{m\le n}W_m^\diamond\le h-e_\delta\}-\delta\\
&\le\mathbb P\{\min_{m\le n}W_m\le h\}\\
&\le\mathbb P\{\min_{m\le n}W_m^\diamond\le h+e_\delta\}+\delta.
\end{align*}
\end{proposition}
\begin{proof}
Expanding the wealth recursion results in
\begin{align*}
&W_m=W_0+m \theta^{* \top}\mu_R
-(\theta^*)^\top\sum_{k=1}^mZ_{b+k}\\
&\quad+\mu_R^\top\sum_{k=1}^m\Delta_{b+k-1}+M_m
+\sum_{k=1}^mf(C_{k-1}),\\
&M_m=-\sum_{k=1}^m\Delta_{b+k-1}^\top Z_{b+k},\\
&\mathbb E[\Delta_{b+k-1}^\top Z_{b+k}\mid R_1,\ldots,R_{b+k-1}]=0,\\
&\|\Delta_{b+k-1}^\top Z_{b+k}\|_p
\le M_p\|\Delta_{b+k-1}\|_p,
\quad\text{by independence.}
\end{align*}
Applying Doob and Rio, as in Theorem~\ref{thm:sum-main}, gives
\begin{align*}
&\left\|\max_{m\le n}|M_m|\right\|_p
\le\frac p{p-1}\|M_n\|_p\\
&\le\frac{pM_p}{\sqrt{p-1}}
\left(\sum_{k=1}^n\|\Delta_{b+k-1}\|_p^2\right)^{1/2},\\
&\|\Delta_j\|_p
\le q^j\|\Delta_0\|_2
+\gamma\sqrt{p-1}M_p\left(\sum_{i=0}^{j-1}q^{2i}\right)^{1/2}\\
&\le q^j\|\Delta_0\|_2
+\frac{\gamma\sqrt{p-1}M_p}{\sqrt{1-q^2}},\\
&\implies \left\|\max_{m\le n}|M_m|\right\|_p\\
&\le\frac{pM_p}{\sqrt{p-1}}
\left\{q^b\|\Delta_0\|_2\left(\sum_{k=1}^nq^{2(k-1)}\right)^{1/2}
+\frac{\gamma\sqrt{n(p-1)}M_p}{\sqrt{1-q^2}}\right\}\\
&\le B_{p,b,n},
\quad\text{by Minkowski and the geometric sum formula.}
\end{align*}
Let $F_j=\sum_{i=1}^j(Z_i-Z_i^\diamond)$, $F_0=0$ and
$E_N=\max_{j\le N}\|F_j\|_2$. As in
Theorem~\ref{kmt for the linear component},
\begin{align*}
&\sum_{j=1}^{\ell}(\Delta_j-\Delta_j^\diamond)
=-\gamma\sum_{i=0}^{\ell-1}Q^iF_{\ell-i},\\
&\implies \left\|\sum_{j=1}^{\ell}(\Delta_j-\Delta_j^\diamond)\right\|_2
\le\gamma E_N\sum_{i=0}^{\ell-1}q^i
\le\frac{E_N}{\lambda_*},\\
&\implies \max_{m\le n}\left\|\sum_{k=1}^m
(\Delta_{b+k-1}-\Delta_{b+k-1}^\diamond)\right\|_2
\le\frac{2E_N}{\lambda_*}, \\
& \max_{m\le n}\left\|\sum_{k=1}^m
(Z_{b+k}-Z_{b+k}^\diamond)\right\|_2
=\max_{m\le n}\|F_{b+m}-F_b\|_2\le2E_N.
\end{align*}
and for single iterate errors
\begin{align*}
&\Delta_j-\Delta_j^\diamond
=-\gamma F_j+\gamma(I_d-Q)\sum_{i=1}^{j-1}Q^{j-1-i}F_i,
\quad\text{by summation by parts,}\\
&\implies \max_{j\le N}\|\Delta_j-\Delta_j^\diamond\|_2
\le\gamma E_N+\gamma^2\lambda^*E_N\sum_{i\ge0}q^i
=\gamma\left(1+\frac{\lambda^*}{\lambda_*}\right)E_N.
\end{align*}
Let $D_m=\max_{0\le k\le m}|W_k-W_k^\diamond|$. Then
\begin{align*}
&|C_{k-1}-C_{k-1}^\diamond|
\le D_{k-1}+\sqrt d\,
\|\Delta_{b+k-1}-\Delta_{b+k-1}^\diamond\|_2,\\
&D_m\le C_WE_N+\max_{k\le n}|M_k|
+r_{\rm borrow}\sum_{k=1}^m|C_{k-1}-C_{k-1}^\diamond|\\
&\le C_{W,n}E_N+\max_{k\le n}|M_k|
+r_{\rm borrow}\sum_{k=1}^mD_{k-1}\\
&\le\left(C_{W,n}E_N+\max_{k\le n}|M_k|\right)
\left[1+r_{\rm borrow}\sum_{k=1}^m(1+r_{\rm borrow})^{k-1}\right]\quad\text{by induction, }\\
&=(1+r_{\rm borrow})^m
\left(C_{W,n}E_N+\max_{k\le n}|M_k|\right)
\end{align*}
Thus, for $u,v>0$,
\begin{align*}
&\mathbb P\{D_n>(1+r_{\rm borrow})^n(C_{W,n}u+v)\}\\
&\le\mathbb P\{E_N>u\}+\mathbb P\{\max_{m\le n}|M_m|>v\}\\
&\le\frac{K_{p,d,S}NM_p^p}{u^p}+\frac{B_{p,b,n}^p}{v^p},
\quad\text{by Lemma~\ref{lem:innovation-new} and Markov's inequality applied with our prior bound,}\\
&u=\left(\frac{2K_{p,d,S}NM_p^p}{\delta}\right)^{1/p},
\quad v=\left(\frac2\delta\right)^{1/p}B_{p,b,n}
\implies\mathbb P\{D_n>e_\delta\}\le\delta.
\end{align*}
 Finally, on $\{D_n\le e_\delta\}$,
\begin{align*}
&\min_{m\le n}W_m^\diamond\le h-e_\delta
\implies\min_{m\le n}W_m\le h
\implies\min_{m\le n}W_m^\diamond\le h+e_\delta.
\end{align*}
Taking probabilities of these events yields our crossing probability bounds.
\end{proof}

Combining Proposition~\ref{prop:wealth-new},
\eqref{eq:cov-confidence-new}, and
Proposition~\ref{prop:risk-CI-new}, with the sample mean
$\bar R_N=N^{-1}\sum_{j=1}^N R_j$ satisfying
$\mathbb P\{\|\bar R_N-\mu_R\|_2>M_2/\sqrt{N\alpha_\mu}\}
\le\alpha_\mu$ by independence and Markov's inequality,
yields finite sample confidence intervals for bankruptcy probabilities. These confidence intervals additionally require the known moment bounds in \eqref{eq:cov-confidence-new}, including finite $K_A$ and $K_S$, and have coverage at least $1-\alpha_{\rm cov}-\alpha_\mu$ when that equation is evaluated at $\delta=\alpha_{\rm cov}$.

\section{Logistic regression experiment}
\label{app:logistic-experiment}

For our exact simulation we let $X\sim Rademacher(0.5)$, with $Y\mid X\sim\operatorname{Bernoulli}(\sigma(X\log3))$,
where $\sigma(u)=(1+e^{-u})^{-1}$ or
$\mathbb P(Y=1\mid X=1)=3/4$ and
$\mathbb P(Y=1\mid X=-1)=1/4$. Under this symmetric model we only need to estimate a regularized version of the slope $\log(3)$ by $\theta$ since the intercept will be zero. Also independent of the response and covariate, draw weight $W=(4/5)V$, where
$\mathbb P(V>v)=v^{-5}$ for $v\ge1$. We use this weight distribution to give heavy tails, have mean 1 and satisfy our moment conditions. Then the regularized weighted negative log likelihood objective we consider is
\[
g(\theta;(X,Y,W))
=W\{\log(1+e^{X\theta})-YX\theta\}
+\frac{\theta^2}{32}.
\]
 Moments of every order below five exist for our model.

By taking expectations and then setting the derivative to zero, the true optimum solves the equation
\[
\sigma(\theta^*)+\theta^*/16=3/4,
 ,  \theta^*\approx0.836.
\]
For $a=\sigma(\theta^*)$, the second derivative, noise covariance and subsequent bias are
\begin{align*}
A&=a(1-a)+1/16,\\
S&=\frac{16}{15}
\left\{\frac3{16}+(a-3/4)^2\right\}-(a-3/4)^2,\\
\Gamma&=\frac{S}{2A},   
A^{-1}b=-\frac{a(1-a)(1-2a)S}{4A^2}
=0.05584.
\end{align*}

We use $\gamma\in\{1/128,1/256,1/512\}$ for our Richardson-Romberg extrapolations
$2\theta_k(\gamma)-\theta_k(2\gamma)$. Initialization at $\theta^*$ removes the initialization effect. All learning rates, including $2\gamma$,
satisfy our step size restrictions for any
moments of order less than $9/2$.

For initialization at zero and $\gamma=1/128$, the
biased FCLT predicts the normalized partial-sum mean
\[
\frac I4+4tA^{-1}b,   
I=\int_{\theta^*}^{0}
\frac{u-\theta^*}{\sigma(u)-3/4+u/16}\,du
=-2.91.
\]

Burn in at $\gamma=1/128$ discards 51,105 updates.
The smaller rate experiment uses $\gamma_b=n^{-3/5}$ and
\[
n_0=\left\lceil
\frac{32}{\gamma_b}
\left|\log(\sqrt n\gamma_b)\right|
\right\rceil=71,183
\]
from the FCLT conditions. Both retain $n=262,144$ updates after Burn in.
We also run a Richardson-Romberg trajectory with
$\gamma=\tfrac12n^{-2/5}$, satisfying the constraints for the Richardson-Romberg FCLT.

For the trajectory comparison, set $T=32$ and
\[
n_0=2\left\lceil
\frac{64}{2\gamma}\log(1/\gamma)
\right\rceil.
\]
Here $\mu=\alpha_0=1/16$, so
$\nu=64>3/\min(A,\alpha_0)=48$.
We compare the maximum of
$|\theta_{n_0+j}-\theta^*|/\sqrt\gamma$ over
$j=0,\ldots,T/\gamma$ with an independent stationary
OU process with
\[
U_0\sim N(0,\Gamma),   
U_{j+1}=e^{-A\gamma}U_j+\xi_j,
\]
\[
\xi_j\overset{\mathrm{i.i.d.}}{\sim}
N\!\left(0,\Gamma(1-e^{-2A\gamma})\right).
\]
After burn in, we observe $T/\gamma$ SGD updates with
$T=32$, which is 4,096, 8,192, and 16,384 iterates for 
\(\gamma=1/128\),
\(1/256\), and
\(1/512\) respectively.
The respective empirical KS distances between
maximum distributions are $0.01342$, $0.00650$, and
$0.01125$. 
\section{Empirical Gaussian approximation}\label{app:empirical-coupling}
Our trajectory coupling experiment starts SGD at zero, whereas the trajectory distribution comparison in Appendix~\ref{app:logistic-experiment} starts at $\theta^*$.
\begingroup\small
Otherwise we continue to work with the weighted logistic model of Appendix~\ref{app:logistic-experiment}. For accumulated errors we look at
$E_n=\max_{m\le n}|\sum_{k\le m}(\Delta_k-\Delta_k^\diamond)|/\sqrt n$, for all types of accumulated error Gaussian processes (bias, burn in, Richardson-Romberg) discussed in the paper. For the trajectory, we evaluate $\sup_{t\le32}|(\theta_{b+\lfloor t/\gamma\rfloor}-\theta^*)/\sqrt\gamma-X_t|$. We construct the reference Gaussian process described in the main body as follows. We couple the SGD noise and Gaussian increments in blocks. Across simulated paths, we assign the smallest Gaussian block total to the path with the smallest SGD noise total, and continue in increasing order. Within each path, we place the smallest centered Gaussian increment at the position of the smallest SGD noise value, and continue in increasing order. We then add the assigned Gaussian total divided by the block length to each increment. This construction still retains the Gaussian path and SGD path marginal distributions.

The learning rates are $.5n^{-1/2}$ (bias), $2n^{-3/4}$ (burn in), $.5n^{-.49}$ (RR), and $32/n$ (trajectory). We use $b=\lceil40\gamma^{-1}|\log(\sqrt n\gamma)|\rceil$ for burn in and the prescribed trajectory burn in with $\nu=64$. Bias and RR start at $\theta^*$. Let $p_1=p_2=p=4$ and $p_3=2.2$, the accumulated KMT rates have orders $n^{-1/4}$, $n^{-1/4}$, and $n^{-.235}$, respectively in this case. The trajectory rate is $\gamma^{1/4}\log^{1/4}(1/\gamma)+\sqrt\gamma\log(1/\gamma)$.
The 95\% error bars use 1,500 bootstrap resamples of eight independent groups of simulated paths. We resample whole groups because paths within each group are dependent.
On 512 paths at each of $n=512$ and $n=16{,}384$, refining from 8 to 32 subdivisions per SGD interval changed the 95th percentile by at most 1.09\% and RMS by at most 1.50\%.

For wealth, $A=1$, $R_k\sim N(.02,.1^2)$, $\theta_0=.02$, $W_0=.006$, $\gamma=.02/n$, and both interest rates are zero. We compare the actual and Gaussian process bankruptcy probabilities using 32,768 independent Monte Carlo runs. Independently we ran 32,768 runs at $n=4096$, and the estimated bankruptcy probability for the actual wealth process was $0.000092$ lower than for the reference process. The conservative 95\% confidence interval for the difference was $[-0.000298,0.000120]$.
\begin{center}
{\scriptsize $\bullet$ 95th percentile\qquad\textcolor{blue!60!black}{$\blacksquare$ RMS maximum error}\qquad\textcolor{orange!80!black}{-- KMT rate guide}}\par
\begin{minipage}{.49\linewidth}\centering
\begin{tikzpicture}
\begin{axis}[width=\linewidth,height=1.65in,xmode=log,log basis x=2,ymode=log,xlabel={Retained iterations $n$},ylabel={Normalized maximum error},tick label style={font=\tiny},label style={font=\scriptsize},title style={font=\small},legend style={font=\tiny,draw=none,fill=none},legend pos=south west,grid=major,grid style={gray!15},scaled ticks=false,title={(a) Bias correction},xtick={1024,4096,16384,65536}]
\addplot+[black,mark=*,mark size=1pt,error bars/.cd,y dir=both,y explicit] coordinates {(1024,0.283789966) += (0,0.0214394524) -= (0,0.0220445149) (4096,0.225351801) += (0,0.00866141033) -= (0,0.0066135404) (16384,0.179123781) += (0,0.00299062353) -= (0,0.00398187525) (65536,0.135450333) += (0,0.00213327484) -= (0,0.00153823704)};
\addplot+[blue!60!black,mark=square*,mark size=1pt] coordinates {(1024,0.157082008) (4096,0.12686441) (16384,0.101956538) (65536,0.0780174437)};
\addplot+[orange!80!black,dashed] coordinates {(1024,0.283789966) (4096,0.194977599) (16384,0.135023877) (65536,0.0940532465)};
\end{axis}
\end{tikzpicture}
\end{minipage}
\hfill
\begin{minipage}{.49\linewidth}\centering
\begin{tikzpicture}
\begin{axis}[width=\linewidth,height=1.65in,xmode=log,log basis x=2,ymode=log,xlabel={Retained iterations $n$},ylabel={Normalized maximum error},tick label style={font=\tiny},label style={font=\scriptsize},title style={font=\small},legend style={font=\tiny,draw=none,fill=none},legend pos=south west,grid=major,grid style={gray!15},scaled ticks=false,title={(b) burn in},xtick={1024,4096,16384,65536}]
\addplot+[black,mark=*,mark size=1pt,error bars/.cd,y dir=both,y explicit] coordinates {(1024,0.280306584) += (0,0.00520522823) -= (0,0.00185798729) (4096,0.227399914) += (0,0.00814827554) -= (0,0.010104085) (16384,0.164572041) += (0,0.00280151353) -= (0,0.00166832944) (65536,0.124081524) += (0,0.00158009753) -= (0,0.00131521483)};
\addplot+[blue!60!black,mark=square*,mark size=1pt] coordinates {(1024,0.152372728) (4096,0.121391133) (16384,0.0909710257) (65536,0.0683968771)};
\addplot+[orange!80!black,dashed] coordinates {(1024,0.280306584) (4096,0.216886292) (16384,0.158814437) (65536,0.112244827)};
\end{axis}
\end{tikzpicture}
\end{minipage}
\par\vspace{2pt}
\begin{minipage}{.49\linewidth}\centering
\begin{tikzpicture}
\begin{axis}[width=\linewidth,height=1.65in,xmode=log,log basis x=2,ymode=log,xlabel={Retained iterations $n$},ylabel={Normalized maximum error},tick label style={font=\tiny},label style={font=\scriptsize},title style={font=\small},legend style={font=\tiny,draw=none,fill=none},legend pos=south west,grid=major,grid style={gray!15},scaled ticks=false,title={(c) Richardson--Romberg},xtick={4096,8192,16384,32768,65536}]
\addplot+[black,mark=*,mark size=1pt,error bars/.cd,y dir=both,y explicit] coordinates {(4096,0.22773575) += (0,0.00647386283) -= (0,0.00848291352) (8192,0.198614204) += (0,0.00740756082) -= (0,0.00666664442) (16384,0.178616403) += (0,0.00493330311) -= (0,0.00547944999) (32768,0.153348976) += (0,0.00180739553) -= (0,0.00186037477) (65536,0.133167338) += (0,0.00256141376) -= (0,0.00123952575)};
\addplot+[blue!60!black,mark=square*,mark size=1pt] coordinates {(4096,0.121824981) (8192,0.111227139) (16384,0.098998961) (32768,0.0867404731) (65536,0.0767108099)};
\addplot+[orange!80!black,dashed] coordinates {(4096,0.22773575) (8192,0.189954779) (16384,0.158712486) (32768,0.132803261) (65536,0.111263196)};
\end{axis}
\end{tikzpicture}
\end{minipage}
\hfill
\begin{minipage}{.49\linewidth}\centering
\begin{tikzpicture}
\begin{axis}[width=\linewidth,height=1.65in,xmode=log,log basis x=2,ymode=log,xlabel={Retained iterations $n$},ylabel={Normalized maximum error},tick label style={font=\tiny},label style={font=\scriptsize},title style={font=\small},legend style={font=\tiny,draw=none,fill=none},legend pos=south west,grid=major,grid style={gray!15},scaled ticks=false,title={(d) Trajectory / stationary OU},xtick={512,1024,2048,4096,8192,16384}]
\addplot+[black,mark=*,mark size=1pt,error bars/.cd,y dir=both,y explicit] coordinates {(512,0.781396418) += (0,0.0190742738) -= (0,0.0115086068) (1024,0.654346516) += (0,0.0127294786) -= (0,0.00938633808) (2048,0.562503899) += (0,0.00874777027) -= (0,0.0102035905) (4096,0.471011082) += (0,0.00508488096) -= (0,0.00894156039) (8192,0.395532194) += (0,0.00308655043) -= (0,0.00409591784) (16384,0.332567908) += (0,0.00302431879) -= (0,0.00245970835)};
\addplot+[blue!60!black,mark=square*,mark size=1pt] coordinates {(512,0.626591222) (1024,0.531778527) (2048,0.452995628) (4096,0.38110273) (8192,0.323984258) (16384,0.273399347)};
\addplot+[orange!80!black,dashed] coordinates {(512,0.781396418) (1024,0.692643403) (2048,0.598305989) (4096,0.508014938) (8192,0.426335354) (16384,0.354945567)};
\end{axis}
\end{tikzpicture}
\end{minipage}
\par\vspace{2pt}
\begin{tikzpicture}
\begin{axis}[width=.94\linewidth,height=1.50in,xmode=log,log basis x=2,xlabel={Trading iterations $n$},ylabel={Bankruptcy probability},tick label style={font=\scriptsize},label style={font=\scriptsize},title style={font=\small},legend style={font=\scriptsize,draw=none,fill=none},legend pos=south east,grid=major,grid style={gray!15},ymin=.13,ymax=.265,scaled ticks=false,title={(e) Bankruptcy with zero borrowing and lending rates},xtick={16,32,64,128,256,1024,4096}]
\addplot+[black,mark=*,mark size=1pt,error bars/.cd,y dir=both,y explicit] coordinates {(16,0.15411377) += (0,0.00395519926) -= (0,0.00389378376) (32,0.202423096) += (0,0.00439205178) -= (0,0.00433924158) (64,0.2265625) += (0,0.00457180282) -= (0,0.00452328556) (128,0.239776611) += (0,0.0046609203) -= (0,0.00461475168) (256,0.241485596) += (0,0.00467199489) -= (0,0.00462612996) (1024,0.233764648) += (0,0.00462114935) -= (0,0.00457391229) (4096,0.239715576) += (0,0.00466052291) -= (0,0.00461434344)};
\addlegendentry{Simulated bankruptcy}
\addplot+[orange!80!black,dashed,mark=x,mark size=2pt] coordinates {(16,0.157501221) (32,0.205505371) (64,0.228546143) (128,0.241516113) (256,0.242126465) (1024,0.234100342) (4096,0.239776611)};
\addlegendentry{Gaussian-reference approximation}
\end{axis}
\end{tikzpicture}
\end{center}
{\small The dashed error curves show KMT rate shapes, RMS and 95th-percentile errors are normalized as in their respective FCLTs.}
\endgroup
\end{document}